\documentclass{article}
\usepackage[nonatbib,final]{neurips_2021}
\usepackage[utf8]{inputenc} 
\usepackage[T1]{fontenc}    
\usepackage[numbers,compress]{natbib}

\usepackage{hyperref}       
\usepackage{url}            
\usepackage{booktabs}       
\usepackage{amsfonts}       
\usepackage{microtype}      
\usepackage{wrapfig}
\usepackage{caption}

\usepackage{times}
\usepackage{algorithm}
\usepackage{adjustbox}
\usepackage[noend]{algpseudocode}
\usepackage{amsthm}
\usepackage{graphicx}
\usepackage{latexsym}
\usepackage{mathtools}
\usepackage{multirow}
\usepackage{paralist}
\usepackage{minitoc} 
\usepackage{xspace}
\usepackage{cleveref}
\usepackage{enumitem}

\usepackage[dvipsnames]{xcolor} 
\usepackage{tikz}
\usepackage{pgfplots}
\usetikzlibrary{shapes}
\usetikzlibrary{positioning}
\usetikzlibrary{plotmarks}
\usetikzlibrary{patterns}
\usetikzlibrary{intersections,shapes.arrows}
\usetikzlibrary{pgfplots.fillbetween}
\definecolor{darkpink}{rgb}{0.91, 0.33, 0.5}

\crefname{section}{Sec.}{Sections}
\crefname{appendix}{Appx.}{Appxs}

\crefname{theorem}{Thm.}{Thms.}
\crefname{lemma}{Lem.}{Lems.}
\crefname{corollary}{Cor.}{Cors.}
\crefname{proposition}{Prop.}{Props.}

\crefname{algorithm}{Alg.}{Algs.}
\Crefname{algorithm}{Algorithm}{Algorithms}
\crefname{figure}{Fig.}{Figs.}
\crefname{table}{Tab.}{Tabs.} 
\usepackage{bbm, mathtools}

\newcommand{\myparagraph}[1]{\par\noindent\textbf{{#1}.}} 

\newcommand \reals {\mathbb{R}}

\newcommand \inv {^{-1}} 

\newcommand{\mray}{{\operatorname{FI}}}  
\newcommand{\mauveray}{Frontier Integral\xspace}
\newcommand{\klam}[1]{\mathrm{KL}_{#1}}
\newcommand{\lerror}[1]{\mathcal{L}_{#1}}

\newcommand \kl {\mathrm{KL}}
\newcommand \js {\mathrm{JS}}
\newcommand \lc {\mathrm{LC}}
\newcommand \tv {\mathrm{TV}}

\newcommand{\Phatn}{\ensuremath{\hat P_n}}

\newcommand{\fdiv}{\ensuremath{f}}
\newcommand{\ftil}{\ensuremath{f^*}}
\newcommand{\ftilg}{\ensuremath{(f^*)'}}
\newcommand{\ftilh}{\ensuremath{(f^*)''}}
\newcommand{\Df}[2]{D_\fdiv(#1 \Vert #2)}
\newcommand{\Dftil}[2]{D_{\ftil}(#1 \Vert #2)}

\newcommand{\ConstZ}{\ensuremath{C_0}}
\newcommand{\ConstZTil}{\ensuremath{C_0^*}}
\newcommand{\ConstI}{\ensuremath{C_1}}
\newcommand{\ConstITil}{\ensuremath{C_1^*}}
\newcommand{\ConstII}{\ensuremath{C_2}}
\newcommand{\ConstIITil}{\ensuremath{C_2^*}}

\newcommand{\zipf}{\mathrm{Zipf}}
\newcommand{\dir}{\mathrm{Dir}}

\newcommand \expect {\mathbb{E}}

\newcommand{\Supp}[1]{\mathrm{Supp}(#1)}

\newcommand \indone {\mathbbm{1}}
\newcommand \prob {{\mathbb{P}}}

\DeclarePairedDelimiterX{\inp}[2]{\langle}{\rangle}{#1, #2} 

\DeclarePairedDelimiterX{\norm}[1]{\Vert}{\Vert}{#1} 

\newcommand{\abs}[1]{\left\lvert #1 \right\rvert}

\newcommand \D {\mathrm{d}}

\newcommand \Xcal {\mathcal X}

\newcommand \Tcal {\mathcal T}

\newcommand \Fcal {\mathcal F}

\newcommand \Pcal {\mathcal P}

\newcommand \Scal {\mathcal S}

\newtheorem{theorem}{Theorem}
\newtheorem{lemma}[theorem]{Lemma}

\newtheorem{property}[theorem]{Property}
\newtheorem{proposition}[theorem]{Proposition}
\newtheorem{corollary}[theorem]{Corollary}
\newtheorem{remark}[theorem]{Remark}
\theoremstyle{definition}

\newtheorem{assumption}[theorem]{Assumption}
\newtheorem{example}[theorem]{Example}

\title{Divergence Frontiers for Generative Models: \\ Sample Complexity, Quantization Effects, \\ and Frontier Integrals}

\author{Lang Liu$^1$ \qquad Krishna Pillutla$^2$ \qquad Sean Welleck$^{2,3}$ 
\\ \vspace{0.3cm}
 \textbf{Sewoong Oh}$^2$ \qquad \textbf{Yejin Choi}$^{2,3}$ \qquad \textbf{Zaid Harchaoui}$^1$ 
\\
$^1$ Department of Statistics, University of Washington \\
$^2$ Paul G. Allen School of Computer Science \& Engineering, University of Washington \\
$^3$ Allen Institute for Artificial Intelligence}

\begin{document}

\maketitle
\doparttoc 
\faketableofcontents 

\begin{abstract}
    The spectacular success of deep generative models calls for quantitative tools to measure their statistical performance. Divergence frontiers have recently been proposed as an evaluation framework for generative models, due to their ability to measure the quality-diversity trade-off inherent to deep generative modeling. We establish non-asymptotic bounds on the sample complexity of divergence frontiers. We also introduce frontier integrals which provide summary statistics of divergence frontiers. We show how smoothed estimators such as Good-Turing or Krichevsky-Trofimov can overcome the missing mass problem and lead to faster rates of convergence. We illustrate the theoretical results with numerical examples from natural language processing and computer vision.
\end{abstract}

\section{Introduction} \label{sec:intro}

Deep generative models have recently taken a giant leap forward in their ability to model complex, high-dimensional 
distributions. Recent advances are able to produce incredibly detailed and realistic images~\cite{kingma2018glow,razavi2019generating,karras2020ada}, 
strikingly consistent and coherent text~\cite{radford2019language,zellers2019grover,brown2020language}, and music of near-human quality~\cite{dhariwal2020jukebox}. 
The advances in these models, particularly in the image domain, 
have been spurred by the development of quantitative evaluation tools
which enable a large-scale comparison of models, as well as diagnosing of where and why a generative model fails~\cite{salimans2016inception,paz2017revisiting,heusel2017gans,sajjadi2018assessing,karras2019style}.

{\em Divergence frontiers} were recently proposed by~\citet{djolonga2020precision} to quantify the trade-off between quality and diversity in generative modeling with modern deep neural networks ~\cite{sajjadi2018assessing,kynknniemi2019improved,simon2019revisiting,naeem2020reliable,pillutla2021mauve}.
In particular, a good generative model must not only produce high-quality samples that are likely under the target distribution
but also cover the target distribution with diverse samples.

While this framework is mathematically elegant and empirically successful~\cite{kynknniemi2019improved,pillutla2021mauve}, the statistical properties of divergence frontiers
are not well understood. Estimating divergence frontiers from data for large generative models involves two approximations:
(a) joint quantization of the model distribution and the target distribution into discrete distributions with quantization level $k$, and
(b) statistical estimation of the divergence frontiers based on the quantized distributions. 

\citet{djolonga2020precision} argue that the quantization often introduces a positive bias, 
making the distributions appear closer than they really are; 
while a small sample size can result in a pessimistic estimate of the divergence frontier. 
The latter effect is due to the {\em missing mass} of the samples, causing the two distributions to appear farther 
than they really are because the samples do not cover some parts of the distributions.
The first consideration favors a large $k$, while the second favors a small $k$.

In this paper, we are interested in answering the following questions: (a) Given two distributions, how many samples are needed to achieve a desired estimation accuracy, or in other words, what is the sample complexity of the estimation procedure; (b) Given a sample size budget, how to choose the quantization level to balance the errors induced by the two approximations; (c) Can we have estimators better than the na\"ive empirical estimator. 

\myparagraph{Outline}
We review the definitions of divergence frontiers and propose a novel statistical summary in \Cref{sec:metrics}.
We establish non-asymptotic bounds for the estimation of divergence frontiers in \Cref{sec:consist}, 
and discuss the choice of the quantization level by balancing the errors induced by the two approximations.
We show how smoothed distribution estimators, such as the add-constant estimator and the Good-Turing estimator, improve the estimation accuracy in \Cref{sec:improvements}.
Finally, we demonstrate in \Cref{sec:experiments}, through simulations on synthetic data as well as generative adversarial networks on images and transformer-based language models on text, that the bounds exhibit the correct dependence of the estimation error on the sample size $n$ and the support size $k$.

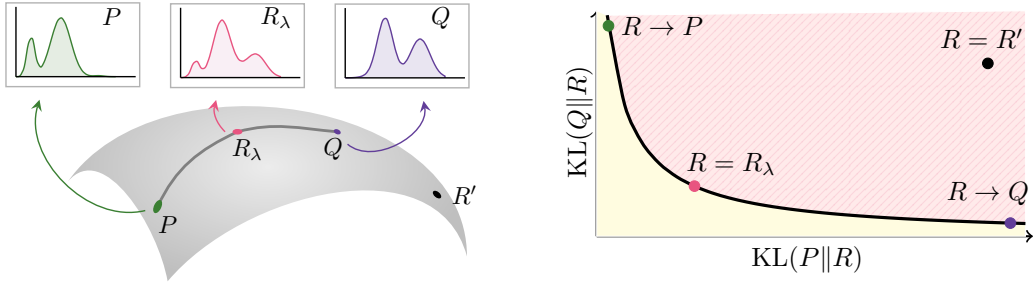
\begin{figure}[t]
    \centering
    \adjustbox{min width=0.48\textwidth}{\begin{tikzpicture}[scale=1]

\shade[inner color=gray!15, outer color=gray!50] 
  (1,0) 
  to[out=50,in=105] (4, 0.25) 
  to[out=90,in=55] (0, 1)  
  to[out=50,in=100] (1, 0);
    
\draw[thick,color=gray] 
    (0.9, 0.75) to[out=65, in=205]
    (1.7, 1.5) to[out=15, in=175]
    (2.7, 1.5) ;

\draw[color=OliveGreen, fill,rotate around={65:(0.9,0.75)}] (0.9,0.75) ellipse (1.8pt and 0.84pt) node[label={[shift={(0.1,0.15)}, text=black]270:{\scalebox{0.7}{$P$}}}] (p_manif) {};
\draw[color=darkpink, fill,rotate around={0:(1.7,1.5)}] (1.7,1.5) ellipse (1.2pt and 0.6pt) node[label={[shift={(0.1,0.17)}, text=black]270:{\scalebox{0.7}{$R_\lambda$}}}] (rl_manif) {};
\draw[color=RoyalPurple, fill,rotate around={-25:(2.7,1.5)}] (2.7,1.5) ellipse (0.84pt and 0.48pt) node[label={[shift={(-0.04,0.18)}, text=black]270:{\scalebox{0.7}{$Q$}}}] (q_manif) {};

\draw[color=black, fill,rotate around={320:(3.7,0.87)}] (3.7,0.87) ellipse (1.2pt and 0.6pt) node[label={[shift={(-0.1,0)}, text=black]0:{\scalebox{0.7}{$R'$}}}] (r1) {};

\def\MixtureOfGaussianP{\x, {
    0.4 * exp( -((\x-0.15)^2)/ (2 * 0.04^2) ) + 
    0.59 * exp( -((\x-0.45)^2)/ (2 * 0.10^2) ) + 
    0.01 * exp( -((\x-0.83)^2)/ (2 * 0.05^2) )
}}
\def\MixtureOfGaussianQ{\x, {
    0.01 * exp( -((\x-0.20)^2)/ (2 * 0.07^2) ) + 
    0.6 * exp( -((\x-0.4)^2)/ (2 * 0.08^2) ) + 
    0.39 * exp( -((\x-0.75)^2)/ (2 * 0.10^2) ) 
}}
\def\MixtureOfGaussianR{\x, {
    0.4 * (
        0.4 * exp( -((\x-0.15)^2)/ (2 * 0.04^2) ) + 
        0.59 * exp( -((\x-0.45)^2)/ (2 * 0.10^2) ) + 
        0.01 * exp( -((\x-0.83)^2)/ (2 * 0.05^2) )
    ) + 
    0.6 * (
        0.01 * exp( -((\x-0.20)^2)/ (2 * 0.07^2) ) + 
        0.6 * exp( -((\x-0.4)^2)/ (2 * 0.08^2) ) + 
        0.39 * exp( -((\x-0.75)^2)/ (2 * 0.10^2) ) 
    )
}}

\node(pbox) at (0.05, 2.4) {
    \begin{tikzpicture}[scale=1]
    \draw[color=gray!60] (-0.1,-0.1) rectangle (1.24,0.74);
    \draw[color=OliveGreen, preaction={fill=OliveGreen!60, fill opacity=0.2}, fill opacity=0.2, pattern color=blue, domain=0:1,smooth] (0, 0) -- plot (\MixtureOfGaussianP) -- (1, 0);
    \draw (0,0) -- (1.2,0) {};
    \draw (0,0) -- (0,0.7) {};
    \node at (0.96, 0.6) {{\scalebox{0.7}{$P$}}};
    \end{tikzpicture}
};

\node(rbox) at (1.7, 2.4) {
    \begin{tikzpicture}[scale=1]
    \draw[color=gray!60] (-0.1,-0.1) rectangle (1.24,0.74);
    \draw[color=darkpink, preaction={fill=Thistle!60, fill opacity=0.2}, fill opacity=0.2, pattern color=blue, domain=0:1,smooth] (0, 0) -- plot (\MixtureOfGaussianR) -- (1, 0);
    \draw[-] (0,0) -- (1.2,0) {};
    \draw[-] (0,0) -- (0,0.7) {};
    \node at (0.96, 0.6) {{\scalebox{0.7}{$R_\lambda$}}};
    \end{tikzpicture}
};

\node(qbox) at (3.35, 2.4) {
    \begin{tikzpicture}[scale=1]
    \draw[color=gray!60] (-0.1,-0.1) rectangle (1.24,0.74);
    \draw[color=RoyalPurple, preaction={fill=RoyalPurple!60, fill opacity=0.2}, fill opacity=0.2, pattern color=blue, domain=0:1,smooth] (0, 0) -- plot (\MixtureOfGaussianQ) -- (1, 0);
    \draw (0,0) -- (1.2,0) {};
    \draw (0,0) -- (0,0.7) {};
    \node at (0.96, 0.6) {{\scalebox{0.7}{$Q$}}};
    \end{tikzpicture}
};

\draw[-stealth, bend left=60, color=OliveGreen] (p_manif) to (pbox.240) ;
\draw[-stealth, bend left, color=darkpink] (rl_manif.180) to ([yshift=0.2]rbox.250) ;
\draw[-stealth, bend right=70, color=RoyalPurple] (q_manif) to (qbox.290) ;

\end{tikzpicture}}
    \hfill
    \adjustbox{max width=0.48\textwidth}{\begin{tikzpicture}

\draw[thick,->] (0.2,0) -- (6,0) node[below, xshift=-3cm] {$\mathrm{KL}(P \Vert R)$};
\draw[thick,->] (0.2,0) -- (0.2,3) node[left, rotate=90, xshift=-0.6cm, yshift=0.23cm] {$\mathrm{KL}(Q \Vert R)$};    

\def\divcurve{\x, {
    (1 / \x)
    }}

\draw[color=white, preaction={fill=red!40, fill opacity=0.2}, fill opacity=0.2, pattern color=darkpink!70, pattern=north east lines, domain=0.34:5.9,smooth] (5.9, 3) --  plot (\divcurve) -- (5.9, 3) ;

\draw[color=white, fill=yellow!15, domain=0.34:5.9,smooth] (0.2, 0) --  (0.2, 3) -- plot (\divcurve) --  (5.9, 0) -- cycle ;

\draw[very thick, color=black, domain=0.34:5.9,smooth] plot (\divcurve) ;

\draw[color=darkpink, fill] (1.5, 0.667) circle (2pt) node[label={[text=black, xshift=-0.25cm, yshift=-0.1cm]80:$R=R_\lambda$}] {};
\draw[color=RoyalPurple, fill] (5.7, 0.175) circle (2pt) node[label={[text=black, xshift=-0.3cm, yshift=-0.01cm]90:$R\to Q$}] {};
\draw[color=OliveGreen, fill] (0.357, 2.8) circle (2pt) node[label={[text=black, xshift=-0.08cm, yshift=0cm]0:$R\to P$}] {};
\draw[color=black, fill] (5.4, 2.3) circle (2pt) node[label={[text=black, xshift=-0.08cm, yshift=-0.05cm]90:$R = R'$}] {};

\end{tikzpicture}}
    \caption{
    \textbf{Left}:
    Comparing two distributions $P$ and $Q$. Here, $R_\lambda = \lambda P + (1-\lambda)Q$ is the interpolation between $P$ and $Q$ for $\lambda \in (0, 1)$ and $R'$ denotes some arbitrary distribution.
    \textbf{Right}: The corresponding divergence frontier (black curve) between $P$ and $Q$. The interpolations $R_\lambda$ for $\lambda \in (0, 1)$ make up the frontier, while all other distributions such as $R'$ must lie above the frontier. 
    }
    \label{fig:main:illustration}
    \vspace{-0.2in}
\end{figure}

\myparagraph{Related work}
The most widely used metrics for generative models include Inception Score~\cite{salimans2016inception}, Fr\'echet Inception Distance~\cite{heusel2017gans}, and Kernel Inception Distance~\cite{binkowski2018kernel}.
The former two are extended to conditional generative models in~\cite{benny2021evaluation}.
They summarize the performance by a single value and thus cannot distinguish different failure cases, i.e., low quality and low diversity.
Motivated by this limitation, \citet{sajjadi2018assessing} propose a metric to evaluate the quality of generative models using two separate components: precision and recall.
This formulation is extended in~\cite{simon2019revisiting} to arbitrary probability measures using a density ratio estimator,
while alternative definitions based on non-parametric representations of the manifolds of the data were
proposed in~\cite{kynknniemi2019improved}.
These notions are generalized by the divergence frontier framework of \citet{djolonga2020precision}.
\citet{pillutla2021mauve} propose Mauve, an area-under-the-curve summary based on divergence frontiers for neural text generation.
They find that Mauve correlates well with human judgements on how close the machine generated text and the human text are. 

Another line of related work is the estimation of functionals of discrete distributions; see~\cite{verdu2019survey} for an overview.
In particular, estimation of KL divergences has been studied by~\cite{cai2006universal,zhang2014nonpar,bu2018kl,han2020minimax} in both fixed and large alphabet regimes. These results focus on the expected $\mathbf{L}_1$ and $\mathbf{L}_2$ risks and require additional assumptions on the two distributions such as boundedness of density ratio which is not needed in our results.
Recently, \citet{sreekumar2021nonasymptotic} investigated a modern way to estimate $f$-divergences using neural networks.
On the practical side, there is a new line of successful work that uses deep neural networks to find data-dependent quantizations for the purpose of estimating information theoretic quantities from samples 
\cite{sablayrolles2018spreading,hamalainen2020deep}. 

\myparagraph{Notation} Let $\Pcal(\Xcal)$ be the space of probability distributions on some measurable space $\Xcal$.
For any $P, Q \in \Pcal(\Xcal)$, let $\kl(P \Vert Q)$ be the Kullback-Leibler (KL) divergence between $P$ and $Q$.
For $\lambda \in (0, 1)$, we define the {\em interpolated KL divergence} as $\klam{\lambda}(P \Vert Q) := \kl(P \Vert \lambda P + (1-\lambda) Q)$. 
For a partition $\Scal := \{S_1, \dots, S_k\}$ of $\Xcal$, we define $P_{\Scal}$ the quantized version of $P$ so that $P_{\Scal} \in \Pcal(\Scal)$ with $P_{\Scal}(S_i) = P(S_i)$ for any $i \in [k] := \{1, \dots, k\}$.

\section{Divergence frontiers}\label{sec:metrics}
Divergence frontiers compare two distributions $P$ and $Q$ using a frontier of statistical divergences.
Each point on the frontier compares the individual distributions against a mixture of the two.
By sweeping through mixtures, the curve interpolates between measurements of two types of costs.
\Cref{fig:main:illustration} illustrates divergence frontiers, which we formally introduce below.

\myparagraph{Evaluating generative models via divergence frontiers}
Consider a generative model $Q \in \Pcal(\Xcal)$ which attempts to model the target distribution $P \in \Pcal(\Xcal)$.
It has been argued in~\cite{sajjadi2018assessing,kynknniemi2019improved} that one must consider two types of costs to evaluate $Q$ with respect to $P$:
(a) a type I cost (loss in precision), which is the mass of $P$ that $Q$ does not adequately capture, and (b) a type II cost (loss in recall), which is the mass of $Q$ that has low or zero probability mass under $P$.

Suppose $P$ and $Q$ are uniform distributions on their supports, and $R$ is uniform on the union of their supports. Then, the type I cost is the mass of $\Supp{Q}\setminus \Supp{P}$, or equivalently, the mass of $\Supp{R}\setminus \Supp{P}$.
We measure this using the surrogate 
$\kl(Q\Vert R)$, which is large if 
there exists $a$ such that $Q(a)$ is large but $R(a)$ is small.
Likewise, the type II cost is measured by $\kl(P \Vert R)$. When $P$ and $Q$ are not constrained to be uniform, it is not clear what the measure $R$ should be. \citet{djolonga2020precision} propose to 
vary $R$ over all possible probability measures and consider the Pareto frontier of the multi-objective optimization $\min_R \big( \kl(P\Vert R), \kl(Q \Vert R) \big)$.
This leads to a curve called the {\em divergence frontier}, and is reminiscent of the precision-recall curve in binary classification. See~\cite{cortes2005confidence,clemencon2009precision,clemenccon2010overlaying,flach2012machine} and references therein on trade-off curves in machine learning.

Formally, it can be shown that the divergence frontier $\Fcal(P, Q)$ of probability measures $P$ and $Q$ is carved out by mixtures $R_\lambda = \lambda P + (1-\lambda)Q$ for $\lambda \in (0, 1)$ (cf. \Cref{fig:main:illustration}). It admits the closed-form
\[
    \Fcal(P, Q) = 
    \Big\{
    \big(\kl(P \Vert R_\lambda), \,  \kl(Q \Vert R_\lambda)\big) \,:\,
    \lambda \in (0, 1)
    \Big\} \,.
\]

\myparagraph{Practical computation of divergence frontiers}
In practical applications, $P$ is a complex, high-dimensional distribution which could either be discrete, as in natural language processing, 
or continuous, as in computer vision. 
Likewise, $Q$ is often a deep generative model
such as GPT-3 for text and GANs for images. It is infeasible to compute the divergence frontier $\Fcal(P, Q)$ directly because we only have samples from $P$ and the integrals  or sums over $Q$ are intractable. 

Therefore, the recipe used by practitioners~\cite{sajjadi2018assessing,djolonga2020precision,pillutla2021mauve} has been to 
(a) jointly quantize $P$ and $Q$ over a partition $\Scal=\{S_t\}_{t=1}^k$ of $\Xcal$ to obtain discrete distributions $P_\Scal = (P(S_t))_{t=1}^k$ and
$Q_\Scal = (Q(S_t))_{t=1}^k$,
(b) estimate the quantized distributions from samples to get $\hat P_\Scal$ and $\hat Q_\Scal$, 
and
(c) compute $\Fcal(\hat P_\Scal, \hat Q_\Scal)$. 
In practice, the best quantization schemes are data-dependent transformations such as $k$-means clustering or lattice-type quantization of dense representations of images or text~\cite{sablayrolles2018spreading}.

\myparagraph{Statistical summary of divergence frontiers}
In the minimax theory of hypothesis testing, where the goal is also to study two types of errors (different from the ones considered here), it is common to theoretically analyze their linear combination; see, e.g., \cite[Sec.~1.2]{ingster2003nonparametric} and \cite[Thm. 7]{cai2011optimal}.
In the same spirit, we consider a linear combination of the two costs, quantified by the KL divergences,
\begin{align}\label{eq:linear_cost}
    \lerror{\lambda}(P, Q) := \lambda \,\kl(P \Vert R_\lambda) + (1 - \lambda) \kl(Q \Vert R_\lambda).
\end{align}
Note that $R_\lambda$ is exactly the minimizer of the linearized objective $\lambda \kl(P \Vert R) + (1 - \lambda) \kl(Q \Vert R)$ according to \cite[Props. 1 and 2]{djolonga2020precision}. $\lerror{\lambda}$ is also known as the $\lambda$-skew Jensen-Shannon Divergence~\cite{nielsen2013matrix}.

The linearized cost $\lerror{\lambda}$ depends on the choice of $\lambda$.
To remove this dependency, we define a novel integral summary, called the {\em frontier integral} $\mray(P, Q)$ of two distributions $P$ and $Q$ as
\begin{align}
    \mray(P, Q) := 2\int_0^1 \lerror{\lambda}(P, Q) \, \D\lambda \;.
\end{align}
We can interpret the frontier integral as the average linearized cost over $\lambda \in (0, 1)$.
While the length of the divergence frontier can be unbounded (e.g., when $\kl(P\Vert Q)$ is unbounded), 
the frontier integral is always bounded in $[0, 1]$.
Moreover, it is a symmetric divergence with $\mray(P, Q) = 0$ iff $P = Q$.
In practice, it can be estimated using the same recipe as the divergence frontier.

\myparagraph{Error decomposition}
In \Cref{sec:consist}, we decompose the error in estimating the frontier integral into two components: the statistical error of estimating the quantized distribution and the quantization error.
Our goal is to derive the rate of convergence for the overall estimation error.
To control the statistical error, we use a different treatment for the masses that appear in the sample and the ones that never appear (i.e., the missing mass).
We obtain a high probability bound as well as a bound for its expectation, leading to upper bounds for its sample complexity and rate of convergence.
These results carry over to the divergence frontiers as well.
As for the quantization error, we construct a distribution-dependent quantization scheme whose error is at most $O(k^{-1})$, where $k$ is the quantization level.
A combination of these two bounds sheds light on the optimal choice of the quantization level.
In \Cref{sec:experiments}, we verify empirically the tightness of the rates on synthetic and real data.

\section{Main results}\label{sec:consist}
\begin{figure}[t]
    \centering
    \adjustbox{max width=0.32\textwidth}{\begin{tikzpicture}[scale=1]
 
\begin{axis}[
    width=10cm,
    height=6.5cm,
    xmin=0.3,
    xmax=11,
    ymin=0.0,
    ybar=1pt,
    bar width=8pt,
    xlabel={\color{white}{position}},
    axis lines = left,
    axis line style={->, thick},
    legend style={font=\Huge,draw=none},
    legend image post style={scale=2.5},
	tick label style={font=\Large},
	label style={font=\Large}
]
\addplot[blue!30, fill, postaction={
        pattern=north east lines, pattern color=blue
    }] coordinates {
    (1, 0.34) (2, 0.204) (3, 0.159) (4, 0.156) (5, 0.075) (6, 0.027) (7, 0.02) (8, 0.015) (9, 0.015) (10, 0.01)
};
\addlegendentry{$P$}

\addplot[red!42, fill,  postaction={
        pattern=bricks, pattern color=red
    }] coordinates {
    (1, 0.35) (2, 0.25) (3, 0.1) (4, 0.2) (5, 0.1) (6, 0.0) (7, 0.0) (8, 0.0) (9, 0.0) (10, 0.0) 
};
\addlegendentry{$\hat P$}

\addplot[line width=0pt, samples=50, smooth, Cyan!75,fill=yellow, fill opacity=0.25] coordinates {(5.45, 0.05) (10.6, 0.05)} \closedcycle;
\addplot[ line width=4pt, samples=50, smooth, Cyan!75] coordinates {(5.47, -0.05) (5.47, 0.05)};
\addplot[ line width=4pt, samples=50, smooth, Cyan!75] coordinates {(10.6, -0.05) (10.6, 0.05)};
\addplot[ line width=4pt, samples=50, smooth, Cyan!75] coordinates {(5.47, 0.05) (10.6, 0.05)};

\tikzstyle{textbf} = [draw,rectangle,text width=5cm,text centered]
\node[color=NavyBlue, font=\itshape\bfseries\LARGE] at (axis cs: 8.0, 0.07) {Missing mass};

\end{axis}
 
\end{tikzpicture}}
    \adjustbox{max width=0.38\textwidth}{\begin{tikzpicture}[scale=1]
    \draw[draw=white] (0, 0) rectangle ++(0.3,1);
    \draw [-stealth,line width=1.1pt](0,1.6) -- (2.4,1.6);
    \node at (1.1,1.8) {\scriptsize Krichevsky-Trofimov};
\end{tikzpicture}}
    \adjustbox{max width=0.32\textwidth}{\begin{tikzpicture}[scale=1]
 
\begin{axis}[
    width=10cm,
    height=6.5cm,
    xmin=0.3,
    xmax=11,
    ymin=0.0,
    ybar=1pt,
    bar width=8pt,
    xlabel={\color{white}{position}},
    axis lines = left,
    axis line style={->, thick},
    legend style={font=\Huge,draw=none},
    legend image post style={scale=2.5},
	tick label style={font=\Large},
	label style={font=\Large}
]
\addplot[blue!30, fill, postaction={
        pattern=north east lines, pattern color=blue
    }] coordinates {
    (1, 0.34) (2, 0.204) (3, 0.159) (4, 0.156) (5, 0.075) (6, 0.027) (7, 0.02) (8, 0.015) (9, 0.015) (10, 0.01)
};
\addlegendentry{$P$}

\addplot[red!42, fill,  postaction={
        pattern=bricks, pattern color=red
    }] coordinates {
    (1, 0.3) (2, 0.22) (3, 0.1) (4, 0.18) (5, 0.1) (6, 0.02) (7, 0.02) (8, 0.02) (9, 0.02) (10, 0.02) 
};
\addlegendentry{$\hat P_{\text{\large KT}}$}

\end{axis}
 
\end{tikzpicture}}
    \caption{
    The empirical estimator with missing mass and the Krichevsky-Trofimov estimator.
    }
    \vspace{-0.15in}
    \label{fig:main:miss-mass:smoothing}
\end{figure}
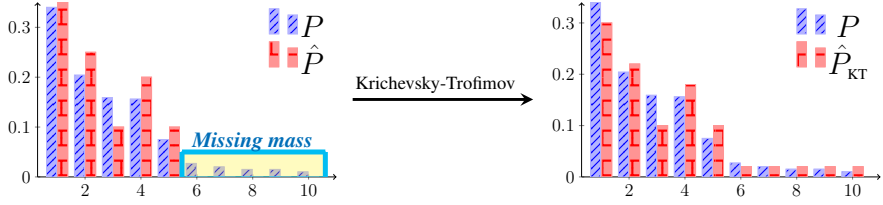

In this section, we summarize our main theoretical results.
The results hold for both the linearized cost $\lerror{\lambda}$ and the frontier integral $\mray$, we focus on $\mray$ here due to space constraints.
For $P, Q \in \Pcal(\Xcal)$,
let $\{X_i\}_{i=1}^n$ and $\{Y_i\}_{i=1}^n$ be i.i.d.~samples from $P$ and $Q$, respectively, and
denote by $\hat P_n$ and $\hat Q_n$ the respective empirical measures of $P$ and $Q$. 
The two samples are assumed to have the same size $n$ for simplicity.
We denote by $C$ an absolute constant which can vary from line to line.
The precise statements and proofs can be found in the Appendix.

\myparagraph{Sample complexity for the frontier integral}
We are interested in deriving a non-asymptotic bound for the absolute error of the empirical estimator, i.e., $\big|\mray(\hat P_n, \hat Q_n) - \mray(P, Q)\big|$.
When both $P$ and $Q$ are supported on a finite alphabet with $k$ items, a natural strategy is to exploit the smoothness properties of $\mray$, giving a na\"ive upper bound $O(L \sqrt{k/n})$ on the absolute error, where $L = \log{1/p_{*}}$ with $p_{*} = \min_{a \in \Supp{P}} P(a)$ reflecting the smoothness of $\mray$.
The dependency on $p_{*}$ requires $P$ to have a finite support and a short tail.
However, in many real-world applications, the distributions can either be supported on a countable set or have long tails~\cite{chen1999empirical,wang2017learning}.
By considering the \emph{missing mass} in the sample, we are able to obtain a high probability bound that is independent of $p_{*}$.
\begin{theorem}\label{thm:tail_bound_ray}
    Assume that $P$ and $Q$ are discrete and let $k = \max\{\abs{\Supp{P}}, \abs{\Supp{Q}}\} \in \mathbb{N} \cup \{\infty\}$.
    For any $\delta \in (0, 1)$, it holds that, with probability at least $1-\delta$,
    \begin{align}\label{eq:tail_bound_ray_oracle}
      \abs{\mray(\hat P_n, \hat Q_n) - \mray(P, Q)} \le C \left[ \bigg( \sqrt{\frac{\log{1/\delta}}{n}} + \alpha_n(P) + \alpha_{n}(Q) \bigg)\log{n} + \beta_{n}(P) + \beta_{n}(Q) \right]\,,
    \end{align}
    where $\alpha_n(P) = \sum_{a \in \Xcal} \sqrt{n^{-1} P(a)}$ and $\beta_n(P) = \expect\big[ \sum_{a: \hat P_n(a) = 0} P(a) \max\left\{ 1, \log{(1/P(a))} \right\} \big]$.
    Furthermore, if the support size $k < \infty$, then $\alpha_n(P) \le \sqrt{k/n}$ and $\beta_n(P) \le k\log{n}/n$.
    In particular, with probability at least $1 - \delta$,
    \begin{align}\label{eq:tail_bound_ray}
        \abs{\mray(\hat P_n, \hat Q_n) - \mray(P, Q)} \le C \left[ \sqrt{\frac{\log{1/\delta}}{n}} + \sqrt{\frac{k}{n}} + \frac{k}{n} \right] \log{n}\,.
    \end{align}
\end{theorem}

\begin{figure}[t]
    \centering
    \adjustbox{max width=0.3\textwidth}{\includegraphics{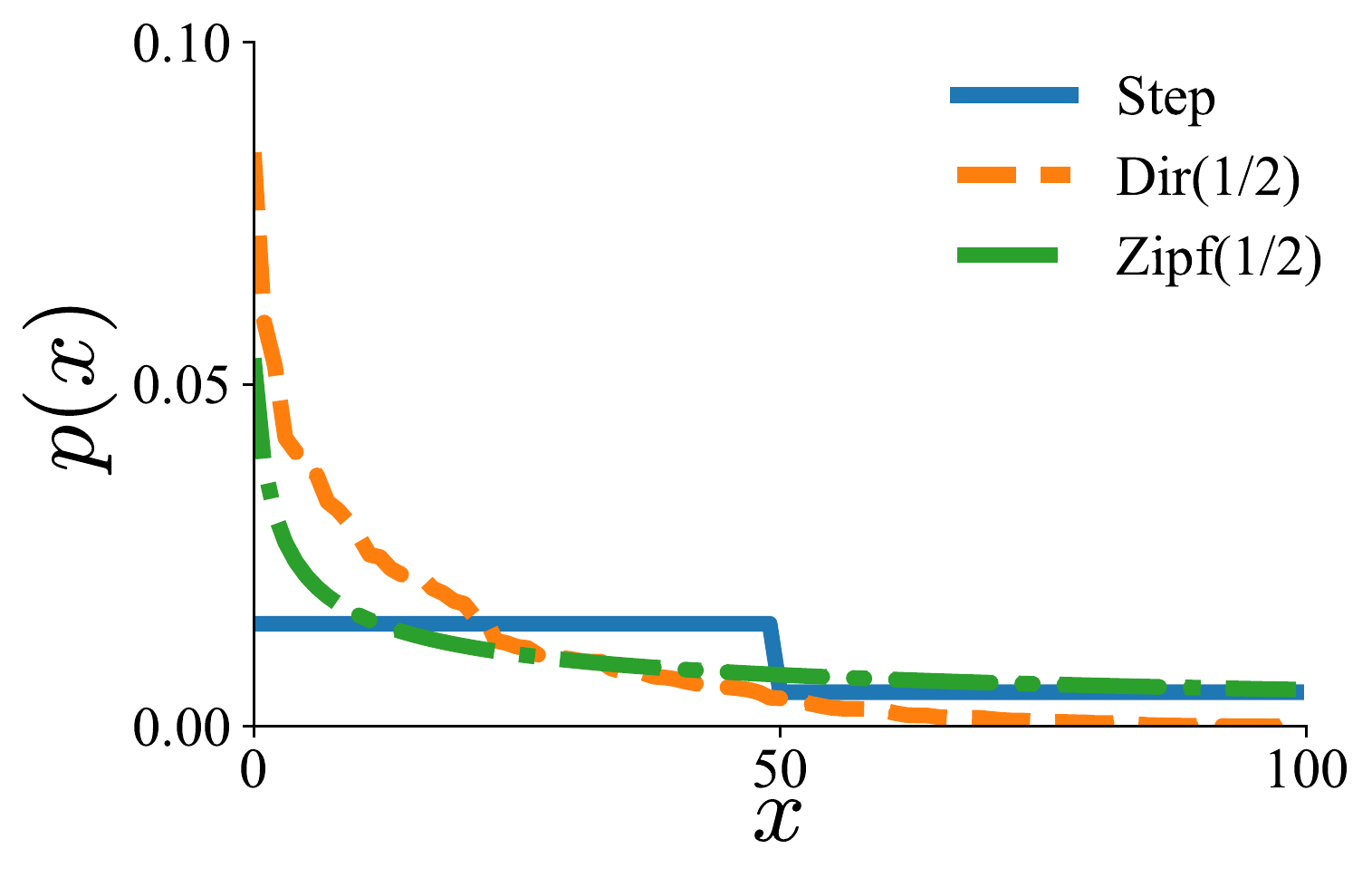}}
    \adjustbox{max width=0.35\textwidth}{\begin{tikzpicture}[scale=3]
\draw[->] (0,0) -- (1.2,0) node[below, xshift=0cm, yshift=-0.13cm] {$x$};
\draw[->] (0,0) -- (0,0.8) node[left, xshift=0cm, yshift=0] {$f\big(\tfrac{p(x)}{q(x)}\big)$};    

\def\fn{\x, {
    0.65 * exp( -((\x-0.66)^2)/ (2 * 0.08 * 0.08) ) + 
    0.35 * exp( -((\x-0.25)^2)/ (2 * 0.05 * 0.05) )
    }}
\fill [blue!20] (0.001,0.001) rectangle (0.1849,0.66);
\fill [yellow!70] (0.1849,0.001) rectangle (0.3151,0.66);
\fill [blue!20] (0.3151,0.001) rectangle (0.5230,0.66);
\fill [yellow!70] (0.5230,0.001) rectangle (0.5812,0.66);
\fill [red!28] (0.5812,0.001) rectangle (0.7388,0.66);
\fill [yellow!70] (0.7388,0.001) rectangle (0.7970,0.66);
\fill [blue!20] (0.7970,0.001) rectangle (1.1,0.66);
\draw[thick, color=black, domain=0:1.1,smooth] plot (\fn) ;
\draw[thick, densely dotted, color=black!70, smooth] (0.001, 0.66) -- (1.1, 0.66) ;
\draw[thick, densely dotted, color=black!70, smooth] (0.001, 0.4) -- (1.1, 0.4) ;
\draw[thick, densely dotted, color=black!70, smooth] (0.001, 0.15) -- (1.1, 0.15) ;

\node at (-0.1, 0.01) {\small $T_0$};
\node at (-0.1, 0.15) {\small $T_1$};
\node at (-0.1, 0.4) {\small $T_2$};
\node at (-0.1, 0.65) {\small $T_3$};

\end{tikzpicture}}
    \adjustbox{max width=0.3\textwidth}{\begin{tikzpicture}
	\begin{axis}[
	    width=10cm,
	    height=6.5cm,
	    xmin=0,
		ymin=0,
		ymax=0.55,
		xlabel=$t$,
		ylabel=$f(t)$,
		axis lines = left,
		axis line style={->, thick},
		legend style={font=\huge,draw=none},
		tick label style={font=\Large},
		label style={font=\huge}
	]
	\addplot[line width=3.5pt, red!80, domain=0.001:2.7]
		{x*ln((x/(0.5*x+0.5)))-0.5*(x-1)};
	\addlegendentry{$f_{\mathrm{KL}, \lambda}(t)$}
	
	\addplot[line width=3.5pt, blue!80, domain=0.001:2.7, dashed]
		{-ln((0.5*x+0.5))+0.5*(x-1)};
	\addlegendentry{$f^*_{\mathrm{KL}, \lambda}(t)$}
	\end{axis}
\end{tikzpicture}}
    \caption{
    \textbf{Left}:
    Tail decay of three distributions.
    \textbf{Middle}:
    Oracle quantization into $3$ bins: blue, yellow and red. Bin $i$ is given by the set $\{x \, :\, \fdiv(p(x)/q(x)) \in [T_{i-1}, T_{i})\}$.
    \textbf{Right}: The generator and conjugate generator of $\klam{\lambda}$ at $\lambda = 1/2$. 
    }
    \vspace{-0.15in}
    \label{fig:main:proof:illustration}
\end{figure}
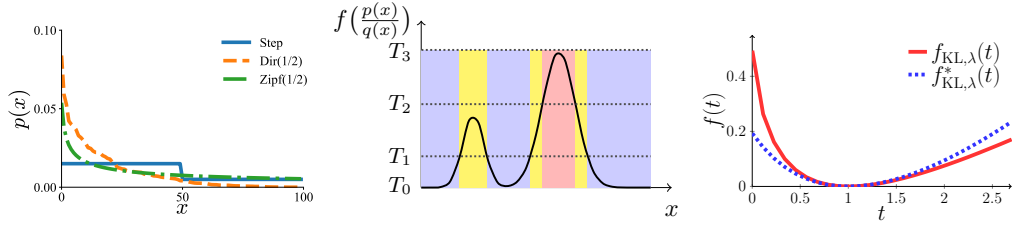

Before we discuss the bounds in \Cref{thm:tail_bound_ray}, let us introduce the missing mass problem.
This problem was first studied by Good and Turing~\cite{good1953frequency}, where the eponymous Good-Turing estimator was proposed to estimate the probability that a new observation drawn from a fixed distribution has never appeared before, in other words, is missing in the current sample; see \Cref{fig:main:miss-mass:smoothing} (left) for an illustration.
The Good-Turing estimator has been widely used in language modeling~\cite{Katz1987estimation,Church1991comparison,chen1999empirical} and studied in theory~\cite{McAllester2000convergence,orlitsky2003good,orlitsky2015turing}.
An inspiring result coming from this line of work is that the missing mass in a sample of size $n$ concentrates around its expectation~\cite{mcallester2005concentration}, which itself decays as $O(k/n)$ when the distribution is supported on $k$ items~\cite{berend2012missing}.

There are several merits to \Cref{thm:tail_bound_ray}. 
First, \eqref{eq:tail_bound_ray_oracle} holds for any distributions with a countable support.
Second, it does not depend on $p_{*}$ and is adapted to the tail behavior of $P$ and $Q$.
For instance, if $P$ is defined as $P(a) \propto a^{-2}$ for $a \in [k]$, then $\alpha_n(P) \propto (\log{k})/\sqrt{n}$, which is much better than $\sqrt{k/n}$ in \eqref{eq:tail_bound_ray} in terms of the dependency on $k$.
This phenomenon is also demonstrated empirically in \Cref{sec:experiments}.
Third, it captures a parametric rate of convergence, i.e., $O(n^{-1/2})$, up to a logarithmic factor.
In fact, this rate is not improvable in a related problem of estimating $\kl(P \Vert Q)$, even with the assumption that $P/Q$ is bounded~\cite{bu2018kl}.
The bound in \eqref{eq:tail_bound_ray} is a distribution-free bound, assuming $k$ is finite.
Note that it also gives an upper bound on the sample complexity by setting the right hand side of \eqref{eq:stat_error_ray} to be $\epsilon$ and solve for $n$, this is roughly $O((\sqrt{\log{1/\delta}} + \sqrt{k})^2/\epsilon^2)$.

The proof of \Cref{thm:tail_bound_ray} relies on two new results: (a) a concentration bound around $\expect[\mray(\hat P_n, \hat Q_n)]$, which can be obtained by McDiarmid's inequality, and (b) an upper bound for the statistical error, i.e., $\expect\big| \mray(\hat P_n, \hat Q_n) - \mray(P, Q) \big|$, which is upper bounded by
\begin{align}\label{eq:stat_error_ray}
    O\left([\alpha_n(P) + \alpha_n(Q)]\log{n} + \beta_n(P) + \beta_n(Q) \right) \le O\big((\sqrt{k/n} + k/n)\log{n}\big)\,.
\end{align}
The concentration bound gives the term $\sqrt{n^{-1} \log{1/\delta}}$.
The statistical error bound is achieved by splitting the masses of $P$ and $Q$ into two parts: one that appears in the sample and one that never appears.
The first part can be controlled by a Lipschitz-like property of the frontier integral, leading to the term $\alpha_n(P) + \alpha_n(Q)$, and the second part, $\beta_n(P) + \beta_n(Q)$, falls into the missing mass framework.
In addition, the rate $k/n$ for $\beta_n$ shown here matches the rate for the missing mass.

\myparagraph{Statistical consistency of the divergence frontiers}
While \Cref{thm:tail_bound_ray} establishes the consistency of the frontier integral, it is also of great interest to know whether the divergence frontier itself can be consistently estimated.
In fact, similar bounds hold for the worst-case error of $\Fcal(\hat P_n, \hat Q_n)$.
\begin{corollary}\label{cor:consis_df}
    Under the same assumptions as in \Cref{thm:tail_bound_ray}, the bounds in \eqref{eq:tail_bound_ray_oracle} and \eqref{eq:tail_bound_ray} hold for
    \begin{align*}
        \sup_{\lambda \in [\lambda_0, 1 - \lambda_0]} \big\lVert \big( \kl(\hat P_n \Vert \hat R_\lambda), \kl(\hat Q_n \Vert \hat R_\lambda) \big) - \big( \kl(P \Vert R_\lambda), \kl(Q \Vert R_\lambda) \big) \big\rVert_1,
    \end{align*}
    where $\hat R_\lambda := \lambda \hat P_n + (1 - \lambda) \hat Q_n$,
    with $C$ replaced by $C / \lambda_0$ for any $\lambda_0 \in (0, 1)$.
    In particular, if $\lambda_0$ is chosen as $\lambda_n = o(1)$ and $\lambda_n = \omega(\sqrt{k/n} \log{n})$, then the expected worst-case error above converges to zero at rate $O(\lambda_n^{-1} \sqrt{k/n} \log{n})$.
\end{corollary}
The truncation in \Cref{cor:consis_df} is necessary without imposing additional assumptions, since $\kl(P \Vert R_\lambda)$ is close to $\kl(P \Vert Q)$ for small $\lambda$ and it is known that the minimax quadratic risk of estimating the KL divergence over all distributions with $k$ bins is always infinity~\cite{bu2018kl}.

\myparagraph{Upper bound for the quantization error}
Recall from \Cref{sec:metrics} that computing the divergence frontiers in practice usually involves a quantization step.
Since every quantization will inherently introduce a positive bias in the estimation procedure, it is desirable to control the error, which we call the quantization error, induced by this step.
We show that there exists a quantization scheme with error proportional to the inverse of its level.
We implement this scheme and empirically verify this rate in \Cref{sec:a:experiments}; certain regimes appear to show even faster convergence.

Let $\Xcal$ be an {arbitrary} measurable space and $\Scal$ be a partition of $\Xcal$.
The quantization error of $\Scal$ is the difference $\abs{\mray(P_{\Scal}, Q_{\Scal}) - \mray(P, Q)}$.
It can be shown that there exists a distribution-dependent partition $\Scal_k$ with level $|\Scal_k|=k$ whose quantization error is no larger than the inverse of its level, i.e.,
\begin{align}\label{eq:quant_error}
    \abs{\mray(P, Q) - \mray(P_{\Scal_k}, Q_{\Scal_k})} \le C k^{-1}.
\end{align}
The key idea behind the construction of this partition is visualized in \Cref{fig:main:proof:illustration} (middle).
Combining this bound with the bounds in \eqref{eq:stat_error_ray} leads to the following bound for the total estimation error.
\begin{theorem}\label{thm:est_error_ray}
    Assume that $\Scal_k$ is a partition of $\Xcal$ such that $\abs{\Scal_k} = k \ge 2$.
    Then the total error $\expect\abs{\mray(\hat P_{\Scal_k, n}, \hat Q_{\Scal_k, n}) - \mray(P, Q)}$ is upper bounded by
    \begin{align}\label{eq:est_error_ray}
        C \big[ \left( \alpha_n(P) + \alpha_n(Q) \right) \log{n} + \beta_n(P) + \beta_n(Q) + \abs{\mray(P, Q) - \mray(P_{\Scal_k}, Q_{\Scal_k})} \big].
    \end{align}
    Moreover, if the quantization error satisfies the bound in \eqref{eq:quant_error}, we have
    \begin{align}\label{eq:est_error_ray_bound}
        \expect\abs{\mray(\hat P_{\Scal_k, n}, \hat Q_{\Scal_k, n}) - \mray(P, Q)} \le C \left[ \left( \sqrt{k/n} + k/n \right) \log{n} + 1/k \right].
    \end{align}
\end{theorem}

Based on the bound in \eqref{eq:est_error_ray}, a good choice of $k$ is $\Theta(n^{1/3})$ which balances between the two types of errors.
We illustrate in \Cref{sec:experiments} that this choice works well in practice.
This balancing is enabled by the existence of a good quantizer with a distribution-free bound in \eqref{eq:quant_error}.
In practice, this suggests a data-dependent quantizer using nonparametric density estimators.
However, directions such as kernel density estimation \cite{ritter2002quantizing,hegde2004vector,van1999faithful} and nearest-neighbor methods \cite{alamgir2014density} have not met empirical success, as they suffer from the curse of dimensionality common in nonparametric estimation.
In particular, 
\cite{wang2005divergence,silva2007universal,silva2010information} propose quantized divergence estimators but only prove asymptotic consistency, and little progress has been made since then. 
On the other hand, modern data-dependent quantization techniques based on deep neural networks can  successfully estimate properties of the density from high dimensional data \cite{sablayrolles2018spreading,hamalainen2020deep}.
Theoretical results for those techniques could complement our analysis.
We leverage these powerful methods to scale our approach on real data in \Cref{sec:experiments}.

\section{Towards better estimators and interpolated $f$-divergences} \label{sec:improvements}
\myparagraph{Smoothed distribution estimators}
When the support size $k$ is large, the statistical performance of the empirical estimator considered in the previous section can be improved.
To overcome this challenge, practitioners often use more sophisticated distribution estimators such as the Good-Turing estimator~\citep{good1953frequency,orlitsky2015turing} and add-constant estimators~\citep{krichevsky1981performance,braess2004bernstein}.
We focus on the add-constant estimator defined below and state here its estimation error when it is applied to estimate the frontier integral from data.
Again, this result also holds for the linearized cost $\lerror{\lambda}$.
We investigate and compare the performance of various distribution estimators in \Cref{sec:experiments}.

For notational simplicity, we assume that $P$ and $Q$ are supported on a common finite alphabet with size $k < \infty$.
Note that this is true for the quantized distributions $P_{\Scal}$ and $Q_{\Scal}$.
For any constant $b > 0$, the add-constant estimator of $P$ is defined as $\hat P_{n, b}(a) = (N_a + b)/(n + kb)$ for each $a \in \Supp{P}$,
where $N_a = \abs{\{i: X_i = a\}}$ is the number of times $a$ appears in the sample.

Thanks to the smoothing, there is no mass missing in the add-constant estimator.
This effect is illustrated for the Krichevsky-Trofimov (add-$1/2$) estimator in \Cref{fig:main:miss-mass:smoothing}.
As a result, we can directly utilize the smoothness properties of the frontier integral to get the following bound.
\begin{proposition}\label{prop:stat_error_mauveray_add}
    Under the same assumptions as in \Cref{thm:est_error_ray}, we have
    \begin{align}\label{eq:est_error_ray_add}
        &\quad \expect\abs{\mray(\hat P_{\Scal_k, n, b}, \hat Q_{\Scal_k, n, b}) - \mray(P, Q)} \nonumber \\
        &\le C \left[ \left(\frac{n(\alpha_n(P) + \alpha_n(Q))}{n + bk} + \gamma_{n,k}(P) + \gamma_{n,k}(Q) \right) \log{(n/b+k)} + \frac1k \right],
    \end{align}
    where $\gamma_{n,k}(P) = (n + bk)^{-1} bk \sum_{a \in \Xcal} \abs{P(a) - 1/k}$.
    It can be further upper bounded by $\frac{\sqrt{nk} + bk}{n + bk} \log{(n/b + k)} + \frac{1}{k}$ up to a multiplicative constant.
\end{proposition}

Let us compare the bounds in \Cref{prop:stat_error_mauveray_add} with the ones in \Cref{thm:est_error_ray}.
For the distribution-dependent bound, the term $\alpha_n(P) \log{n}$ in \eqref{eq:est_error_ray} is improved by a factor $n/(n+bk)$ in \eqref{eq:est_error_ray_add}.
The missing mass term $\beta_n(P)$ is replaced by the total variation distance between $P$ and the uniform distribution on $[k]$ with a factor $bk/(n + bk)$.
The improvements in both two terms are most significant when $k/n$ is large.
As for the distribution-free bound,
when $k/n$ is small, the bound in \Cref{prop:stat_error_mauveray_add} scales the same as the one in \eqref{eq:est_error_ray_bound};
when $k/n$ is large (i.e., bounded away from $0$ or diverging), it scales as $O(\log{n} + \log{(k/n)} + k^{-1})$ while the one in \eqref{eq:est_error_ray_bound} scales as $O(k\log{n}/n + k^{-1})$.
Given the improvement, it would be an interesting venue for future work to consider adaptive estimators in the spirit of \cite{goldenshluger2009structural}.

\myparagraph{Generalization to $\fdiv$-divergences}
Estimation of the $\chi^2$ divergence is useful for variational inference~\cite{dieng2017variational} and GAN training~\cite{mao2017least,tao2018chisquared}.
More generally, estimating $\fdiv$-divergences from samples is a fundamental problem in machine learning and statistics~\cite{nguyen2010estimating,im2018quantitatively,chen2018variational,rubenstein2019practical}.
We extend our previous results to estimating general $\fdiv$-divergences (which satisfy some regularity assumptions) using the same two-step procedure of quantization and estimation of multinomial distributions from samples.
 
We start by reviewing the definition of $\fdiv$-divergences.
Let $f: (0, \infty) \rightarrow \mathbb{R}$ be a nonnegative and convex function with $f(1) = 0$.
Let $P, Q \in \Pcal(\Xcal)$ be dominated by some measure $\mu \in \Pcal(\Xcal)$ with densities $p$ and $q$, respectively.
The $\fdiv$-divergence generated by $\fdiv$ is defined as
\[
    \Df{P}{Q} = \int_{\Xcal} q(x) f\left( \frac{p(x)}{q(x)} \right) \D \mu(x) \,,
\]
with the convention that $f(0) = f(0^+)$ and $0 f(p/0) = p \ftil(0)$, where $\ftil(0) = \ftil(0^+) \in [0, \infty]$ for $\ftil(t) = tf(1/t)$.
We call $\ftil$ the conjugate generator to $\fdiv$.
An illustration of the generator to $\klam{1/2}$ can be found in \Cref{fig:main:proof:illustration} (right).
Note that the conjugacy here is unrelated to the convex conjugacy but is based on the {\em perspective transform}.
The function $\ftil$ also generates an $f$-divergence, which is referred to as the \emph{conjugate divergence} to $D_\fdiv$ since $\Dftil{P}{Q} = \Df{Q}{P}$.

The quantization error bound \eqref{eq:quant_error} holds for all $\fdiv$-divergences which are {\em bounded}, i.e., $\fdiv(0) + \ftil(0) < \infty$.
The high probability bounds in \Cref{thm:tail_bound_ray} also hold for $\fdiv$-divergences, under some regularity assumptions: (a) $\abs{\fdiv'(t)} \propto \log{t^{-1}}$ and $\abs{\ftilg(t)} \propto \log{t\inv}$ for small $t$, which guarantees that $\fdiv$ is approximately Lipschitz and cannot vary too fast;
(b) $t\fdiv''(t)$ and $t\ftilh(t)$ are bounded, which is a technical assumption that helps control the variation of $\fdiv$ around zero.
The interpolated $\chi^2$ divergence, defined analogously as the interpolated KL divergence, satisfies these conditions.

In the Appendix, we prove all the results for general $f$-divergences and show that both the frontier integral and the linearized cost are $\fdiv$-divergences satisfying the regularity conditions, recovering \Cref{thm:tail_bound_ray} and \Cref{thm:est_error_ray} as special cases.

\section{Experiments}\label{sec:experiments}
\begin{figure}[t]
  \centering
  \includegraphics[width=\textwidth]{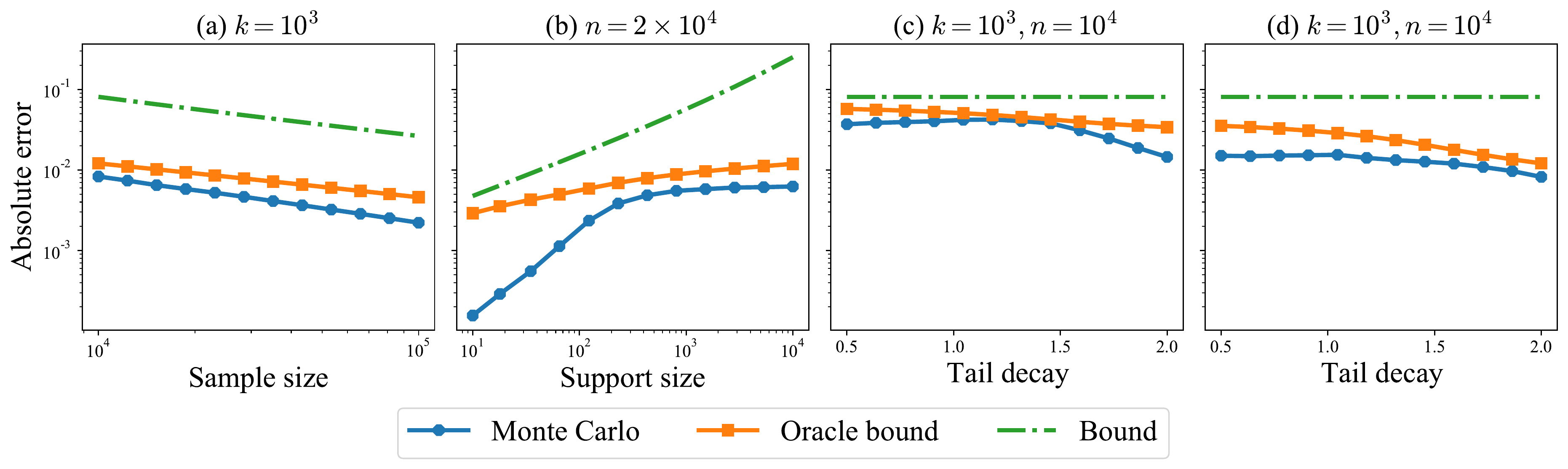}
  \caption{Statistical error of the estimated frontier integral on synthetic data. \textbf{(a)}: Zipf$(2)$ and Zipf$(2)$ with $k = 10^3$; \textbf{(b)}: Zipf$(2)$ and Zipf$(2)$ with $n = 2\times 10^4$; \textbf{(c)}: $\dir(\mathbf{1})$ and Zipf$(r)$ with $k = 10^3$ and $n = 10^4$; \textbf{(d)}: $\zipf(2)$ and $\zipf(r)$ with $k = 10^3$ and $n = 10^4$. The bounds are scaled by $100$.}
  \label{fig:main:bound:new}
  \vspace{-0.2in}
\end{figure}

We investigate the empirical behavior of the divergence frontier and the frontier integral on both synthetic and real data.
Our main findings are: (a) the statistical error bound approximately reveals the rate of convergence of the empirical estimator;
(b) the smoothed distribution estimators improve the estimation accuracy;
(c) the quantization level suggested by the theory works well empirically.
The results for the divergence frontier and the frontier integral are almost identical.
We focus on the latter here due to space constraints.
In all the plots, we visualize the average absolute error computed from 100 repetitions with shaded region denoting one standard deviation around the mean.
More details and additional results, including the ones for the divergence frontier, are deferred to \Cref{sec:a:experiments}.
The code to reproduce the experiments is available online\footnote{\url{https://github.com/langliu95/divergence-frontier-bounds}.}.

\begin{figure}[t]
    \centering
    \includegraphics[width=\textwidth]{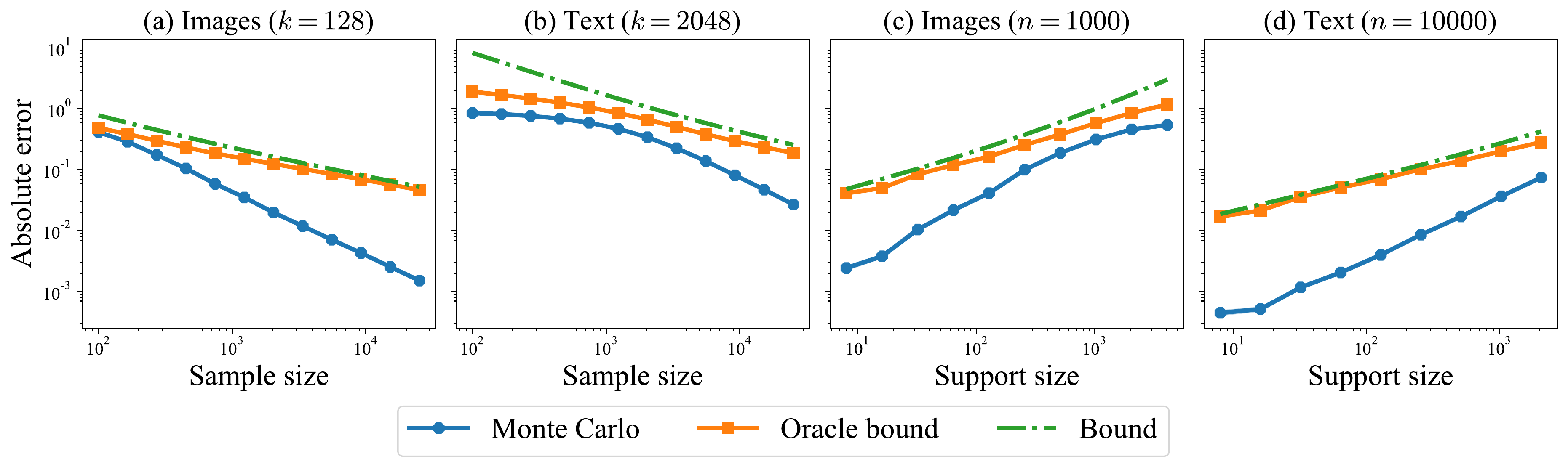}\\
    \caption{Statistical error of the estimated frontier integral on real data. \textbf{(a)}: Image data (CIFAR-10) with $k = 128$; \textbf{(b)}: Text data (WikiText-103) with $k = 2048$; \textbf{(c)}: Image data (CIFAR-10) with $n = 1000$; \textbf{(d)}: Text data (WikiText-103) with $n = 10000$. The bounds are scaled by $30$.}
    \label{fig:main:real_bound:new}
    \vspace{-0.2in}
\end{figure}

\subsection{Experimental setup}
We work with synthetic data in the case when $k = \abs{\Xcal} < \infty$ as well as real image and text data. 

\myparagraph{Synthetic Data}
Following the experimental settings in~\cite{orlitsky2015turing}, we consider three types of distributions: (a) the Zipf$(r)$ distribution with $r \in \{0, 1, 2\}$ where $P(i) \propto i^{-r}$. Note that Zipf$(r)$ is regularly varying with index $-r$; see,~e.g.,~\cite[Appx. B]{shorack2000probability}.
(b) the Step distribution where $P(i) = 1/2$ for the first half bins and $P(i) = 3/2$ for the second half bins. (c) the Dirichlet distribution $\dir(\alpha)$ with $\alpha \in \{\mathbf{1}/2, \mathbf{1}\}$; see \Cref{fig:main:illustration} (left) for an illustration.
In total, there are 6 different distributions, giving $21$ different pairs of $(P, Q)$.
For each pair $(P, Q)$, we generate i.i.d.~samples of size $n$ from each of them, and estimate the divergence frontier as well as the frontier integral from these samples. 

\myparagraph{Real Data}
We consider two domains: images and text. 
For the image domain, we train a
StyleGAN2~\cite{karras2020ada} on the CIFAR-10 dataset~\cite{krizhevsky2009learning} using the publicly available code\footnote{
\url{https://github.com/NVlabs/stylegan2-ada-pytorch}.
} with default hyperparameters.
To evaluate the divergence frontiers,
we use the test set of 10k images as the target distribution $P$ 
and we sample 10k images from the generative model as the model distribution $Q$. 
For the text domain, we fine-tune a pretrained GPT-2~\cite{radford2019language} model with 124M parameters (i.e., GPT-2 small)
on the Wikitext-103 dataset~\cite{merity2017pointer}.
We use the open-source HuggingFace Transformers library~\cite{wolf2020transformers} for training, and generate 10k 500-token completions using top-$p$ sampling and 100-token prefixes.

We take the following steps to compute the frontier integral.
First, we represent each image/text by its features~\cite{heusel2017gans,sajjadi2018assessing,kynknniemi2019improved}.
Second, we learn a low-dimensional feature embedding which maintains the neighborhood structure of the data while encouraging the features to be uniformly distributed on the unit sphere~\cite{sablayrolles2018spreading}.
Third, we quantize these embeddings on a uniform lattice with $k$ bins.
For each support size $k$, this gives us quantized distributions $P_{\Scal_k}$ and $Q_{\Scal_k}$.
Finally, we sample $n$ i.i.d. observations from each of these distributions and consider the empirical distributions $\hat P_{\Scal_k, n}$ and $\hat Q_{\Scal_k, n}$ as well as the smoothed distribution estimators computed from these samples.

\myparagraph{Performance Metric}
We are interested in the estimation of the frontier integral $\mray(P, Q)$ using estimators $\mray(\hat P_n, \hat Q_n)$ for the empirical estimator as well as the smoothed distribution estimator.
We measure the quality of estimation using the absolute error, which is defined as $\lvert \mray(\hat P_n, \hat Q_n) - \mray(P, Q) \rvert$.
For the real data, we measure the error of estimating $\mray(P_{\Scal_k}, Q_{\Scal_k})$ by $\mray(\hat P_{\Scal_k, n}, \hat Q_{\Scal_k, n})$. 

\subsection{Tightness of the Statistical Bound}

We investigate the tightness of the statistical error bound of \Cref{thm:tail_bound_ray} with respect to the sample size $n$ and the support size $k$, in order to verify the validity of the theory in practically relevant settings. 

We estimate the expected absolute error
$\expect|\mray(\hat P_n, \hat Q_n) - \mray(P, Q)|$
from a Monte Carlo estimate using $100$ random trials. 
We compare it with the following 
bounds in \Cref{thm:tail_bound_ray}:\footnote{
Specifically, we use the expected bound of \Cref{prop:fdiv:consistency} (\Cref{sec:a:plug-in}), from which \Cref{thm:tail_bound_ray} is derived.
}
\begin{enumerate}[label=(\alph*),itemsep=0em, topsep=0em, leftmargin=1.6em]
    \item \textbf{Bound}: the distribution independent bound
        $
        (\sqrt{k/n} + k/n ) \log n
    $.
    \item \textbf{Oracle Bound}: the distribution dependent bound
    $
        \left(\alpha_n(P) + \alpha_n(Q)\right)\log n + \beta_n(P) + \beta_n(Q)
    $.
    We assume that the quantities $\alpha_n$ and $\beta_n$ defined in \Cref{thm:tail_bound_ray} are known.
\end{enumerate}
We fix $k$, plot each of these quantities in a log-log plot with varying $n$ and compare their \emph{slopes}.\footnote{
A log-log plot of the function $f(x) = c x^{\gamma}$ is a straight line with slope $\gamma$, which thus captures the \emph{degree}.
}
We then repeat the experiment with $n$ fixed and $k$ varying. We often scale the bounds by a constant for easier visual comparison of the slopes; this only changes the intercept and leaves the slope unchanged.

\myparagraph{\Cref{thm:tail_bound_ray} is tight for synthetic data}
\Cref{fig:main:bound:new} gives the Monte Carlo estimate and the bounds of the statistical error for various synthetic data distributions.
In \Cref{fig:main:bound:new}(a), we observe that the bound has approximately the same slope as the Monte Carlo estimate, while the oracle bound has a slightly worse slope. In \Cref{fig:main:bound:new}(b), we observe that the oracle bound captures the correct rate for $k > 300$, while the distribution-independent bound captures the correct rate at small $k$. For the right two plots, both bounds capture the right rate over a wide range of tail decay.
The oracle bound is tighter for fast decay, where the distribution-independent bounds on $\alpha_n(Q)$ and $\beta_n(Q)$ can be very pessimistic.

\myparagraph{\Cref{thm:tail_bound_ray} is somewhat tight for real data}
\Cref{fig:main:real_bound:new} contains the analogous plot for real data, where the observations are similar. 
In \Cref{fig:main:real_bound:new}(b), we see that the oracle bound captures the right rate for small sample sizes where $k/n > 1$. However, for large $n$, the distribution-independent bound is better at matching the slope of the Monte Carlo estimate.
The same is true for \Cref{fig:main:real_bound:new}(c), where the oracle bound is better for large $k$.
For parts (a) and (d), however,
both bounds do not capture the right slope of the Monte Carlo estimate; \Cref{thm:tail_bound_ray} is not a tight upper bound in this case. 

\begin{figure}[t]
    \centering
    \includegraphics[width=\textwidth]{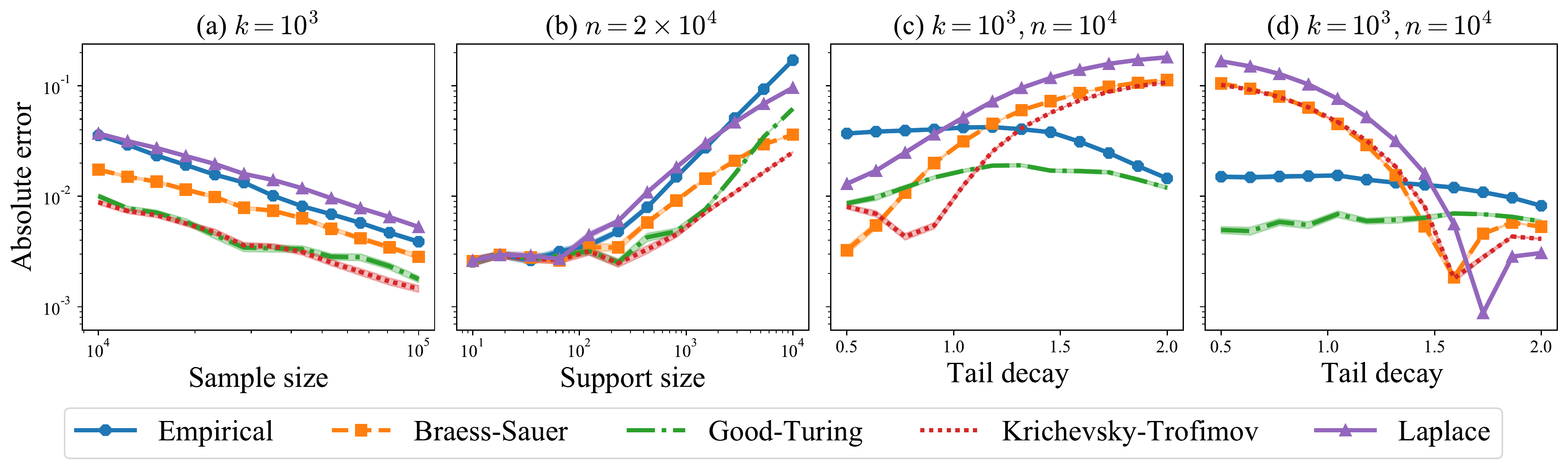}
    \caption{Statistical error with smoothed distribution estimators on synthetic data. \textbf{(a)}: $\zipf(0)$ and $\dir(\mathbf{1}/2)$ with $k = 10^3$; \textbf{(b)}: $\zipf(0)$ and $\dir(\mathbf{1}/2)$ with $n = 2\times 10^4$; \textbf{(c)}: $\dir(\mathbf{1})$ and $\zipf(r)$ with $k = 10^3$ and $n = 10^4$; \textbf{(d)}: $\zipf(2)$ and $\zipf(r)$ with $k = 10^3$ and $n = 10^4$.}
    \label{fig:main:smoothing:new}
\end{figure}

\begin{figure}[t]
    \centering
    \includegraphics[width=\textwidth]{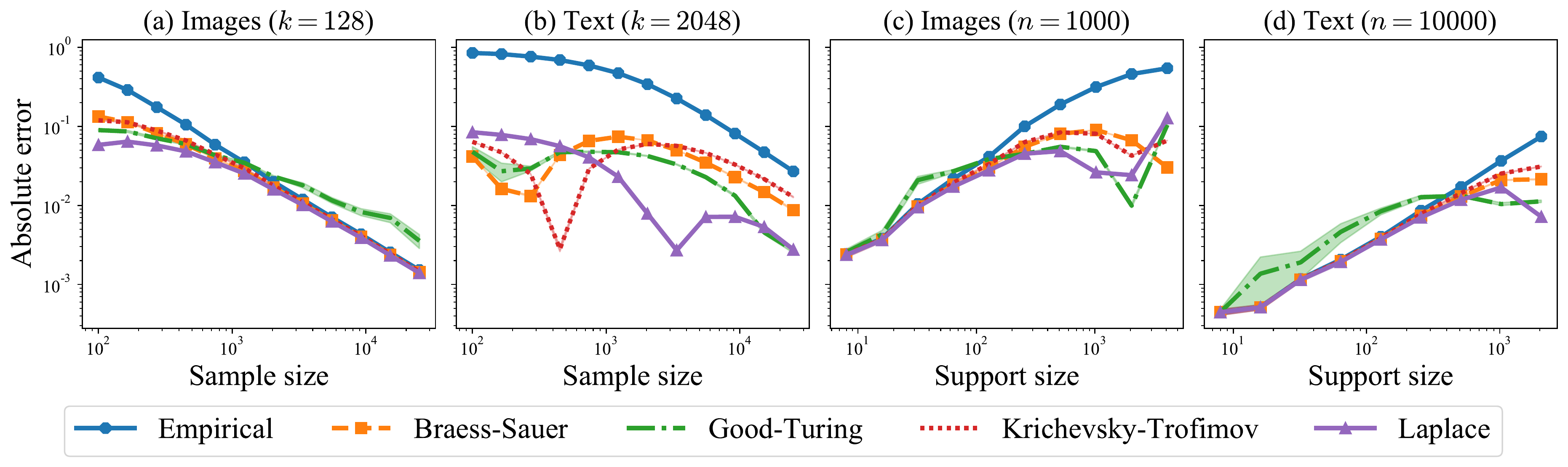}
    \caption{Statistical error with smoothed distribution estimators on real data. \textbf{(a)}: Image data (CIFAR-10) with $k = 128$; \textbf{(b)}: Text data (WikiText-103) with $k = 2048$; \textbf{(c)}: Image data (CIFAR-10) with $n = 1000$; \textbf{(d)}: Text data (WikiText-103) with $n = 10000$. The bounds are scaled by $15$.}
    \label{fig:main:real_bound:new:2}
    \vspace{-0.2in}
\end{figure}

\subsection{Effect of Smoothed Distribution Estimators}

We now show that smoothed estimators can lead to improved estimation over the na\"ive empirical estimator and thus improved sample complexity as shown in \Cref{prop:stat_error_mauveray_add}. This is practically significant in the context of generative models, since one can have an equally good estimate of the divergence frontier with fewer samples using smoothed estimators~\cite{sajjadi2018assessing,djolonga2020precision}.

Concretely, we compare the Monte Carlo estimates of the absolute error 
$\expect|\mray(\hat P_n, \hat Q_n) - \mray(P, Q)|$ for the plug-in estimate (denoted ``Empirical'') with the corresponding estimates for smoothed estimators. We consider 4 smoothed estimators as in~\cite{orlitsky2015turing}:
the (modified) \emph{Good-Turing} estimator, as well as three add-constant estimators: the \emph{Laplace}, \emph{Krichevsky-Trofimov} and \emph{Braess-Sauer} estimators.

\myparagraph{Smoothed estimators are more efficient than the empirical estimator}
We compare the smoothed estimators to the empirical one in \Cref{fig:main:smoothing:new}
on synthetic data and
\Cref{fig:main:real_bound:new:2} on real data. 
In general, the smoothed distribution estimators reduces the absolute error. For 
parts (a) and (b) of \Cref{fig:main:smoothing:new}, the Good-Turing and the Krichevsky-Trofimov estimators have the best absolute error. For parts (c) and (d), the Good-Turing estimator is adapted to various regimes of tail-decay, outperforming the empirical estimator. The Krichevsky-Trofimov and Braess-Sauer estimators, on the other hand, exhibit small absolute error for particular decay regimes.
The results are similar for real data in \Cref{fig:main:real_bound:new:2}.

\myparagraph{Practical guidance on choosing a smoothed estimator}
While
the smoothed estimators offer a marked improvement when $k/n$ is large (that is, close to 1), the best estimator is problem-dependent. As a rule of thumb, we suggest the Krichevsky-Trofimov estimator which works well in the large $k/n$ regime but is still competitive when $k/n$ is small (i.e., large $n$).

\subsection{Quantization Error}

Next, we study the effect of the quantization level $k$ on the total error.
We consider a simple $2$-dimensional synthetic setting where the distributions $P, Q$ are either the multivariate normal or $t$-distributions.
We use data-driven quantization with $k$-means to obtain a quantization $\Scal_k$: each component of the partition is the region corresponding to one cluster.
Finally, we plot the absolute error $\expect|\mray(P, Q) - \mray(\hat P_{\Scal_k, n}, \hat Q_{\Scal_k, n})|$, 
where the $\mray(P, Q)$ is computed using numerical integration and the expectation is estimated with Monte Carlo simulations.

\myparagraph{The choice $k=\Theta(n^{1/3})$ works the best}
We compare $k=n^{1/r}$ for $r=2, 3, 4, 5$ in \Cref{fig:main:quant}.
For small $n$, $r\ge 3$ all perform similarly, but $r=3$ clearly outperforms other choices for $n \ge 10^4$. While our theory does not directly apply for data-dependent partitioning schemes, the choice $k=\Theta(n^{1/3})$ suggested by \Cref{thm:est_error_ray} nevertheless works well in practice.

\begin{figure}[t]
    \centering
    \includegraphics[width=\textwidth]{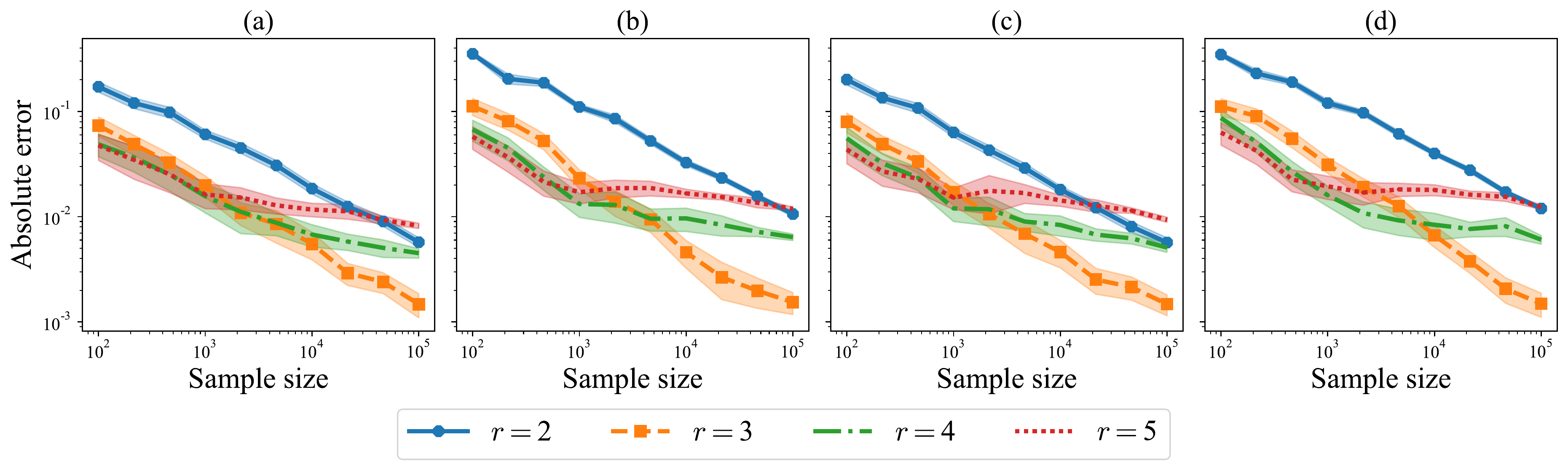}
    \caption{Total error with quantization level $k \propto n^{1/r}$ on 2-dimensional continuous data. \textbf{(a)}: $\mathcal{N}(0, I_2)$ and $\mathcal{N}(1, I_2)$; \textbf{(b)}: $\mathcal{N}(0, I_2)$ and $\mathcal{N}(0, 5I_2)$; \textbf{(c)}: $t_4(0, I_2)$ and $t_4(1, I_2)$ (multivariate t-distribution with 4 degrees of freedom); \textbf{(d)}: $t_4(0, I_2)$ and $t_4(0, 5I_2)$.}
    \label{fig:main:quant}
    \vspace{-0.2in}
\end{figure}

\section{Conclusion}\label{sec:conclusion}
In this paper, we study the statistical behavior of the divergence frontiers and the proposed integral summary estimated from data.
We decompose the estimation error into two components, the statistical error and the quantization error, to conform with the approximation procedure commonly used in practice.
We establish non-asymptotic bounds on both of the two errors.
Our bounds shed light on the optimal choice of the quantization level $k$---they suggests that the two errors can be balanced at $k = \Theta(n^{1/3})$.
We also derive a new concentration inequality for the frontier integral, which provides the sample complexity of achieving a small error in high probability.
Finally, we demonstrate both theoretically and empirically that the use of smoothed distribution estimators can improve the estimation accuracy.
All the results can be generalized to a large class of interpolation-based $f$-divergences.
Provided new theoretical results on modern data-dependent quantization schemes using deep neural networks, it would be an interesting direction for future work to specialize our bounds to such quantization schemes.
Extending our results to $\beta$-divergences could also be interesting.

\clearpage

\paragraph*{Acknowledgments.}
The authors would like to thank J. Thickstun for fruitful discussions. Part of this work was done while Z. Harchaoui was visiting the Simons Institute for the Theory of Computing. This work was supported by NSF DMS-2134012, NSF CCF-2019844, the CIFAR program ``Learning in Machines and Brains'', and faculty research awards.

\bibliographystyle{abbrvnat}
\bibliography{biblio}

\clearpage


\clearpage

\appendix

\addcontentsline{toc}{section}{Appendix} 
\part{Appendix} 
\parttoc 
\clearpage

\section{$\fdiv$-divergence: review and examples}
\label{sec:a:def_f_div}
We review the definition of $\fdiv$-divergences and give a few examples.

Let $\fdiv:(0, \infty) \to \mathbb{R}$ be a convex function
with $\fdiv(1) = 0$.
Let $P, Q \in \Pcal(\Xcal)$ be dominated by some measure $\mu \in \Pcal(\Xcal)$ with densities $p$ and $q$, respectively.
The $\fdiv$-divergence generated by $\fdiv$ is
\[
    \Df{P}{Q} = \int_{\Xcal} q(x) f\left( \frac{p(x)}{q(x)} \right) \D \mu(x) \,,
\]
with the convention that $f(0) := \lim_{t\to 0^+} f(t)$ and $0 f(p/0) = p \ftil(0)$, where $\ftil(0) = \lim_{x \rightarrow 0^+} xf(1/x) \in [0, \infty]$.
Hence, $\Df{P}{Q}$ can be rewritten as
\[
    \Df{P}{Q} = \int_{q > 0} q(x) f\left( \frac{p(x)}{q(x)} \right) \D \mu(x) + \ftil(0) P[q = 0] \,,
\]
with the agreement that the last term is zero if $P[q = 0] = 0$ no matter what value $\ftil(0)$ takes (which could be infinity).
For any $c \in \mathbb{R}$, it holds that $D_{f_c}(P \Vert Q) = \Df{P}{Q}$ where $f_c(t) = f(t) + c(t-1)$.
Hence, we also assume, w.l.o.g., that $f(t) \ge 0$ for all $t \in (0, \infty)$.
To summarize, $f$ is convex and nonnegative with $f(1) = 0$.
As a result, $f$ is non-increasing on $(0, 1]$ and non-decreasing on $[1, \infty)$.

The conjugate generator to $\fdiv$ is the function 
$\ftil: (0, \infty) \to [0, \infty)$ defined by\footnote{
The conjugacy between $\fdiv$ and $\ftil$ is unrelated to the usual Fenchel or Lagrange duality in convex analysis, but is related to the perspective transform.
}
\[
    \ftil(t) = t f(1/t) \,,
\]
where again we define $\ftil(0) = \lim_{t\to 0^+} \ftil(t)$.
Since $\ftil$ can be constructed by the perspective transform of $f$, it is also convex.
We can verify that $\ftil(1) = 0$ and $\ftil(t) \ge 0$ for all $t \in (0, \infty)$, so it defines another divergence $D_{\ftil}$.
We call this the \emph{conjugate divergence} to $D_{\fdiv}$ since
\[
    \Dftil{P}{Q} = \Df{Q}{P} \,.
\]
The divergence $D_\fdiv$ is symmetric if and only if $\fdiv = \ftil$, and we write it as $D_\fdiv(P, Q)$ to emphasize the symmetry.

\begin{example}\label{ex:f_div}
    We illustrate a number of examples. 
    \begin{enumerate}[noitemsep,topsep=0pt, label=(\alph*)]
        \item KL divergence: It is an $f$-divergence generated by 
            $\fdiv_\kl(t) = t\log t - t + 1$.
        \item Interpolated KL divergence:
        For $\lambda \in (0, 1)$, the interpolated KL divergence is defined as
        \[
            \klam{\lambda}(P \Vert Q) = \kl(P \Vert \lambda P + (1-\lambda) Q) \,,
        \]
        which is a $\fdiv$-divergence generated by 
        \[
            \fdiv_{\kl, \lambda}(t) = t \log\left( \frac{t}{\lambda t + 1-\lambda}  \right) - (1 - \lambda) (t - 1) \,.
        \]
        \item Jensen-Shannon divergence: The Jensen-Shannon Divergence is defined as
        \[
            D_\js(P, Q) = \frac12 \klam{1/2}(P \Vert Q) + \frac12 \klam{1/2}(Q \Vert P).
        \]
        More generally, we have the $\lambda$-skew Jensen-Shannon Divergence~\cite{nielsen2013matrix}, which is defined for $\lambda \in (0, 1)$ as $D_{\js, \lambda} = \lambda \klam{\lambda}(P \Vert Q) + (1-\lambda) \klam{1-\lambda}(Q \Vert P)$. This is an $\fdiv$-divergence generated by
        \[
            \fdiv_{\js, \lambda}(t) = \lambda t \log\left(\frac{t}{\lambda t + 1-\lambda} \right) + (1-\lambda) \log\left( \frac{1}{\lambda t + 1-\lambda} \right) \,.
        \]
        Note that this is the linearized cost defined in \eqref{eq:linear_cost}
        \item \mauveray: From \Cref{prop:line:closed-form}, 
        $\mray$ is an $\fdiv$-divergence generated by
        \[
            \fdiv_\mray(t) = \frac{t+1}{2} - \frac{t}{t-1} \log t \,.
        \]
        \item Interpolated $\chi^2$ divergence: Similar to the interpolated KL divergence, we can define the interpolated $\chi^2$ divergence $D_{\chi^2, \lambda}$
        and the corresponding convex generator $\fdiv_{\chi^2, \lambda}$ 
        for $\lambda \in (0, 1)$ as
        \[
            D_{\chi^2, \lambda}(P \Vert Q) = D_{\chi^2}(P \Vert \lambda P + (1-\lambda) Q)\,, \quad 
            \text{and,} \quad
            \fdiv_{\chi^2, \lambda}(t) = \frac{(t-1)^2}{\lambda t + 1-\lambda} \,.
        \]
        The usual Neyman and Pearson $\chi^2$ divergences are respectively obtained in the limits $\lambda \to 1$ and $\lambda \to 0$.
        \item Squared Le Cam distance: The squared Le Cam distance is, up to scaling, a special case of the interpolated $\chi^2$ divergence
        with $\lambda = 1/2$:
        \[
            D_{\lc}(P, Q) = \frac{1}{4} D_{\chi^2, 1/2}(P \Vert Q) \,.
        \]
        \item Squared Hellinger Distance: It is an $f$-divergence generated by $\fdiv_H(t) = (1 - \sqrt{t})^2$.
    \end{enumerate}
\end{example}
 
\section{Properties of the frontier integral} 
\label{sec:a:mray-properties}
We prove some properties of the frontier integral here. 

First, the frontier integral can be computed in closed form as below.
\begin{proposition}\label{prop:line:closed-form}
    Let $P$ and $Q$ be dominated by some probability measure $\mu$ with density $p$ and $q$, respectively.
    Then,
    \begin{align}
        \mray(P, Q) = \int_{\Xcal}
        \indone{\{p(x) \neq q(x)\}} \left(\frac{p(x) + q(x)}{2} - \frac{p(x)q(x)}{p(x) - q(x)} \log\frac{p(x)}{q(x)} 
        \right) \D\mu(x) \,,
    \end{align}
    with the convention $0 \log 0 = 0$.
    Moreover, $\mray$ is an $f$-divergence generated 
    by the convex function 
    \[
        f_{\mray}(t) = \frac{t+1}{2} - \frac{t}{t-1} \log t \,,
    \]
    with the understanding that 
    $f_{\mray}(1) = \lim_{t \to 1} f_{\mray}(t) = 0$.
\end{proposition}
\begin{proof}[Proof of \Cref{prop:line:closed-form}]
    Let $\bar \lambda = 1-\lambda$.
    By Tonelli's theorem, it holds that $\mray(P, Q) = 2\int_{\Xcal} h(p(x), q(x)) \D \mu(x)$, where
    \[
        h(p, q) = 
        \int_0^1 \left(
        \lambda p \log p + \bar\lambda q\log q
        - (\lambda p  + \bar\lambda q) \log (\lambda p  + \bar\lambda q)\right) \D\lambda.
    \]
    When $p=q$, the integrand is $0$. 
    If $q = 0$, then the second term inside the integral is $0$, while the first term is 
    \[
    \int_0^1 \lambda p \log\frac{1}{\lambda} \D\lambda = \frac{p}{4} \,.
    \]
    Finally, when $p \neq q$ are both non-zero, we evaluate the integral to get, 
    \begin{align*}
        h(p, q) = 
        \frac{p}{2}\log p + \frac{q}{2}\log q
        - \frac{2p^2 \log p - p^2 - 2 q^2 \log q + q^2}{4(p - q)}\,,
    \end{align*}
    and rearranging the expression completes the proof.
\end{proof}

Next, the frontier integral is symmetric and bounded. 
\begin{proposition}\label{prop:sym_div}
    The frontier integral satisfies the following properties:
    \begin{enumerate}[label=(\alph*),topsep=0.0ex,partopsep=0.0ex,itemsep=0.0ex,leftmargin=1.5em]
        \item $\mray(P, Q) = \mray(Q, P)$.
        \item $0 \le \mray(P, Q) \le 1$ with 
            $\mray(P, Q) = 0$ if and only if $P=Q$.
    \end{enumerate}
\end{proposition}
\begin{proof}[Proof of \Cref{prop:sym_div}]
    The first part follows from the closed form expression in \Cref{prop:line:closed-form}. 
    For the second part, we get the upper bound as
    \[
        \mray(P, Q) \le \int_{\Xcal} \frac{p(x) + q(x)}{2} \D\mu(x) = 1 \,.
    \]
    We have $\mray(P, Q) \ge 0$ with $\mray(P, P) = 0$ since
    $\mray$ is an $f$-divergence.
    Further, since $f_{\mray}$ is strictly convex at $1$, we get that $\mray(P, Q) = 0$ only if $P=Q$.
\end{proof}

\section{Regularity assumptions}
\label{sec:a:reg-assump}
In this section, we state and discuss the regularity assumptions required for the statistical error bounds. 
Throughout, we assume that $\Xcal$ is a finite set (for instance, on the quantized space).
We upper bound the expected error of the empirical $f$-divergences estimated from data.

We use the convention that all higher order derivatives of $f$ and $\ftil$ at $0$ are defined as the corresponding limits as $x \to 0^+$ (if they exist).
Further, we use the notation 
\begin{align}
    \psi(p, q) &= q \fdiv(p/q)  = p \ftil(q/p),
\end{align}
so that $\Df{P}{Q} = \sum_{a \in \Xcal} \psi(P(a), Q(a))$.

\subsection{Assumptions}
\label{sub:append_asmp}

We make the following assumptions about 
the functions $\fdiv$ and $\ftil$.
\begin{assumption}\label{asmp:fdiv:appendix}
The generator $\fdiv$ is twice continuously differentiable with $f'(1) = 0$. Moreover,
\begin{enumerate}[noitemsep,topsep=0pt,label={\textbf{(A\arabic*})}]
    \item \label{asmp:fdiv:bounded} 
        We have $\ConstZ := \fdiv(0) < \infty$
        and $\ConstZTil := \ftil(0) < \infty$.
    \item \label{asmp:fdiv:1st-deriv} 
        There exist constants $\ConstI, \ConstITil < \infty$  
        such that for every $x \in (0, 1)$, we have, 
        \begin{align*}
            |\fdiv'(t)| &\le \ConstI \left(1 \vee \log {1}/{t}  \right), \quad \text{and}, \quad
            |\ftilg(t)| \le \ConstITil \left(1 \vee \log {1}/{t}  \right) \,.
        \end{align*}
    \item \label{asmp:fdiv:2nd-deriv} 
        There exist constants $\ConstII, \ConstIITil < \infty$ such that 
        for every $t \in (0, \infty)$, we have, 
        \begin{align*}
            \frac{t}{2} \fdiv''(t) &\le \ConstII \,, \quad \text{and}, \quad
            \frac{t}{2} \ftilh(t) \le \ConstIITil \,.
        \end{align*}
\end{enumerate}
\end{assumption}

\begin{remark}
We discuss the asymptotics of the assumptions.
\begin{enumerate}[noitemsep,topsep=0pt, label=(\alph*)]
    \item Assumption~\ref{asmp:fdiv:bounded} ensures boundedness of the $\fdiv$-divergence. Indeed, 
    $\fdiv(0) = \infty$ leads to $\Df{P}{Q} = \infty$ if 
    there exists an atom $a \in \Xcal$ such that $P(a) = 0$ but $Q(a) \neq 0$.
    This happens, for instance, with the reverse KL divergence 
    ($\fdiv(t) = -\log t + t - 1$). 
    By symmetry, $\ftil(0) = \infty$ leads to a case where 
    $\Df{P}{Q} = \infty$ if 
    there exists an atom $a \in \Xcal$ such that $Q(a) = 0$ but $P(a) \neq 0$,
    as in the (forward) KL divergence. 
    
    \item Since $\fdiv'$ is monotonic nondecreasing and $\fdiv'(1) = 0$, we have that $\fdiv'(0) \le 0$
    (with strict inequality if $\fdiv$ is strictly convex at $1$).
    In fact, $\fdiv'(0) = -\infty$ for each of the divergences considered in \Cref{ex:f_div}.
    Assumption~\ref{asmp:fdiv:1st-deriv} requires $\fdiv'(t)$
    to behave as $\log{1/t}$ when $t \to 0$. Likewise for $\ftilg$.
    
    \item Likewise, we have that $\fdiv''(0) = \infty$ and $\fdiv''(\infty) = 0$ for each of the divergence considered in \Cref{ex:f_div}. However, Assumption~\ref{asmp:fdiv:2nd-deriv} 
    makes assumptions on the rates of these limits.
    Namely, $\fdiv''$ should diverge no faster than $1/t$ as $t\to 0$ 
    and $\fdiv''$ should converge to $0$ at least as fast as $1/t^2$ as $t \to \infty$. We can summarize the implied 
    asymptotics of $\fdiv''$ as 
    \[
        \fdiv''(t) = 
        \begin{cases}
            \Omega(1/t) \,, & \text{ if } t \to 0 \,, \\
            O(1/t^2)\,, & \text{ if } t \to \infty \,.
        \end{cases}
    \]
\end{enumerate}
\end{remark}

\subsection{Examples satisfying the assumptions}
We now consider the examples in \Cref{ex:f_div}. The constants are summarized in \Cref{tab:fdiv:asmp-examples}.

\begin{table}[t]
    \caption{Examples of $\fdiv$-divergences and whether they satisfy Assumptions~\ref{asmp:fdiv:bounded}-\ref{asmp:fdiv:2nd-deriv}. Here, $\lambda \in (0, 1)$ is a parameter 
    of the interpolated or skew divergences, and we define 
    $\bar \lambda := 1- \lambda$.
    \label{tab:fdiv:asmp-examples}
    }
    \centering
    \begin{tabular}{l c cccccc}
    $\fdiv$-divergence & \begin{tabular}{c} Satisfies \\ Assumptions? \end{tabular} & 
    $\ConstZ$ & $\ConstZTil$ & $\ConstI$ & $\ConstITil$ & 
    $\ConstII$ & $\ConstIITil$ 
    \\
    \toprule
    KL & No & $1$ & $\infty$ & & & &  
    \\[0.2cm]
    Interpolated KL & Yes & $\bar \lambda$ & $ \log\tfrac{1}{\lambda} - \bar\lambda$ & $1$ & $\tfrac{\bar \lambda^2}{\lambda}$ & $\tfrac{1}{2}$ & $\tfrac{\bar \lambda}{8\lambda}$ 
    \\[0.2cm]
    JS & Yes & $\tfrac{1}{2}\log 2$ & $\tfrac{1}{2}\log 2$ & $\tfrac{1}{2}$ & $\tfrac{1}{2}$ & $\tfrac{1}{4}$ & $\tfrac{1}{4}$ 
    \\[0.2cm]
    Skew JS & Yes & $\bar\lambda \log\tfrac{1}{\bar \lambda}$ & $\lambda \log\tfrac{1}{\lambda}$ & $\lambda$ & $\bar \lambda$ & $\tfrac{\lambda}{2}$ & $\tfrac{\bar\lambda}{2}$
    \\[0.2cm]
    Frontier integral & Yes & $\tfrac{1}{2}$ & $\tfrac{1}{2}$ & $4$ & $4$ & $\tfrac{1}{2}$ & $\tfrac{1}{2}$ 
    \\[0.2cm]
    LeCam & Yes & $\tfrac{1}{2}$ & $\tfrac{1}{2}$ & $2$ & $2$ & $\tfrac{8}{27}$ & $\tfrac{8}{27}$  
    \\[0.2cm]
    Interpolated $\chi^2$ & Yes & $\tfrac{1}{\bar\lambda}$ & $\tfrac{1}{\lambda}$ & $\tfrac{2}{\bar\lambda^2}$ & $\tfrac{2}{\lambda^2}$ & $\frac{4}{27 \lambda \bar\lambda^2}$ & $\tfrac{4}{27 \lambda^2\bar\lambda}$ 
    \\[0.2cm]
    Hellinger & No & $1$ & $1$ & $\infty$ & $\infty$ & & \\
    \bottomrule
    \end{tabular}
\end{table}

\myparagraph{KL divergence}
We have 
\[
\fdiv_\kl(t) = t \log t - t + 1 \,
\quad \text{and} \quad
\ftil_{\kl}(t) = -\log t + t - 1 \,.
\]
We have $\fdiv(0) = 1$ but $\ftil(0) = \infty$. Therefore, the KL divergence does not satisfy our assumptions. Indeed, 
this is because the KL divergence can be unbounded. 

\myparagraph{Interpolated KL Divergence}
Let $\lambda \in (0, 1)$ be a parameter and denote 
$\bar \lambda = 1-\lambda$. We have
\begin{align*}
    \fdiv_{\kl, \lambda}(t) = t \log\left( \frac{t}{\lambda t + \bar \lambda}  \right) - \bar \lambda(t-1)
    \quad \text{and} \quad
    \ftil_{\kl, \lambda}(t) = -\log(\bar \lambda t + \lambda) + \bar\lambda(t-1) \,.
\end{align*}
The corresponding derivatives are 
\begin{align*}
    \fdiv_{\kl, \lambda}'(t) = \frac{\bar \lambda}{\lambda t + \bar \lambda} + \log\left(\frac{t}{\lambda t + \bar \lambda}  \right) - \bar \lambda, \quad  & \quad
    (\ftil_{\kl, \lambda})'(t) = \bar \lambda - \frac{\bar \lambda}{\bar \lambda t + \lambda}, \\
    \fdiv_{\kl, \lambda}''(t) = \frac{\bar \lambda^2}{t(\lambda t + \bar \lambda)^2} ,
    \quad & \quad
    (\ftil_{\kl, \lambda})''(t) =\frac{\bar \lambda^2}{(\bar \lambda t + \lambda)^2} \,.
\end{align*}

\begin{proposition}\label{prop:const_interpolate}
    The interpolated KL divergence generated by 
    $\fdiv_{\kl, \lambda}$ satisfies \Cref{asmp:fdiv:appendix} with 
    \[
    \ConstZ = 1-\lambda, \quad 
    \ConstZTil = \log\frac{1}{\lambda} - 1+\lambda, \quad
    \ConstI = 1, \quad 
    \ConstITil = \frac{(1-\lambda)^2}{\lambda},\quad
    \ConstII = \frac{1}{2}, \quad
    \ConstIITil = \frac{1-\lambda}{8 \lambda} \,.
    \]
\end{proposition}
\begin{proof}
    First, $\ConstZ, \ConstZTil$ can be computed directly. 
    Second, it is clear that
    \[
    - \fdiv_{\kl, \lambda}'(t) = \log \frac{1}{t} 
    + \log(\lambda t + \bar \lambda) 
    - \frac{\bar\lambda}{\lambda t + \bar \lambda} + \bar\lambda 
    \le \log\frac{1}{t} + \log{1} - \bar \lambda + \bar \lambda = \log{\frac{1}{t}}\,
    \]
    for all $x \in (0, 1)$.
    Moreover, since $f$ is convex and $f'_{\kl, \lambda}(1) = 0$, it holds that $f'_{\kl, \lambda}(x) \le 0$ for all $x \in (0, 1)$, and thus $\ConstI=1$. 
    Next, we note that
    $|(\ftil_{\kl, \lambda})'(x)| \le {\bar \lambda^2}/{\lambda} $ holds uniformly on $(0, 1)$ (or equivalently that
    $\ftil_{\kl, \lambda}$ is Lipschitz); this gives $\ConstITil$. 
    Next, we have
    \[
        \ConstII = \sup_{t > 0} \left\{\frac{1}{2} t \fdiv_{\kl, \lambda}''(t) \right\}
        \le \frac{1}{2}\,,
    \]
    since the function inside the sup is monotonic decreasing on $(0, \infty)$.
    Finally, we have
    \[
        \ConstIITil = \sup_{t > 0} \left\{ \frac{1}{2} t (\ftil_{\kl, \lambda})''(t)  \right\} = \frac{\bar \lambda}{8\lambda} \,,
    \]
    since the term inside the sup is maximized at $t=\lambda/\bar\lambda$.
\end{proof}

\myparagraph{Skew Jensen-Shannon Divergence}
Let $\lambda \in (0, 1)$ be a parameter and $\bar \lambda = 1-\lambda$. We have, 
\[
   \fdiv_{\js, \lambda}(t) = \lambda t \log\left(\frac{t}{\lambda t + \bar\lambda} \right) + \bar\lambda \log\left( \frac{1}{\lambda t + \bar\lambda} \right) = \ftil_{\js, 1-\lambda}(t) \,.
\]
Its derivatives are 
\[
   \fdiv_{\js, \lambda}'(t) = \lambda \log\left( \frac{t}{\lambda t + \bar \lambda} \right)\,
   \quad \mbox{and} \quad
   \fdiv_{\js, \lambda}''(t) = \frac{\lambda \bar \lambda}{t (\lambda t + \bar \lambda)} \,.
\]
\begin{proposition}\label{prop:const_skew-js}
    The $\lambda$-skew JS divergence generated by 
    $\fdiv_{\js, \lambda}$ above satisfies \Cref{asmp:fdiv:appendix} with 
    \[
    \ConstZ = (1-\lambda)\log\frac{1}{1-\lambda}, \quad 
    \ConstZTil = \lambda\log\frac{1}{\lambda} , \quad
    \ConstI = \lambda, \quad 
    \ConstITil = 1-\lambda,\quad
    \ConstII = \frac{\lambda}{2}, \quad
    \ConstIITil = \frac{1-\lambda}{2} \,.
    \]
\end{proposition}
\begin{proof}
    For $\ConstI$, we have
    \[
      -\fdiv_{\js, \lambda}'(t) = \lambda \log\frac{1}{t} + \lambda \log(\lambda t + \bar \lambda) 
      \le \lambda \log\frac{1}{t}
    \]
    for $x \in (0, 1)$. Next, we have
    \[
        \ConstII = \frac{\lambda\bar\lambda}{2} \, \sup_{t > 0} \frac{1}{\lambda t + \bar \lambda} = \frac{\lambda}{2} \,.
    \]
    
\end{proof}

\myparagraph{Frontier integral}
We have
\[
     \fdiv_\mray(t) = \frac{t+1}{2} - \frac{t}{t-1} \log t = \ftil_\mray(t) \,.
\]
Its derivatives are 
\[
    \fdiv_\mray'(t) = \frac{(1-t)(3-t) + 2 \log t}{2(1-t)^2}\,\quad \mbox{and} \quad
    \fdiv_\mray''(t) = \frac{2t \log t - t^2 + 1}{t(1-t)^3} \,.
\]
\begin{proposition}\label{prop:const_mray}
    The frontier integral satisfies \Cref{asmp:fdiv:appendix} with 
    \[
    \ConstZ=  \frac{1}{2} = \ConstZTil, \quad
    \ConstI = 1 = \ConstITil, \quad 
    \ConstII = \frac{1}{2} = \ConstIITil \,.
    \]
\end{proposition}
\begin{proof}
    We get $\ConstZ$ by calculating the limit as $x \to 0$ using L'H\^opital's rule. 
    For $\ConstII$, we note that the term inside the sup below is decreasing in $x$ to get 
    \[
        \ConstII = \sup_{t > 0} \frac{2t \log t - t^2 + 1}{(1-t)^3} = \frac{1}{2} \,.
    \]
    By definition, 
    \[
        \fdiv_\mray(t) = 2\int_0^1 \fdiv_{\js, \lambda}(t) \D\lambda \,,
    \] 
    so that, by \Cref{prop:const_skew-js},
    \begin{align*}
        -\fdiv'_\mray(t) &= -2\int_0^1 \fdiv_{\js, \lambda}'(t) \D \lambda
        \le 2\int_0^1 \lambda \log\frac{1}{t} \D \lambda = \log\frac{1}{t} \,.
    \end{align*}
\end{proof}

\myparagraph{Interpolated $\chi^2$ divergence}
Let $\lambda \in (0, 1)$ be a parameter and 
denote $\bar \lambda = 1 - \lambda$. We have, 
\[
     \fdiv_{\chi^2, \lambda}(t) = \frac{(t-1)^2}{\lambda t + 1-\lambda} = \ftil_{\chi^2, 1-\lambda}(t) \,.
\]
Its derivatives are 
\[
    \fdiv_{\chi^2, \lambda}'(t) = \frac{(t-1)(\lambda t + \bar \lambda + 1)}{(\lambda t + \bar \lambda)^2}
    \quad \mbox{and} \quad
    \fdiv_{\chi^2, \lambda}''(t) = \frac{2}{(\lambda t + \bar \lambda)^2} \,.
\]
\begin{proposition}
    For $\lambda \in (0, 1)$, the interpolated $\chi^2$-divergence satisfies \Cref{asmp:fdiv:appendix} with 
    \begin{align*}
        &\ConstZ = \frac{1}{1-\lambda},\quad
        \ConstZTil = \frac{1}{\lambda}, \quad
        \ConstI = \frac{2}{(1-\lambda)^2}, \quad
        \ConstITil = \frac{2}{\lambda^2} \\
        &\ConstII = \frac{4}{27 \lambda(1-\lambda)^2}, \quad
        \ConstIITil = \frac{4}{27 \lambda^2(1-\lambda)} \,.
    \end{align*}
\end{proposition}
\begin{proof}
    Note that $0 \ge \fdiv_{\chi^2, \lambda}'(0) = -(1+\bar\lambda)/\bar\lambda^2 \ge -2/\bar\lambda^2$ is bounded. Since $\fdiv_{\chi^2, \lambda}'$ is monotonic increasing with $\fdiv_{\chi^2, \lambda}'(1) = 0$, this gives the bound on $\ConstI$.
    Next, we bound
    \[
        \ConstII = \sup_{t > 0} \frac{t}{(\lambda t + \bar \lambda)^3} = \frac{4}{27 \lambda \bar \lambda^2} \,,
    \]
    since the supremum is attained at $t = \bar \lambda/(2\lambda)$. 
\end{proof}

\myparagraph{Squared Hellinger distance}
We have,
\[
    \fdiv_H(t) = (1 - \sqrt{t})^2 = \ftil_{H}(t), \quad
    \fdiv_H'(t) = 1 - \frac{1}{\sqrt{t}}, \quad
    \fdiv_H''(t) = \frac{1}{2} t^{-3/2} \,.
\]
The squared Hellinger divergence does not satisfy our assumptions since for $t < 1$,  $|\fdiv_H'(x)| \approx 1/\sqrt{t}$ diverges faster than the $\log 1/t$ rate required by Assumption~\ref{asmp:fdiv:1st-deriv}.

\subsection{Properties and useful lemmas}
We state here some useful properties and lemmas that we use throughout the paper.

First, we express the derivatives of $\psi(p, q) = q f(p/q)$ in terms of the derivatives of $f$:
\begin{subequations}
\begin{align}
    \label{eq:psi-partial-p}
    \frac{\partial\psi}{\partial p}(p, q)
    &= \fdiv'\left( \frac{p}{q} \right)
    = \ftil\left( \frac{q}{p}  \right) - \frac{q}{p} \ftilg\left( \frac{q}{p}  \right) \\
    \label{eq:psi-partial-q}
    \frac{\partial\psi}{\partial q} (p, q) 
    &= \fdiv\left( \frac{p}{q}  \right) - \frac{p}{q} \fdiv'\left( \frac{p}{q}  \right)  
    =  \ftilg\left( \frac{q}{p} \right) \\
    \label{eq:psi-partial-pp}
    \frac{\partial^2\psi}{\partial p^2} (p, q) 
    &= \frac{1}{q} \fdiv''\left( \frac{p}{q} \right) 
    = \frac{q^2}{p^3} \ftilh\left( \frac{q}{p} \right) \ge 0 \\
    \label{eq:psi-partial-qq}
    \frac{\partial^2\psi}{\partial q^2} (p, q) 
    &= \frac{p^2}{q^3} \fdiv''\left( \frac{p}{q} \right)
    = \frac{1}{p} \ftilh\left( \frac{q}{p} \right)  \ge 0 \\
    \label{eq:psi-partial-pq}
    \frac{\partial^2\psi}{\partial p \partial q} (p, q) 
    &= -\frac{p}{q^2} \fdiv''\left( \frac{p}{q} \right)
    = -\frac{q}{p^2} \ftilh\left( \frac{q}{p} \right)  \le 0 \,,
\end{align}
\end{subequations}
where the inequalities $\fdiv'', \ftilh \ge 0$ followed from convexity of 
$\fdiv$ and $\ftil$ respectively.

The next lemma shows that the function $\psi$ is nearly Lipschitz, up to a log factor. 
This lemma can be leveraged to directly obtain a bound on statistical error of the $\fdiv$-divergence in terms of the expected total variation distance, provided the probabilities are not too small. 

\begin{lemma}\label{lem:fdiv:taylor-ex}
    Suppose that $\fdiv$ satisfies Assumption~\ref{asmp:fdiv:appendix}.
    Consider
    $\psi:[0, 1]\times [0, 1] \to [0, \infty)$ given by
    $\psi(p, q) = q \fdiv(p/q)$.
    We have, for all $p, p', q, q' \in [0, 1]$ 
    with $p\vee p' > 0$, $q\vee q' > 0$, that
    \begin{align*}
        |\psi(p', q) - \psi(p, q)| &\le 
        \left(\ConstI \max\left\{1, \log\frac{1}{p \vee p'}\right\} + \ConstZTil\vee\ConstII  \right) | p-p'| \\
        |\psi(p, q') - \psi(p, q)| &\le 
        \left(\ConstITil \max\left\{1, \log\frac{1}{q \vee q'}\right\} + \ConstZ \vee \ConstIITil \right) | q-q'| \,.
    \end{align*}
\end{lemma}
\begin{proof}
    We only prove the first inequality. The second one is identical with the use of $\ftil$ rather than $\fdiv$. 
    Suppose $p' \ge p$. 
    From the fact that $\psi$ is convex in $p$
    together with a Taylor expansion of $\psi(\cdot, q)$ around $p'$, we get, 
    \begin{align*}
    0 \le \psi(p, q) - \psi(p', q) 
        &- (p- p')\frac{\partial \psi}{\partial p}(p', q)
        = \frac{1}{2} \int_{p'}^p \frac{\partial^2\psi}{\partial p^2}(s, q)(p-s) \D s \\
        &= -\frac{p}{2} \int_{p}^{p'} \frac{\partial^2\psi}{\partial p^2}(s, q) \D s
        + \frac{1}{2} \int_p^{p'} s \frac{\partial^2\psi}{\partial p^2}(s, q) \D s \\
        &\le 0 + \ConstII(p' - p)\,,
    \end{align*}
    where we used $\partial^2\psi / \partial p^2$ is non-negative due to convexity and, 
    by \eqref{eq:psi-partial-pp} and Assumption~\ref{asmp:fdiv:2nd-deriv},
    \[
        s \frac{\partial^2\psi}{\partial p^2}(s, q)
        = \frac{s}{q} \fdiv''\left({s}/{q}\right) \le 2\ConstII \,.
    \]
    This yields
    \[
        - (p' - p) \frac{\partial \psi}{\partial p}(p', q) \le
        \psi(p, q) - \psi(p', q) \le 
        - (p' - p) \frac{\partial \psi}{\partial p}(p', q)
        + \ConstII(p' - p) \,.
    \]
    We consider two cases based on the sign of 
    $\tfrac{\partial \psi}{\partial p}(p', q) = f'(p/q)$ (cf. Eq. \eqref{eq:psi-partial-pq}).
    
    \myparagraph{Case 1} 
    $\tfrac{\partial \psi}{\partial p}(p', q) \ge 0$.
    Since $q \mapsto f'(p/q)$ is 
    decreasing in $q$, 
    we have
    \begin{align*}
    0 \le (p' - p) \frac{\partial \psi}{\partial p}(p', q)
    = (p' - p) \fdiv'(p/q) 
    \le \lim_{q \to 0} (p' - p) \fdiv'(p/q) 
    = (p' - p) \ftil(0) \,,
    \end{align*}
    where we used $\fdiv'(\infty) = \ftil(0)$ 
    from \Cref{lem:fdiv:ftil-infty}.
    From Assumption~\ref{asmp:fdiv:bounded}, we get the bound
    \[
        |\psi(p, q) - \psi(p', q)| \le  (\ConstZTil \vee \ConstII) (p' - p) \,.
    \]
    
    \myparagraph{Case 2}
    $\tfrac{\partial \psi}{\partial p}(p', q) < 0$.
    By Assumption~\ref{asmp:fdiv:1st-deriv},
    it holds that
    \[
        \left|\frac{\partial \psi}{\partial p}(p', q)\right|
        \le \ConstI\, \max\{1, \log (q/p')\}
        \le \ConstI\, \max\{1, \log (1/p')\}\,,
    \]
    and thus
    \[
        |\psi(p, q) - \psi(p', q)| \le  \left(\ConstI \max\left\{1, \log \frac{1}{p'}\right\} + \ConstII \right) (p' - p) \,.
    \]
\end{proof}

With the above lemma, the estimation error of the empirical $f$-divergence can be upper bounded by the total variation distance between the empirical measure and its population counterpart up to a logarithmic factor, where:
\begin{align}
    \norm{\Phatn - P}_\tv = \sum_{a \in \Xcal} |\Phatn(a) - P(a)| \,.
\end{align}

Next, we state and prove a technical lemma.
\begin{lemma} \label{lem:fdiv:ftil-infty}
    Suppose the generator $\fdiv$ satisfies Assumptions~\ref{asmp:fdiv:bounded} and~\ref{asmp:fdiv:1st-deriv}. Then, 
    \[
        \lim_{t \to \infty} \fdiv'(t) = \ftil(0)\,,
        \quad 
        \text{and}
        \lim_{t \to \infty} \ftilg(t) = \fdiv(0)\,.
    \]
\end{lemma}
\begin{proof}
    We start by observing that 
    \[
        \lim_{t \to 0} t |\fdiv'(t)| 
        \le \ConstI \lim_{t \to 0} t \vee t\log\frac{1}{t}
        = 0 \,.
    \]
    Next, a direct calculation gives 
    \[
        \ftilg(1/t) = \fdiv(t) - t\fdiv'(t) \,,
    \]
    so that taking the limit $t \to 0$ gives
    \[
        \lim_{t \to \infty} \ftilg(t) = \fdiv(0) 
        - \lim_{t \to 0} t \fdiv'(t) = \fdiv(0) \,.
    \]
    The proof of the other part is identical. 
\end{proof}

\section{Plug-in estimator: statistical error}
\label{sec:a:plug-in}
In this section, we prove the 
high probability concentration bound for the plug-in estimator.
There are two keys steps: bounding the
statistical error and giving a deviation bound. 

Throughout this section, we assume that $P$ and $Q$ are discrete.
Let $\{X_i\}_{i=1}^n$ and $\{Y_j\}_{j=1}^m$ be two independent i.i.d.~samples from $P$ and $Q$, respectively.
We consider the plug-in estimator of the $\fdiv$-divergences, i.e., $\Df{\hat P_n}{\hat Q_m}$.
The main results are (a) an upper bound for its statistical error, and (b) a high probability concentration bound.
They all hold for the linearized cost $\lerror{\lambda}(\hat P_n, \hat Q_n)$ and the frontier integral $\mray(\hat P_n, \hat Q_n)$ due to \Cref{prop:const_skew-js} and \Cref{prop:const_mray}.

\subsection{Statistical error}

\begin{proposition}\label{prop:fdiv:consistency}
    Suppose that $\fdiv$ satisfies \Cref{asmp:fdiv:appendix} 
    and $k := |\Supp{P}| \vee |\Supp{Q}| \in \mathbb{N} \cup \{\infty\}$. 
    Let $n, m \ge 3$.
    Let $c_1 =  \ConstI + \ConstITil$
    and $c_2 = \ConstII \vee \ConstZTil +\ConstIITil \vee \ConstZ$.
    We have,
    \begin{align}\label{eq:fdiv:stat_error_oracle}
        \expect| \Df{P}{Q} - \Df{\Phatn}{\hat Q_m}|
        &\le \big(\ConstI \log{n} + \ConstZTil \vee \ConstII\big) \alpha_{n}(P) + \big(\ConstITil \log{m} + \ConstZ \vee \ConstIITil\big) \alpha_{m}(Q) \\
        &\quad + \big(\ConstI + \ConstZTil \vee \ConstII\big) \beta_{n}(P) + \big(\ConstITil + \ConstZ \vee \ConstIITil\big) \beta_{m}(Q)\,, \nonumber
    \end{align}
    where $\alpha_n(P) = \sum_{a \in \Xcal} \sqrt{n^{-1} P(a)}$ and $\beta_n(P) = \expect\big[ \sum_{a: \hat P_n(a) = 0} P(a) \max\left\{ 1, \log{(1/P(a))} \right\} \big]$.
    Furthermore, if $k < \infty$, then
    \begin{align}\label{eq:fdiv:stat_error}
        \expect| \Df{P}{Q} - \Df{\Phatn}{\hat Q_m}| \le \big(c_1 \log{(n \wedge m)} + c_2\big) \left(\sqrt\frac{k}{n \wedge m} + \frac{k}{n \wedge m}  \right) \,.
    \end{align}
\end{proposition}

The proof relies on two key lemmas---the approximate Lipschitz lemma (\Cref{lem:fdiv:taylor-ex})
and the missing mass lemma (\Cref{lem:fdiv:expected-missing-mass}). 
The argument breaks into two cases in $P$ (and analogously for $Q$) for each atom $a \in \Xcal$:
\begin{enumerate}[noitemsep,topsep=0pt, label=(\alph*)]
    \item $\Phatn(a) > 0$: Since $\Phatn$ is an empirical measure, we have that $\Phatn(a) \ge 1/n$. In this case the approximate Lipschitz lemma gives us the Lipschitzness in $\norm{P - \Phatn}_\tv$ up to a factor of $\log n$.
    \item $\Phatn(a) = 0$: In this case, the mass corresponding to $P(a)$ is missing in the empirical measure and we directly bound its expectation following similar arguments as in the missing mass literature; see, e.g., \cite{berend2012missing,mcallester2005concentration}.
\end{enumerate}

For the first part, we further upper bound the expected total variation distance of the plug-in estimator, which is 
\[
    \norm{\Phatn - P}_\tv = \sum_{a \in \Xcal} |\Phatn(a) - P(a)| \,.
\]
\begin{lemma}\label{lem:fdiv:l1-bound}
    Assume that $P$ is discrete.
    For any $n \ge 1$, it holds that
    \begin{align*}
        \expect \norm{\Phatn - P}_\tv
        \le \alpha_n(P).
    \end{align*}
    Furthermore, if $k = |\Supp{P}| < \infty$, then
    \[
        \expect \norm{\Phatn - P}_\tv
        \le \alpha_n(P)
        \le \sqrt{\frac{k}{n}} \,.
    \]
\end{lemma}
\begin{proof}
    Using Jensen's inequality, we have,
    \begin{align*}
        \expect \sum_{a \in \Supp{P}} |\Phatn(a) - P(a)|
        &\le \sum_{a \in \Supp{P}} \sqrt{\expect(\Phatn(a) - P(a))^2} \\
        &= \sum_{a \in \Supp{P}} \sqrt{\frac{P(a)(1 - P(a))}{n}}
        \le \alpha_n(P)\,,
    \end{align*}
    If $k < \infty$, then it follows from Jensen's inequality applied to the concave function $t \mapsto \sqrt{t}$ that
    \[
        \frac{1}{k} \sum_{i=1}^k \sqrt{a_k} \le \sqrt{\frac{1}{k}{\sum_{i=1}^k a_k}} \,.
    \]
    Hence, $\alpha_n(P) \le k/n$ and it completes the proof.
\end{proof}

For the second part, we treat the missing mass directly.
\begin{lemma}[Missing Mass] \label{lem:fdiv:expected-missing-mass}
    Assume that $k = |\Supp{P}| < \infty$.
    Then, for any $n \ge 3$,
    \begin{align}
        \expect\left[ \sum_{a \in \Xcal} \indone\big\{\Phatn(a)=0\big\} P(a) \right] &\le \frac{k}{n} \label{eq:exp_miss_mass} \\
        \beta_n(P) := \expect\left[ \sum_{a\in\Xcal} \indone\big\{\Phatn(a)=0\big\} P(a) \left(  1 \vee \log\frac{1}{P(a)}\right) \right] &\le \frac{k \log n}{n} \label{eq:exp_var_miss_mass} \,,
    \end{align}
    where $a \vee b := \max\{a, b\}$.
\end{lemma}
\begin{proof}
    We prove the second inequality. The first one is identical.
    Note that $\expect[\indone\{\Phatn(a)=0\}] = 
    \prob(\Phatn(a) = 0) = (1 - P(a))^n$.
    Therefore, the left hand side (LHS) of the second inequality is
    \begin{align*}
        \text{LHS} &= \sum_{a \in \Xcal}  (1-P(a))^n P(a) \max\{1, -\log P(a)\} \\
        &\le \sum_{a \in \Xcal} \frac{1}{n} \vee \frac{\log n}{n} = \frac{k \log n}{n} \,,
    \end{align*}
    where we used \Cref{lem:techn:missing-mass-2} and \Cref{lem:techn:missing-mass-1}.
\end{proof}
\begin{remark}
    According to \cite[Prop. 3]{berend2012missing}, the bound $k/n$ in \eqref{eq:exp_miss_mass} is tight up to a constant factor.
\end{remark}

Now, we are ready to prove \Cref{prop:fdiv:consistency}.
\begin{proof}[Proof of \Cref{prop:fdiv:consistency}]
    Define $\Delta_{n,m}(a) := \left|\psi\big(P(a), Q(a)\big) - \psi\big(\Phatn(a), \hat Q_m(a)\big)\right|$. We have from the triangle inequality that 
    \[
        \Delta_{n,m}(a) \le 
       \underbrace{\left|\psi\big(P(a), Q(a)\big) - \psi\big(\Phatn(a), Q(a)\big)\right|}_{=:\Tcal_1(a)}
        + 
        \underbrace{\left|\psi\big(\Phatn(a), Q(a)\big) - \psi\big(\Phatn(a), \hat Q_m(a)\big)\right|}_{=:\Tcal_2(a)} \,.
    \]
    Since $\Phatn(a) = 0$ or $\Phatn(a) \ge  1/n$, 
    the approximate Lipschitz lemma (\Cref{lem:fdiv:taylor-ex}) gives
    \[
        \Tcal_1(a) \le 
        \begin{cases}
            P(a) \left(\ConstI \max\{1, \log (1/P(a))\} + \ConstZTil \vee \ConstII\right) \,,
            & \text{ if } \Phatn(a) = 0, \\
            |P(a) - \Phatn(a)|\,\big(\ConstI \log n + \ConstZTil \vee \ConstII\big) \,, &\text{ else}.
        \end{cases}
    \]
    Consequently, \Cref{lem:fdiv:l1-bound} yields
    \begin{align*}
        \sum_{a \in \Xcal} \expect[\Tcal_1]
        &\le \sum_{a \in \Xcal} \expect\left[ \indone\{\hat P_n(a) = 0\} P(a) \left(\ConstI \max\{1, \log (1/P(a))\} + \ConstZTil \vee \ConstII\right) \right] \\
        &\quad + \sum_{a \in \Xcal} \expect\left[ \abs{\hat P_n(a) - P(a)} \right]\big(\ConstI \log n + \ConstZTil \vee \ConstII\big) \\
        &\le \left(\ConstI + \ConstZTil \vee \ConstII\right) \beta_n(P) + \big(\ConstI \log n + \ConstZTil \vee \ConstII\big) \alpha_n(P)\,.
    \end{align*}
    Since $\psi(p, q) = q\fdiv(p/q) = p\ftil(q/p)$,
    an analogous bound holds for $\Tcal_2$ with the appropriate adjustment of constants.
    Hence, the inequality \eqref{eq:fdiv:stat_error_oracle} holds.
    Moreover, when $k < \infty$, the inequality \eqref{eq:fdiv:stat_error} follows by invoking again \Cref{lem:fdiv:expected-missing-mass} and \Cref{lem:fdiv:l1-bound}.
\end{proof}

Invoking \Cref{prop:const_interpolate} and \Cref{prop:fdiv:consistency} for the interpolated KL divergence leads to the following result.
\begin{proposition}\label{prop:bound_expect_kl}
    Assume that $k = \abs{\Supp{P}} \vee \abs{\Supp{Q}} < \infty$.
    For any $\lambda \in (0, 1)$, it holds that
    \begin{align*}
        &\quad \expect\abs{\klam{\lambda}(\hat P_n \Vert \hat Q_m) - \klam{\lambda}(P \Vert Q)} \\
        &\le \left[\left( 1 + \frac{(1-\lambda)^2}{\lambda} \right) \log{(n \wedge m)} + \left( \log \frac1\lambda - 1 + \lambda \right) \vee \frac12 + (1-\lambda) \vee \frac{1-\lambda}{8\lambda} \right] \\
        &\quad \times \left( \sqrt{\frac{k}{n \wedge m}} + \frac{k}{n \wedge m} \right)\,.
    \end{align*}
    Moreover, for any $\lambda_{n,m} \in (0, 1/2)$,
    \begin{align*}
        &\quad \expect
        \left[
        \sup_{\lambda \in [\lambda_{n,m}, 1 - \lambda_{n,m}]} \left\{ \abs{\klam{\lambda}(\hat P_n \Vert \hat Q_m) - \klam{\lambda}(P \Vert Q)} + \abs{\klam{1-\lambda}(\hat Q_m \Vert \hat P_n) - \klam{1-\lambda}(Q \Vert P)} \right\}
        \right]\\
        &\le 2\left( (1 + 1/\lambda_{n,m}) \log{n} + \log \frac1{\lambda_{n,m}} \vee \frac12 + 1 \vee \frac1{8\lambda_{n,m}} \right)\left( \sqrt{\frac{k}{n \wedge m}} + \frac{k}{n \wedge m} \right)\,.
    \end{align*}
\end{proposition}
\begin{proof}
    We only prove the second inequality.
    The first one is a direct consequence of \Cref{prop:const_interpolate} and \Cref{prop:fdiv:consistency}.
    From the proof of \Cref{prop:fdiv:consistency} we have
    \begin{align*}
        &\quad \abs{\klam{\lambda}(\hat P_n \Vert \hat Q_m) - \klam{\lambda}(P \Vert Q)} \\
        &\le \sum_{a \in \Xcal} \indone\{\Phatn(a) = 0\}\, P(a) \left(\ConstI \max\{1, \log (1/P(a))\} + \ConstZTil \vee \ConstII\right) \\
        &\quad + \sum_{a \in \Xcal} \indone\{\hat Q_m(a) = 0\} \, Q(a) \left(\ConstITil \max\{1, \log (1/Q(a))\} + \ConstZ \vee \ConstIITil \right) \\
        &\quad + \sum_{a \in \Xcal} \big|P(a) - \Phatn(a)\big| \big(\ConstI \log n + \ConstZTil \vee \ConstII\big)
        + \sum_{a \in \Xcal} \big|Q(a) - \hat Q_m(a)\big| \big(\ConstITil \log m + \ConstZ \vee \ConstIITil\big)\,.
    \end{align*}
    Note that, for the intepolated KL divergence, we have
    \begin{align*}
        \ConstZ = 1 - \lambda \le 1&, \quad \ConstZTil = \log{\frac{1}{\lambda}} - 1 + \lambda \le \log{\frac{1}{\lambda_{n,m}}} \\
        \ConstI = 1&, \quad \ConstITil = \frac{(1-\lambda)^2}{\lambda} \le \frac{1}{\lambda_{n,m}} \\
        \ConstII = 1/2&, \quad \ConstIITil = \frac{1-\lambda}{8\lambda} \le \frac{1}{8\lambda_{n,m}}
    \end{align*}
    for all $\lambda \in [\lambda_{n,m}, 1 - \lambda_{n,m}]$.
    The claim then follows from the same steps of \Cref{prop:fdiv:consistency}.
\end{proof}

\subsection{Concentration bound}
We now state and prove the concentration bound for general $\fdiv$-divergences which satisfy our regularity assumptions.
We start by considering concentration around the expectation.
\begin{proposition} \label{prop:fdiv:deviation_bound}
    Consider the $\fdiv$-divergence $D_\fdiv$ where 
    $\fdiv$ satisfies Assumptions~\ref{asmp:fdiv:bounded}-\ref{asmp:fdiv:2nd-deriv}. 
    For any $t > 0$ and any dicrete distributions $P, Q$, 
    we have, 
    \[
        \prob\left( |\Df{\Phatn}{\hat Q_m} - \expect[\Df{\Phatn}{\hat Q_m}] | > \varepsilon \right) \le 
        2 \exp\left( -\frac{ (n \wedge m) \varepsilon^2}{2 (c_1 \log{(n \wedge m)} + c_2)^2} \right) \,,
    \]
    where $c_1 =  \ConstI + \ConstITil$
    and $c_2 = \ConstII \vee \ConstZTil +\ConstIITil \vee \ConstZ$.
\end{proposition}
\begin{proof}
    We first establish that $D_\fdiv$ satisfies the bounded deviation property 
    and then invoke McDiarmid's inequality. 
    
    We start with some notation.
    As before, define $\psi(p, q) = q\fdiv(p/q)$.
    Without loss of generality, let $\Xcal = \Supp{P} \cup \Supp{Q}$.
    Define the function $\Phi: \Xcal^{n+m} \to \reals$ so that 
    \[
        \Phi(X_1, \cdots, X_n, Y_1, \cdots, Y_m) = \Df{\Phatn}{\hat Q_m}\,.
    \]
    
    We now show the bounded deviation property of $\Phi$.
    Fix some $T = (x_1, \cdots, x_n, y_1, \cdots, y_m) \in \Xcal^{n+m}$ and let 
    $T' = (x_1', \cdots, x_n', y_1', \cdots, y_m') \in \Xcal^{n+m}$ be such that 
    $T$ and $T'$ differ only on $x_i = a \neq a' = x_i'$.
    Suppose the number of occurrences of $a$ in the $x$-component of $T$ is $l$
    and of $a'$ is $l'$, while their corresponding $y$-components are $mq$ and $mq'$ respectvely. 
    We now have
    \begin{align*}
        |\Phi(T') - \Phi(T)|
        &= \left| \psi\left(\frac{s-1}{n}, q\right) 
                    - \psi\left(\frac{s}{n}, q\right)
                    + \psi\left(\frac{s'+1}{n}, q'\right) 
                    - \psi\left(\frac{s'}{n}, q'\right) 
            \right| \\
        &\le \left| \psi\left(\frac{s-1}{n}, q\right) 
                    - \psi\left(\frac{s}{n}, q\right)\right|
            + \left|\psi\left(\frac{s'+1}{n}, q'\right) 
                    - \psi\left(\frac{s'}{n}, q'\right) 
            \right| \\
        &\le \frac{2}{n}(\ConstI \log n + \ConstZTil \vee \ConstII)  =: B_i \,,
    \end{align*}
    where we used the triangle inequality first and 
    then invoked \Cref{lem:fdiv:taylor-ex}.
    Likewise, if $A$ and $A'$ differ only in $y_i$ and $y_i'$, 
    an analogous argument gives
    \begin{align*}
        |\Phi(T') - \Phi(T)|
        &\le \frac{2}{m} (\ConstITil \log m + \ConstZ \vee \ConstIITil)  =: B_i^* \,.
    \end{align*}
    With this we can use McDiarmid's inequality (cf. \Cref{thm:technical:mcdiarmid})
    to bound
    \[
        \prob\left( |\Df{\Phatn}{\hat Q_m} - \expect[\Df{\Phatn}{\hat Q_m}] | > \varepsilon \right) \le 
        h(\varepsilon) \,,
    \]
    where
    \begin{align*}
        h(\varepsilon) &= 2\exp\left(  -\frac{2\varepsilon^2}{\sum_{i=1}^n B_i^2 + \sum_{i=n+1}^{n+m} (B_i^*)^2} \right)
        \le 2 \exp\left( -\frac{ (n \wedge m) \varepsilon^2}{2 (c_1 \log{(n \wedge m)} + c_2)^2} \right) \,.
    \end{align*}
\end{proof}

Hence, the concentration bound around the population $\fdiv$-divergence follows directly from \Cref{prop:fdiv:consistency} and \Cref{prop:fdiv:deviation_bound}.
\begin{theorem}\label{thm:fdiv:sample_complexity}
    Assume that $P$ and $Q$ are discrete and let $k = \abs{\Supp{P}} \vee \abs{\Supp{Q}} \in \mathbb{N} \cup \{\infty\}$.
    For any $\delta \in (0, 1)$, it holds that, with probability at least $1 - \delta$,
    \begin{align*}
        &\abs{\Df{\Phatn}{\hat Q_m} - \Df{P}{Q}}
        \le \big(c_1 \log{(n \wedge m)} + c_2\big) \sqrt{\frac{2}{n \wedge m}} \log{\frac{2}{\delta}} \\
        &\quad + \big(\ConstI \log{n} + \ConstZTil \vee \ConstII\big) \alpha_{n}(P) + \big(\ConstITil \log{m} + \ConstZ \vee \ConstIITil\big) \alpha_{m}(Q) \\
        &\quad + \big(\ConstI + \ConstZTil \vee \ConstII\big) \beta_{n}(P) + \big(\ConstITil + \ConstZ \vee \ConstIITil\big) \beta_{m}(Q)\,.
    \end{align*}
    Furthermore, if $k < \infty$, then, with probability at least $1 - \delta$,
    \begin{align*}
        \abs{\Df{\Phatn}{\hat Q_m} - \Df{P}{Q}} \le \big(c_1 \log{(n \wedge m)} + c_2\big) \left( \sqrt{\frac{2}{n \wedge m}} \log{\frac{2}{\delta}} + \sqrt{\frac{k}{n \wedge m}} + \frac{k}{n \wedge m} \right)\,.
    \end{align*}
\end{theorem}
\begin{proof}[Proof of \Cref{thm:fdiv:sample_complexity}]
    We only prove the second inequality.
    The first one follows from a similar argument.
    According to \Cref{prop:fdiv:consistency}, we have
    \begin{align*}
        &\quad \abs{\Df{\Phatn}{\hat Q_m} - \expect[\Df{\Phatn}{\hat Q_m}]} \\
        &\ge \abs{\Df{\Phatn}{\hat Q_m} - \Df{P}{Q}} - \abs{\expect[\Df{\Phatn}{\hat Q_m}] - \Df{P}{Q}} \\
        &\ge \abs{\Df{\Phatn}{\hat Q_m} - \Df{P}{Q}} - \big(c_1 \log{(n \wedge m)} + c_2\big) \left(\sqrt\frac{k}{n \wedge m} + \frac{k}{n \wedge m}  \right).
    \end{align*}
    By \Cref{prop:fdiv:deviation_bound}, it holds that
    \begin{align*}
        \prob\left( \abs{\Df{\Phatn}{\hat Q_m} - \Df{P}{Q}} > \varepsilon + \big(c_1 \log{(n \wedge m)} + c_2\big) \left(\sqrt\frac{k}{n \wedge m} + \frac{k}{n \wedge m}  \right) \right) \le 
        h(\epsilon) \,,
    \end{align*}
    where
    \begin{align*}
        h(\epsilon) = 2 \exp\left( -\frac{ (n \wedge m) \varepsilon^2}{2 (c_1 \log{(n \wedge m)} + c_2)^2} \right).
    \end{align*}
    The claim then follows from setting $h(\epsilon) = \delta$ and solving for $\epsilon$.
\end{proof}

\section{Add-constant smoothing: statistical error}
\label{sec:a:add-constant}
In this section, we apply add-constant smoothing to estimate the $\fdiv$-divergences and study its statistical error.
All the results hold for the linearized cost $\lerror{\lambda}(\hat P_n, \hat Q_n)$ and the frontier integral $\mray(\hat P_n, \hat Q_n)$ due to \Cref{prop:const_skew-js} and \Cref{prop:const_mray}.

For notational simplicity, we assume that $P$ and $Q$ are supported on a common finite alphabet with size $k < \infty$.
Without loss of generality, let $\Xcal$ be the support.
Consider $P \in \Pcal(\Xcal)$ and an i.i.d.~sample $\{X_i\}_{i=1}^n \sim P$.
The add-constant estimator of $P$ is defined by
\begin{align*}
    \hat P_{n, b}(a) = \frac{N_a + b}{n + kb}, \quad \mbox{for all } a \in \Xcal\,,
\end{align*}
where $b > 0$ is a constant and $N_a = \abs{\{i \in [n]: X_i = a\}}$ is the number of times the symbol $a$ appears in the sample.
In practice, $b = b_a$ could be different depending on the value of $N_a$, but we use the same constant $b$ for simplicity.
Similarly, We define $\hat Q_{m, b}$ with $M_a = \abs{\{i \in [m]: Y_i = a\}}$.
The goal is to upper bound the statistical error
\begin{align}
    \expect\abs{\Df{P}{Q} - \Df{\hat P_{n,b}}{\hat Q_{m,b}}}
\end{align}
under \Cref{asmp:fdiv:appendix}.

Compared to the statistical error of the plug-in estimator, a key difference is that each entry in the add-constant estimator is at least $(n + kb)^{-1} \wedge (m + kb)^{-1}$.
Hence, we can directly apply the approximate Lipschitz lemma without the need to control the missing mass part.
Another difference is that the total variation distance is now between the add-constant estimator and its population counterpart, which can be bounded as follows.
\begin{lemma}\label{lem:tv_smooth}
    Assume that $k = \Supp{P} < \infty$.
    Then, for any $b > 0$,
    \begin{align*}
        \sum_{a \in \Xcal} \expect\abs{\hat P_{n,b}(a) - P(a)} \le \sum_{a \in \Xcal} \frac{\sqrt{nP(a)(1 - P(a))} + bk \abs{P(a) - 1/k}}{n+kb} \le \frac{\sqrt{kn} + 2b(k-1)}{n + kb} \,.
    \end{align*}
\end{lemma}
\begin{proof}
    Note that
    \begin{align*}
        \abs{\hat P_{n,b}(a) - P(a)}
        = \abs{\frac{N_a - nP(a)}{n + kb} + \frac{b(1 - kP(a))}{n + kb}}
        \le \abs{\frac{N_a - nP(a)}{n + kb}} + \abs{\frac{b(1 - kP(a))}{n + kb}}.
    \end{align*}
    Using Jensen's inequality, we have
    \begin{align*}
        \sum_{a \in \Xcal} \expect\abs{\hat P_{n,b}(a) - P(a)}
        &\le \sum_{a \in \Xcal} \left[ \sqrt{\expect\abs{\frac{N_a - nP(a)}{n + kb}}^2} + \frac{c\abs{1 - kP(a)}}{n + kb} \right] \\
        &= \sum_{a \in \Xcal} \left[ \frac{\sqrt{n P(a)(1 - P(a))}}{n + kb} + \frac{b k\abs{1/k - P(a)}}{n + kb} \right].
    \end{align*}
    We claim that
    \begin{align*}
        \sum_{a \in \Xcal} \abs{P(a) - \frac{1}{k}} \le \frac{2(k-1)}{k}.
    \end{align*}
    If this is true, we have
    \begin{align*}
        \sum_{a \in \Xcal} \expect\abs{\hat P_{n,b}(a) - P(a)} \le \frac{\sqrt{kn} + 2b(k-1)}{n + kb}\,,
    \end{align*}
    since $\sum_{a \in \Xcal} \sqrt{P(a)(1 - P(a))} \le \sqrt{k}$
    It then remains to prove the claim.
    Take $a_1, a_2 \in \Xcal$ such that $P(a_1) \ge k^{-1} \ge P(a_2)$.
    It is clear that
    \begin{align*}
        \abs{P(a_1) - \frac1k} + \abs{P(a_2) - \frac1k}
        &\le \abs{P(a_1) + P(a_2) - \frac1k} + \abs{P(a_2) - P(a_2) - \frac1k} \\
        &= P(a_1) + P(a_2).
    \end{align*}
    Repeating this argument gives
    \begin{align*}
        \sum_{a \in \Xcal} \abs{P(a) - \frac1k} \le 1 - \frac1k + \frac{k-1}{k} = \frac{2(k-1)}{k}.
    \end{align*}
\end{proof}

The next proposition gives the upper bound for the statistical error of the add-constant estimator.
\begin{proposition} \label{prop:fdiv_smooth_consist}
    Suppose that $\fdiv$ satisfies Assumption~\ref{asmp:fdiv:appendix} 
    and $k = |\Xcal| < \infty$.
    We have, for any $n, m \ge 3$,
    \begin{align*}
        \expect\big|\Df{P}{Q} - \Df{\hat P_{n,b}}{\hat Q_{m,b}}\big| &\le \left[ \frac{n \alpha_n(P)}{n + kb} + \gamma_{n,k}(P) \right] \, \big(\ConstI \log (n/b + k) + \ConstZTil \vee \ConstII\big)  \\
        &\quad + \left[ \frac{m \alpha_m(Q)}{m + kb} + \gamma_{m,k}(Q) \right] \, \big(\ConstITil \log (m/b + k) + \ConstZ \vee \ConstIITil\big) \\
        &\le \big(\ConstI \log (n/b + k) + \ConstZTil \vee \ConstII\big) \frac{\sqrt{kn} + 2b(k-1)}{n + kb} \\
        &\quad + \big(\ConstITil \log (m/b + k) + \ConstZ \vee \ConstIITil\big) \frac{\sqrt{km} + 2b(k-1)}{m + kb}\,,
    \end{align*}
    where $\gamma_{n,k}(P) = (n + bk)^{-1} bk \sum_{a \in \Xcal} \abs{P(a) - 1/k}$.
\end{proposition}
\begin{proof}
    Following the proof of \Cref{prop:fdiv:consistency}, we define
    \[
        \Delta_{n,m}(a) := \abs{\psi(P(a), Q(a)) - \psi(\hat P_{n,b}(a), \hat Q_{m,b}(a))}\,.
    \]
    We have from the triangle inequality that 
    \[
        \Delta_{n,m}(a) \le 
       \underbrace{\left|\psi\big(P(a), Q(a)\big) - \psi\big(\hat P_{n, b}(a), Q(a)\big)\right|}_{=:\Tcal_1(a)}
        + 
        \underbrace{\left|\psi\big(\hat P_{n,b}(a), Q(a)\big) - \psi\big(\hat P_{n,b}(a), \hat Q_{m,b}(a)\big)\right|}_{=:\Tcal_2(a)} \,.
    \]
    Since $\hat P_{n,b}(a) \ge  b/(n + kb)$, 
    the approximate Lipschitz lemma (\Cref{lem:fdiv:taylor-ex}) gives
    \[
        \Tcal_1(a) \le |P(a) - \hat P_{n,b}(a)|\,\big(\ConstI \log (n/b + k) + \ConstZTil \vee \ConstII\big) \,,
    \]
    By \cref{lem:tv_smooth}, it holds that
    \begin{align*}
        \frac{\sum_{a \in \Xcal} \expect[\Tcal_1(a)]}{\ConstI \log (n/b + k) + \ConstZTil \vee \ConstII}
        &\le \sum_{a \in \Xcal} \left[ \frac{\sqrt{n P(a)}}{n + kb} + \frac{bk \abs{1/k - P(a)}}{n + kb} \right]
        = \frac{n \alpha_n(P)}{n + kb} + \gamma_{n,k}(P) \\
        &\le \frac{\sqrt{kn} + 2b(k-1)}{n + kb}\,.
    \end{align*}
    Since $\psi(p, q) = q\fdiv(p/q) = p\ftil(q/p)$,
    an analogous bound holds for $\Tcal_2(a)$ with the appropriate adjustment of constants and the sample size.
    Putting these together, we get,
    \begin{align*}
        &\quad \expect\big|\Df{P}{Q} - \Df{\hat P_{n,b}}{\hat Q_{m,b}}\big|
        \le \expect\left[\sum_{a \in \Xcal} |\Delta_n(a)|\right] \\
        &\le \left[ \frac{n \alpha_n(P)}{n + kb} + \gamma_{n,k}(P) \right] \, \big(\ConstI \log (n/b + k) + \ConstZTil \vee \ConstII\big)  \\
        &\quad + \left[ \frac{m \alpha_m(Q)}{m + kb} + \gamma_{m,k}(Q) \right] \, \big(\ConstITil \log (m/b + k) + \ConstZ \vee \ConstIITil\big) \\
        &\le \big(\ConstI \log (n/b + k) + \ConstZTil \vee \ConstII\big) \frac{\sqrt{kn} + 2b(k-1)}{n + kb} \\
        &\quad + \big(\ConstITil \log (m/b + k) + \ConstZ \vee \ConstIITil\big) \frac{\sqrt{km} + 2b(k-1)}{m + kb}\,.
    \end{align*}
\end{proof}

The concentration bound for the add-constant estimator can be proved similarly.

\section{Quantization error}
\label{sec:a:quantization}
In this section, we study the quantization error of $\fdiv$-divergences, i.e.,
\begin{align}
    \inf_{\abs{\Scal} \le k} \abs{\Df{P}{Q} - \Df{P_{\Scal}}{ Q_{\Scal}}},
\end{align}
where the infimum is over all partitions of $\Xcal$ of size no larger than $k$, and $P_{\Scal}$ and $Q_{\Scal}$ are the quantized versions of $P$ and $Q$ according to $\Scal$, respectively.
Note that we do not assume $\Xcal$ to be discrete in this section.
All the results hold for the linearized cost $\lerror{\lambda}(\hat P_n, \hat Q_n)$ and the frontier integral $\mray(\hat P_n, \hat Q_n)$ due to \Cref{prop:const_skew-js} and \Cref{prop:const_mray}.

Our analysis is inspired by the following result, which shows that the $\fdiv$-divergence can be approximated by its quantized counterpart; see, e.g., \cite[Theorem 6]{gyorfi1978fdiv}.
\begin{theorem}
    For any $P, Q \in \Pcal(\Xcal)$, it holds that
    \begin{align}
        \Df{P}{Q} = \sup_{\Scal} \Df{P_\Scal}{Q_\Scal},
    \end{align}
    where the supremum is over all finite partitions of $\Xcal$.
\end{theorem}

The next theorem holds for general $\fdiv$-divergences without the requirement of \Cref{asmp:fdiv:appendix}.
\begin{theorem}\label{thm:quant_error_fdiv_appendix}
    For any $k \ge 1$, we have 
    \[
        \adjustlimits\sup_{P, Q} \inf_{|\Scal|\le 2k}
        \left| 
            \Df{P}{Q} - \Df{P_\Scal}{Q_\Scal}
        \right|
        \le \frac{\fdiv(0) + \ftil(0)}{k} \,.
    \]
\end{theorem}
\begin{proof}  
    Assume $\fdiv(0) + \ftil(0) < \infty$. Otherwise, there is nothing to prove. 
    Fix two distributions $P, Q$ over $\Xcal$.
    Partition the measurable space $\Xcal$ into 
    \[
        \Xcal_1 = \left\{x \in \Xcal \, :\, \frac{\D P}{\D Q}(x) \le 1 \right\}\,,
        \quad\text{and,}\quad
        \Xcal_2 = \left\{x \in \Xcal \, :\, \frac{\D P}{\D Q}(x) > 1 \right\} \,,
    \]
    so that 
    \[
        \Df{P}{Q} = 
        \int_{\Xcal_1} \fdiv\left( \frac{\D P}{\D Q}(x)\right) \D Q(x)
        + 
        \int_{\Xcal_2} \ftil\left( \frac{\D Q}{\D P}(x)\right) \D P(x) 
        =: D_\fdiv^+(P \Vert Q) + D_{\ftil}^+(Q \Vert P) \,.
    \]
    We quantize $\Xcal_1$ and $\Xcal_2$ separately, 
    starting with $\Xcal_1$. Define sets $S_1, \cdots, S_k$ as 
    \[
        S_m = \left\{
        x \in \Xcal_1 \, :\, 
        \frac{\fdiv(0)(m-1)}{k} \le \fdiv\left( \frac{\D P}{\D Q}(x) \right) < \frac{\fdiv(0)m}{k} 
        \right\} \,,
    \]
    where the last set $S_k$ is also extended 
    to include $\{x \in \Xcal_1\,:\,\fdiv( (\D P / \D Q)(x)) = \fdiv(0)\}$.
    Since $\fdiv$ is nonincreasing on $(0, 1]$, it follows that $\sup_{x \in \Xcal_1} \fdiv( ({\D P}/{\D Q})(x)) \le \fdiv(0)$.
    As a result, the collection $\Scal = \{S_1, \cdots, S_k\}$ is a partition of $\Xcal_1$.
    This gives
    \begin{align} \label{eq:fdiv:quant:pf:1}
        \frac{\fdiv(0)}{k} \sum_{m=1}^k (m-1) \, Q[S_m] \le
        D_\fdiv^+(P \Vert Q) \le \frac{\fdiv(0)}{k} \sum_{m=1}^k m \, Q[S_m] \,.
    \end{align}
    Further, since $\fdiv$ is nonincreasing on $(0, 1]$, we also have 
    \[
        \frac{\fdiv(0) (m-1)}{k} \le
        \fdiv\left(\sup_{x \in F_m} \frac{\D P}{\D Q}(x)\right)  \le 
        \fdiv\left( \frac{P[F_m]}{Q[F_m]}\right)
        \le 
        \fdiv\left(\inf_{x \in F_m} \frac{\D P}{\D Q}(x)\right) \le \frac{\fdiv(0) m}{k} \,.
    \]
    Hence, it follows that
    \begin{align} \label{eq:fdiv:quant:pf:2}
        \frac{\fdiv(0)}{k} \sum_{m=1}^k (m-1) \, Q[S_m] \le
        D_\fdiv^+(P_{\Scal_1} \Vert Q_{\Scal_1}) \le \frac{\fdiv(0)}{k} \sum_{m=1}^k m \, Q[S_m] \,.
    \end{align}
    Putting \eqref{eq:fdiv:quant:pf:1} and \eqref{eq:fdiv:quant:pf:2} together gives 
    \begin{align}  \label{eq:fdiv:quant:pf:3a}
        \inf_{|\Scal_1| \le k} 
        \left|
        D_\fdiv^+(P \Vert Q) -
        D_\fdiv^+(P_{\Scal_1} \Vert Q_{\Scal_1})
        \right| \le 
        \frac{\fdiv(0)}{k} \sum_{m=1}^k Q[S_m] \le \frac{\fdiv(0)}{k} \,,
    \end{align}
    since $\sum_{m=1}^k Q[S_m] = Q[\Xcal_1] \le 1$.
    Repeating the same argument with $P$ and $Q$
    interchanged and replacing $\fdiv$ by $\ftil$ gives
    \begin{align}  \label{eq:fdiv:quant:pf:3b}
        \inf_{|\Scal_2| \le k} 
        \left|
        D_{\ftil}^+(Q \Vert P) -
        D_{\ftil}^+(Q_{\Scal_2} \Vert P_{\Scal_2})
        \right| \le  \frac{\ftil(0)}{k} \,.
    \end{align}
    To complete the proof, we upper bound the inf of $\Scal$ over all partitions of $\Xcal$ with $|\Scal|=k$ by 
    the inf over $\Scal = \Scal_1 \cup \Scal_2$ with 
    partitions 
    $\Scal_1$ of $\Xcal_1$
    and $\Scal_2$ of $\Xcal_2$, and
    $|\Scal_1| = |\Scal_2| = k$.
    Now, under this partitioning, 
    we have, 
    $D_\fdiv^+(P_{\Scal} \Vert Q_{\Scal}) 
    = D_\fdiv^+(P_{\Scal_1} \Vert Q_{\Scal_1})$
    and $D_{\ftil}^+(Q_{\Scal} \Vert P_{\Scal}) 
    = D_{\ftil}^+(Q_{\Scal_2} \Vert P_{\Scal_2})$. 
    Putting this together with the triangle inequality, 
    we get, 
    \begin{align*}
        &\quad \inf_{|\Scal|\le 2k}
        \Big| 
            \Df{P}{Q} - \Df{P_\Scal}{Q_\Scal}
        \Big| \\
        &\le 
        \inf_{\Scal = \Scal_1 \cup \Scal_2} \left\{
        \left| 
            D_\fdiv^+(P \Vert Q) - D_\fdiv^+(P_{\Scal} \Vert Q_{\Scal})
        \right|
        +
        \left| 
            D_{\ftil}^+(Q \Vert P) - D_{\ftil}^+(Q_{\Scal} \Vert P_{\Scal})
        \right| \right\} \\
        &=
        \inf_{|\Scal_1|\le k} 
        \left| 
            D_\fdiv^+(P \Vert Q) - D_\fdiv^+(P_{\Scal_1} \Vert Q_{\Scal_1})
        \right|
        + \inf_{|\Scal_2|\le k}
        \left| 
            D_{\ftil}^+(Q \Vert P) - D_{\ftil}^+(Q_{\Scal_2} \Vert P_{\Scal_2})
        \right| \\
        &\le \frac{\fdiv(0) + \ftil(0)}{k} \,.
    \end{align*}
\end{proof}

Now, combining \Cref{prop:fdiv:consistency} and \Cref{thm:quant_error_fdiv_appendix} leads to an upper bound for the overall estimation error.

\begin{theorem}\label{thm:fdiv:overall_est}
    Let $\Scal_k$ be a partition of $\Xcal$ such that $\abs{\Scal} = k \ge 2$ and its quantization error satisfies the bound in \Cref{thm:quant_error_fdiv_appendix}, i.e.,
    \begin{align*}
        \abs{\Df{P}{Q} - \Df{P_{\Scal_k}}{Q_{\Scal_k}}} \le \frac{f(0) + f^*(0)}{k}.
    \end{align*}
    Then, for any $n, m \ge 3$,
    \begin{align*}
        &\quad \expect\abs{\Df{\hat P_{\Scal_k, n}}{\hat Q_{\Scal_k, m}} - \Df{P}{Q}} \\
        &\le \big(\ConstI \log{n} + \ConstZTil \vee \ConstII\big) \alpha_{n}(P) + \big(\ConstITil \log{m} + \ConstZ \vee \ConstIITil\big) \alpha_{m}(Q) \\
        &\quad + \big(\ConstI + \ConstZTil \vee \ConstII\big) \beta_{n}(P) + \big(\ConstITil + \ConstZ \vee \ConstIITil\big) \beta_{m}(Q) + \frac{f(0) + f^*(0)}{k} \\
        &\le \big(c_1 \log{(n \wedge m)} + c_2\big) \left(\sqrt\frac{k}{n \wedge m} + \frac{k}{n \wedge m}  \right) + \frac{f(0) + f^*(0)}{k}\,,
    \end{align*}
    where $c_1 =  \ConstI + \ConstITil$
    and $c_2 = \ConstII \vee \ConstZTil +\ConstIITil \vee \ConstZ$.
\end{theorem}

According to \Cref{thm:fdiv:overall_est}, a good choice of quantization level $k$ is of order $\Theta(n^{1/3})$ which balances between the two types of errors.

\section{Experimental details}
\label{sec:a:experiments}
\begin{table}[t]
\caption{Add-constant estimators.}
\label{tab:add_const}
\vspace{0.05in}
\centering
\begin{tabular}{lll}
\hline
Braess-Sauer & Krichevsky-Trofimov & Laplace                                      \\ \hline
\begin{tabular}[c]{@{}l@{}} $b_a = 1/2$ if $a$ does not appear \\ $b_a = 1$ if $a$ appears once \\ $b_a = 3/4$ if $a$ appears more than once \end{tabular}       &  $b \equiv 1/2$ & $b \equiv 1$                   \\ \hline
\end{tabular}
\end{table}

We investigate the empirical behavior of the divergence frontier and the frontier integral on both synthetic and real data.
Our main findings are: 1) the statistical error bound is tight---it approximately reveals the rate of convergence of the plug-in estimator.
2) The smoothed distribution estimators improve the estimation accuracy.
For simplicity, we consider $m = n$ throughout this section.

\myparagraph{Performance Metric}
We are interested in the estimation of the divergence frontier $\Fcal(P, Q)$ and the frontier integral $\mray(P, Q)$ using estimators $\Fcal(\hat P_n, \hat Q_n)$ and $\mray(\hat P_n, \hat Q_n)$, respectively.
We measure the quality of estimation using the absolute error, which is defined as
\[
        \sup_{\lambda \in [0.01, 0.99]} \left\{ \abs{\kl(\hat P_n \Vert \hat R_\lambda) - \kl(P \Vert R)} + \abs{\kl(\hat Q_n \Vert \hat R_\lambda) - \kl(Q \Vert R)} \right\}
\]
for the divergence frontier (cf. \Cref{cor:consis_df} with $\lambda_0 = 0.01$),
and, $\lvert \mray(\hat P_n, \hat Q_n) - \mray(P, Q) \rvert$ for the frontier integral.
Here $\hat R_\lambda := \lambda \hat P_n + (1 - \lambda) \hat Q_n$.
For the real data, we measure the error of estimating $\Fcal(P_{\Scal_k}, Q_{\Scal_k})$ by $\Fcal(\hat P_{\Scal_k, n}, \hat Q_{\Scal_k, n})$ and similarly for $\mray$.
The results for the divergence frontier is almost identical to the result for the frontier integral.
We present both of them in the plots but focus on the latter in the text.

\subsection{Synthetic data}

We focus on the case when the support is finite and illustrate the statistical behavior of the \mauveray on synthetic data.

\paragraph{Settings.}
Let $k = \abs{\Xcal}$ be the support size.
Following the experimental settings in~\cite{orlitsky2015turing}, we consider three types of distributions: 1) the Zipf$(r)$ distribution with $r \in \{0, 1, 2\}$ where $P(i) \propto i^{-r}$. Note that Zipf$(r)$ is regularly varying with index $-r$; see, e.g.,~\cite[Appendix B]{shorack2000probability}. 2) the Step distribution where $P(i) = 1/2$ for the first half bins and $P(i) = 3/2$ for the second half bins. 3) the Dirichlet distribution $\dir(\alpha)$ with $\alpha \in \{\mathbf{1}/2, \mathbf{1}\}$.
In total, there are 6 different distributions.
Since the \mauveray is symmetric, there are $21$ different pairs of $(P, Q)$.
For each pair $(P, Q)$, we generate i.i.d.~samples of size $n$ from each of them, and then compute the absolute error.
We repeat the process $100$ times and report its mean and standard error, which is referred to as the Monte Carlo estimate of the expected absolute error.

\paragraph{Statistical error.}
To study the tightness of the statistical error bounds \eqref{eq:stat_error_ray}, we compare both the distribution-free bound (``Bound'') and the distribution-dependent bound (``Oracle bound'') with the Monte Carlo estimate (``Monte Carlo'').
We call the distribution-free bound the ``bound'' and the distribution-dependent bound the ``oracle bound''.
We consider three different experiments.
First, we fix the support size $k = 10^3$ and increase the sample size $n$ from $10^3$ to $10^4$.
Second, we fix $n = 2\times 10^4$ and increase $k$ from $10$ to $10^4$.
Third, we fix $k = 10^3$ and $n = 10^4$, and set $Q$ to be the Zipf$(r)$ with $r$ ranging from $0$ to $2$.
For each of these experiments, we give four typical plots among all pairs of distributions we consider.
Note that the two bounds are divided by the same constant for the sake of comparison.

\begin{figure}[t]
    \centering
    \adjincludegraphics[width=\textwidth, trim=0.0in 1in 0.0in 0.0in, clip=true]{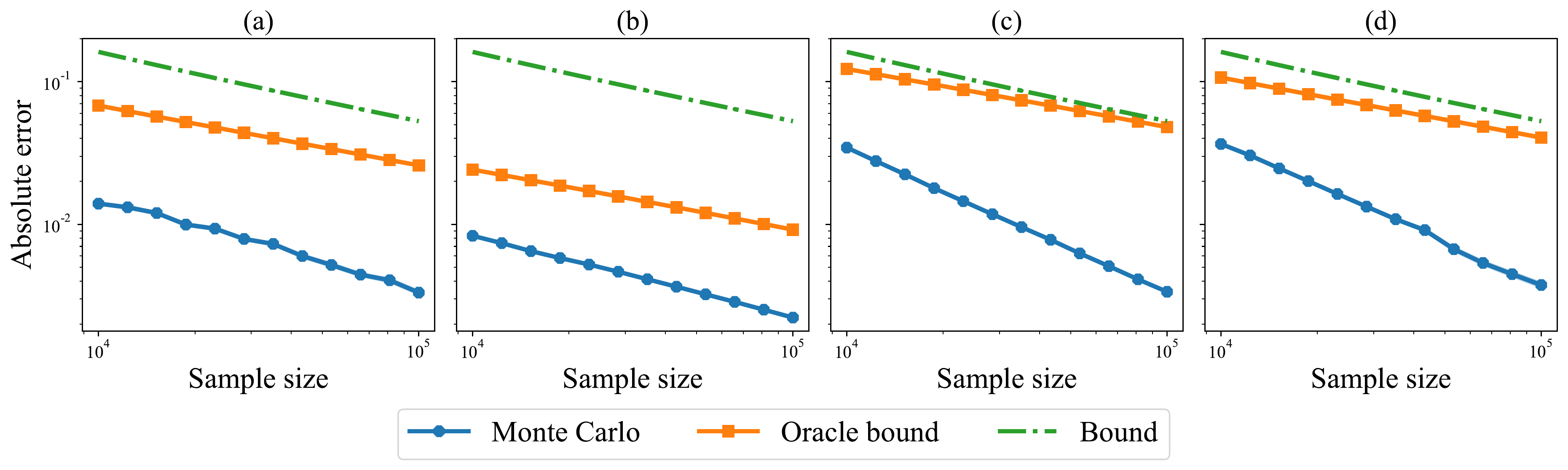}
    \adjincludegraphics[width=\textwidth, trim=0.0in 0.0in 0.0in 0.33in, clip=true]{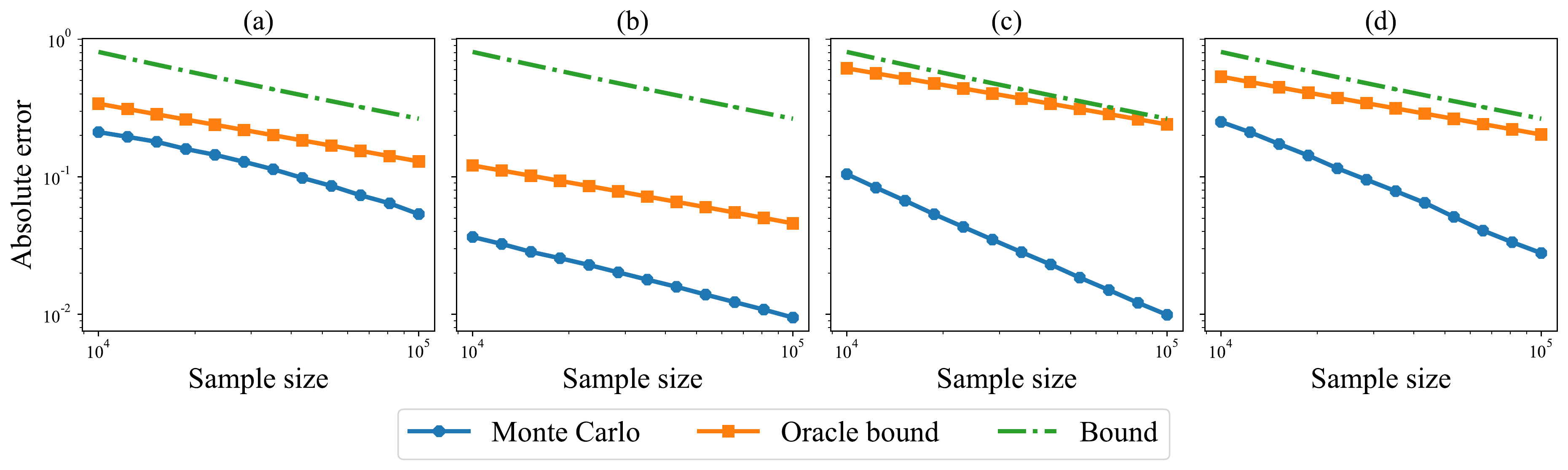}
    \caption{Statistical error versus sample size on synthetic data with $k = 10^3$ (log-log scale) for the frontier integral \textbf{(top)} and the divergence frontier \textbf{(bottom)}.
    \textbf{(a)}: Zipf$(2)$ and $\dir(\mathbf{1})$; \textbf{(b)}: Zipf$(2)$ and Zipf$(2)$; \textbf{(c)}: Zipf$(0)$ and Zipf$(0)$; \textbf{(d)}: $\dir(\mathbf{1})$ and $\dir(\mathbf{1}/2)$.}
    \label{fig:bound_nvary_appendix}
\end{figure}

As shown in \Cref{fig:bound_nvary_appendix}, the two bounds decreases with $n$ at a similar rate.
The oracle bound demonstrates the largest improvement compared to the bound when both $P$ and $Q$ have fast-decaying tails (i.e., with index $-2$).
In some cases, the Monte Carlo estimate demonstrates a similar rate of convergence as the bounds; while, in other cases, the Monte Carlo estimate can have a faster rate.
This suggests that the bound \eqref{eq:stat_error_ray} is at least close to being tight up to a multiplicative constant.

\begin{figure}[t]
    \centering
    \adjincludegraphics[width=\textwidth, trim=0.0in 1in 0.0in 0.0in, clip=true]{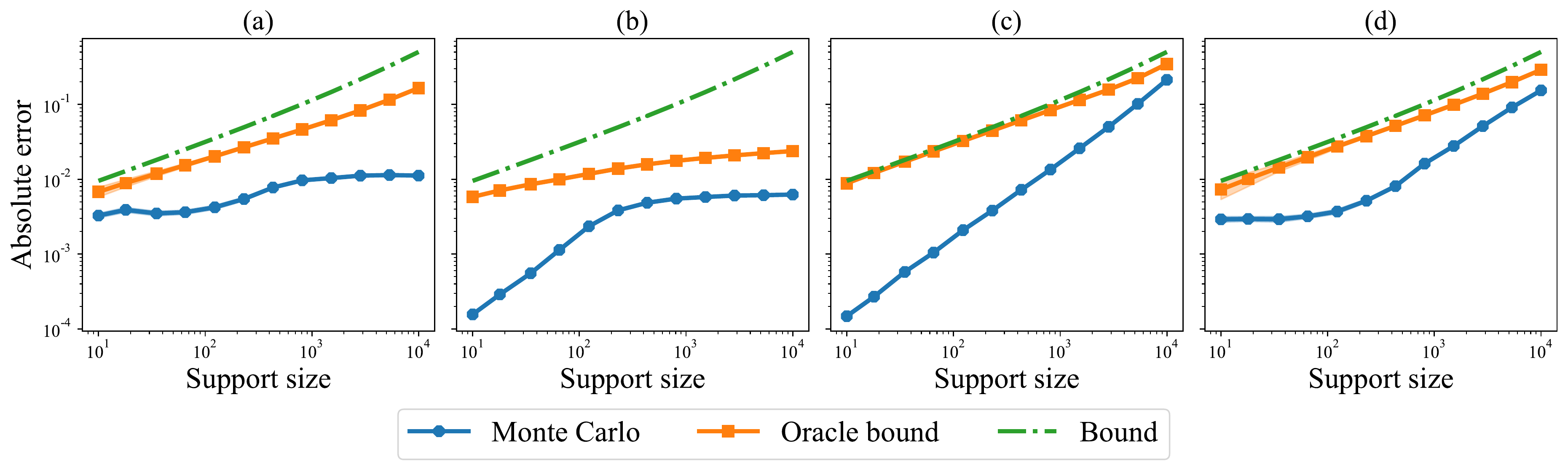}
    \adjincludegraphics[width=\textwidth, trim=0.0in 0in 0.0in 0.33in, clip=true]{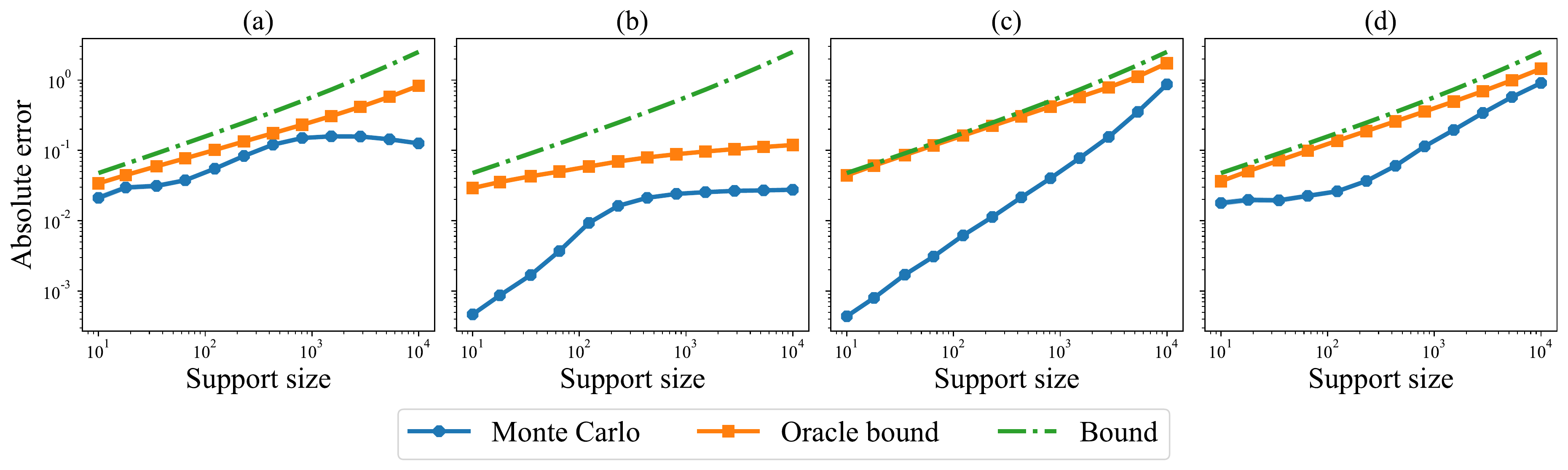}
    \caption{Statistical error versus support size on synthetic data with $n = 2\times 10^4$ (log-log scale) for the frontier integral \textbf{(top)} and the divergence frontier \textbf{(bottom)}. \textbf{(a)}: Zipf$(2)$ and $\dir(\mathbf{1})$; \textbf{(b)}: Zipf$(2)$ and Zipf$(2)$; \textbf{(c)}: Zipf$(0)$ and Zipf$(0)$; \textbf{(d)}: $\dir(\mathbf{1})$ and $\dir(\mathbf{1}/2)$.}
    \label{fig:bound_kvary_appendix}
\end{figure}

\Cref{fig:bound_kvary_appendix} shows that the oracle bound increases with $k$ at a slower rate than the one of the bound.
In fact, it is much slower when both $P$ and $Q$ decay fast.
For the Monte Carlo estimate, it can have either a slower or faster rate than the bound depending on the underlying distributions.

\begin{figure}[t]
    \centering
    \adjincludegraphics[width=\textwidth, trim=0.0in 1in 0.0in 0.0in, clip=true]{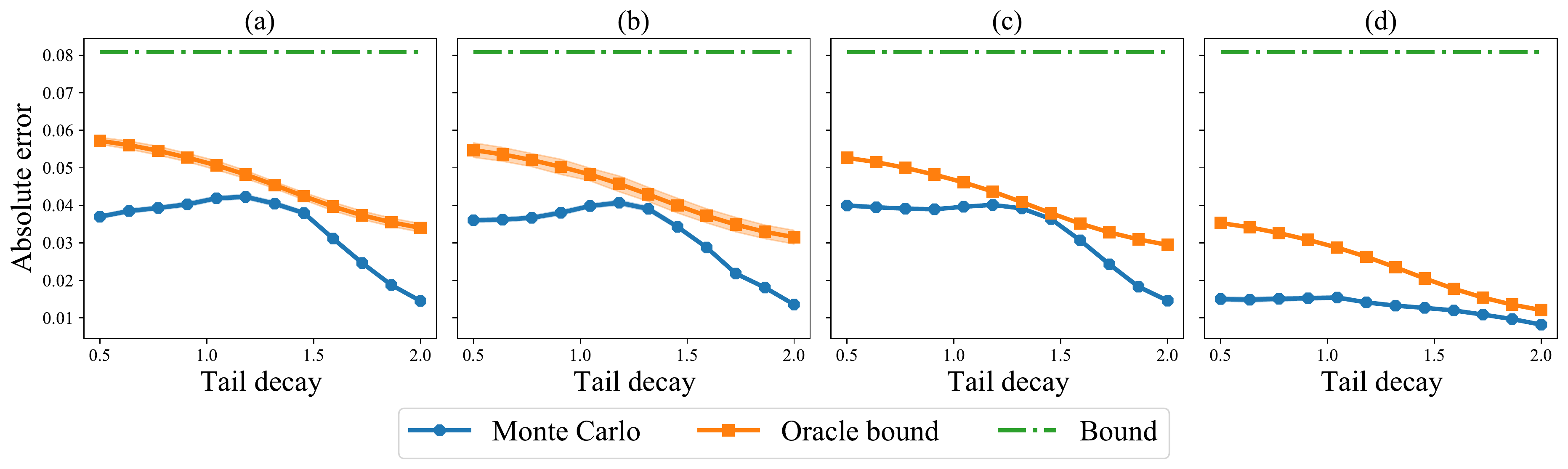}
    \adjincludegraphics[width=\textwidth, trim=0.0in 0in 0.0in 0.33in, clip=true]{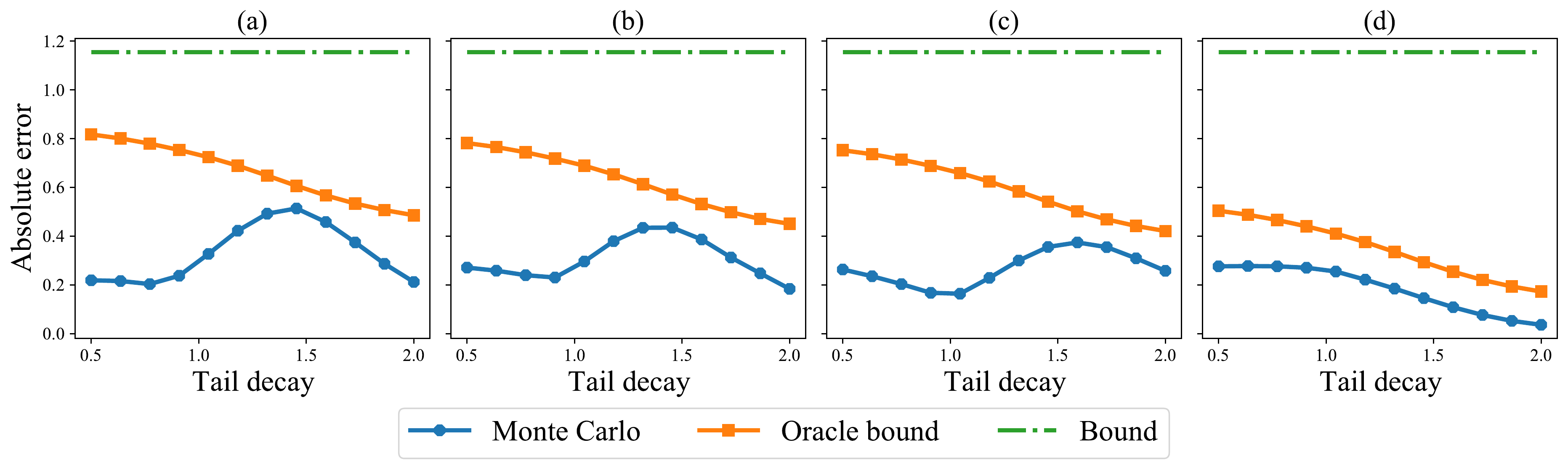}
    \caption{Absolute error versus decaying index of $Q$ on synthetic data with $k = 10^3$ and $n = 10^4$ (log-log scale) for the frontier integral \textbf{(top)} and the divergence frontier \textbf{(bottom)}. \textbf{(a)}: $P \sim \dir(\mathbf{1})$; \textbf{(b)}: $P \sim \dir(\mathbf{1}/2)$; \textbf{(c)}: $P \sim \zipf(1)$; \textbf{(d)}: $P \sim \zipf(2)$.}
    \label{fig:bound_qvary_appendix}
\end{figure}

The results for the third experiment is in \Cref{fig:bound_qvary_appendix}.
While the bound remains the same for different tails of $Q$, the oracle bound is adapted to the decaying index of $Q$.
The absolute error of the Monte Carlo estimate is usually increasing in the beginning and then decreasing after some threshold.

\paragraph{Distribution estimators.}
We then compare 4 different distribution estimators with the empirical measures (``Empirical'') as discussed in~\cite{orlitsky2015turing}.
For each $a \in \Xcal$, let $n_a$ be the number of times $a$ appears in the sample $\{X_i\}_{i=1}^n$ and let $\varphi_t$ be the number of symbols appearing $t$ times in the sample.
The \emph{(modified) Good-Turing} estimator is defined as $\hat P_{\mathrm{GT}, n}(a) \propto n_a$ if $n_a > \varphi_{n_a+1}$ and $\hat P_{\mathrm{GT}, n}(a) \propto [\varphi_{n_a+1} + 1] (n_a+1) / \varphi_{n_a}$ otherwise.
The remaining three estimators are all based on the add-$b$ smoothing introduced in \Cref{sec:consist}.
For the \emph{Braess-Sauer} estimator, the parameter $b = b_a$ is data-dependent and chosen as $b_a = 1/2$ if $n_a = 0$, $b_a = 1$ if $n_a = 1$ and $b_a = 3/4$ otherwise.
For the \emph{Krichevsky-Trofimov} estimator, the parameter $b \equiv 1/2$.
For the \emph{Laplace} estimator, the parameter $b \equiv 1$.
See \Cref{tab:add_const} for a summary.

\begin{figure}[t]
    \centering
    \adjincludegraphics[width=\textwidth, trim=0.0in 1in 0.0in 0.0in, clip=true]{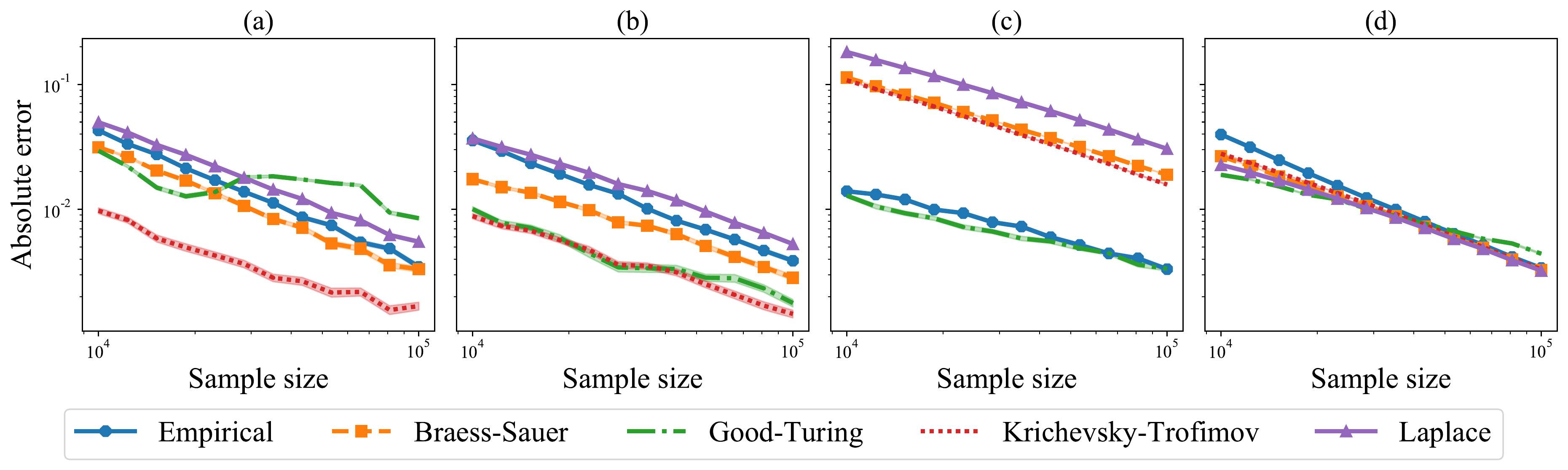}
    \adjincludegraphics[width=\textwidth, trim=0.0in 0in 0.0in 0.33in, clip=true]{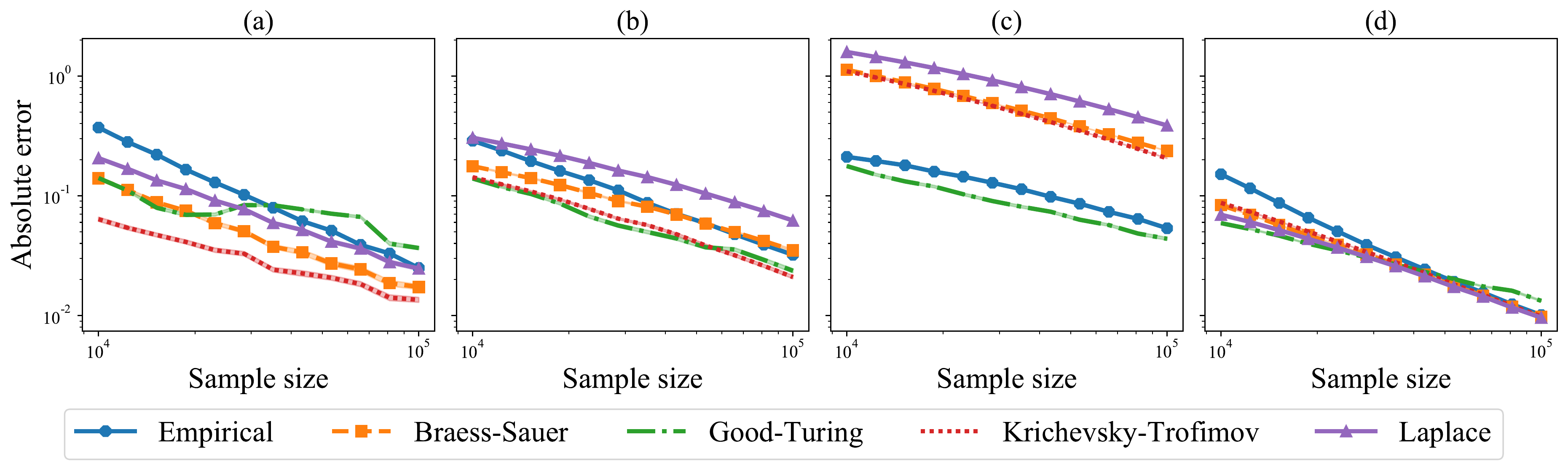}
    \caption{Statistical error versus sample size on synthetic data with $k = 10^3$ (log-log scale) for the frontier integral \textbf{(top)} and the divergence frontier \textbf{(bottom)}. \textbf{(a)}: $\zipf(1)$ and Step; \textbf{(b)}: $\zipf(0)$ and $\dir(\mathbf{1}/2)$; \textbf{(c)}: $\zipf(2)$ and $\dir(\mathbf{1})$; \textbf{(d)}: $\zipf(1)$ and $\zipf(1)$.}
    \label{fig:smoothing_nvary_appendix}
\end{figure}

We consider the same three experiments as for the statistical error.
As shown in \Cref{fig:smoothing_nvary_appendix}, the rate of convergence in $n$ of all estimators are similar except for some fluctuations of the Good-Turing estimator.
When $P = Q$ (i.e., $\zipf(1)$), the add-constant estimators outperforms the empirical measures slightly while the Good-Turing estimator performs better than the empirical measures for relatively small sample size and performs worse as the sample size increases.
When one of the distribution has a fast-decaying tail (i.e., $P \sim \zipf(2)$), the absolute error of the add-constant estimators are much larger than the one of empirical measures, while the Good-Turing estimator has a similar performance as empirical measures.
When $P$ and $Q$ are different and do not have fast-decaying tails, the Krichevsky-Trofimov estimator enjoys the largest improvement compared to the empirical measures.

\begin{figure}[t]
    \centering
    \adjincludegraphics[width=\textwidth, trim=0.0in 1in 0.0in 0.0in, clip=true]{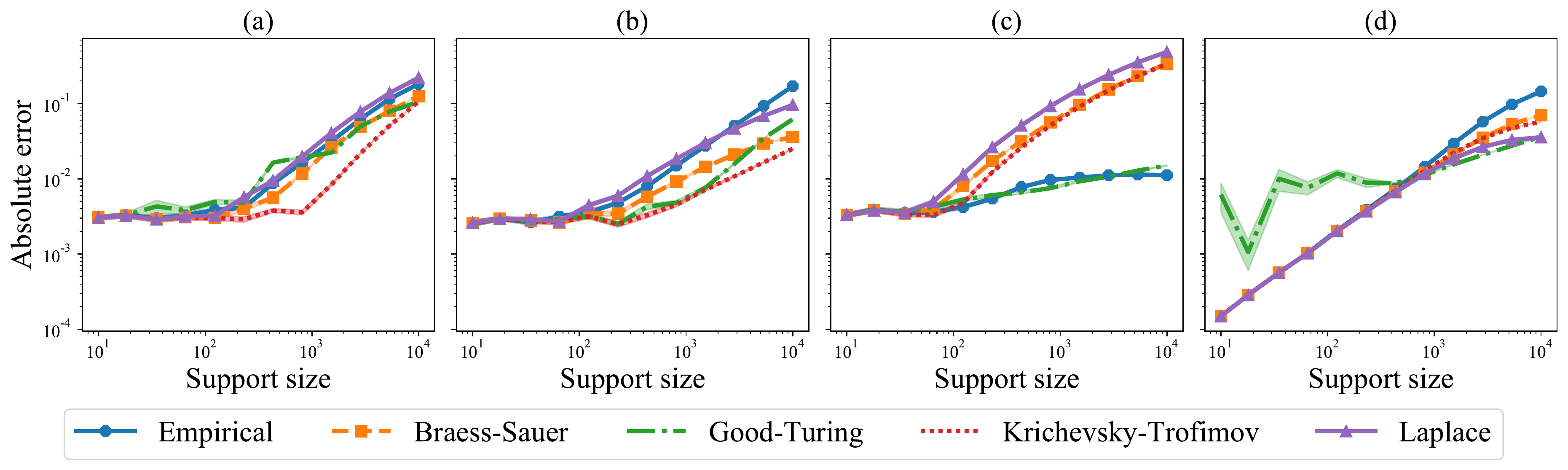}
    \adjincludegraphics[width=\textwidth, trim=0.0in 0in 0.0in 0.33in, clip=true]{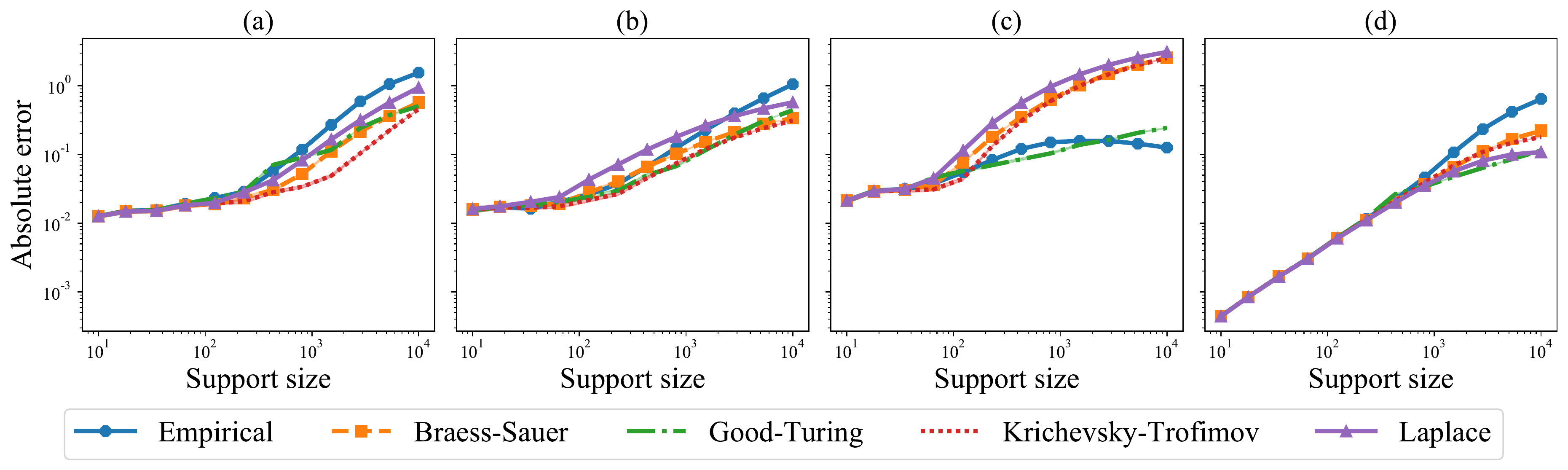}
    \caption{Absolute error versus support size on synthetic data with $n = 2 \times 10^4$ (log-log scale) for the frontier integral \textbf{(top)} and the divergence frontier \textbf{(bottom)}. \textbf{(a)}: $\zipf(1)$ and Step; \textbf{(b)}: $\zipf(0)$ and $\dir(\mathbf{1}/2)$; \textbf{(c)}: $\zipf(2)$ and $\dir(\mathbf{1})$; \textbf{(d)}: $\zipf(1)$ and $\zipf(1)$.}
    \label{fig:smoothing_kvary_appendix}
\end{figure}

\Cref{fig:smoothing_kvary_appendix} presents the results for increasing support size.
The findings are similar to the ones in the first experiment except that the absolute error is increasing here rather than decreasing.

\begin{figure}[t]
    \centering
    \adjincludegraphics[width=\textwidth, trim=0.0in 1in 0.0in 0.0in, clip=true]{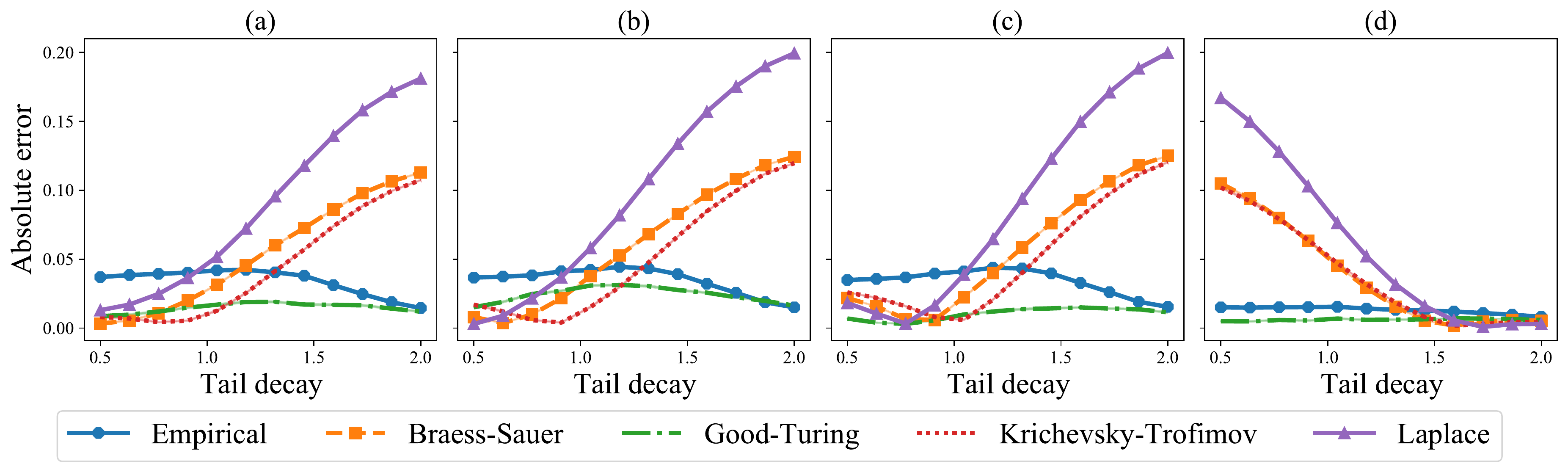}
    \adjincludegraphics[width=\textwidth, trim=0.0in 0in 0.0in 0.33in, clip=true]{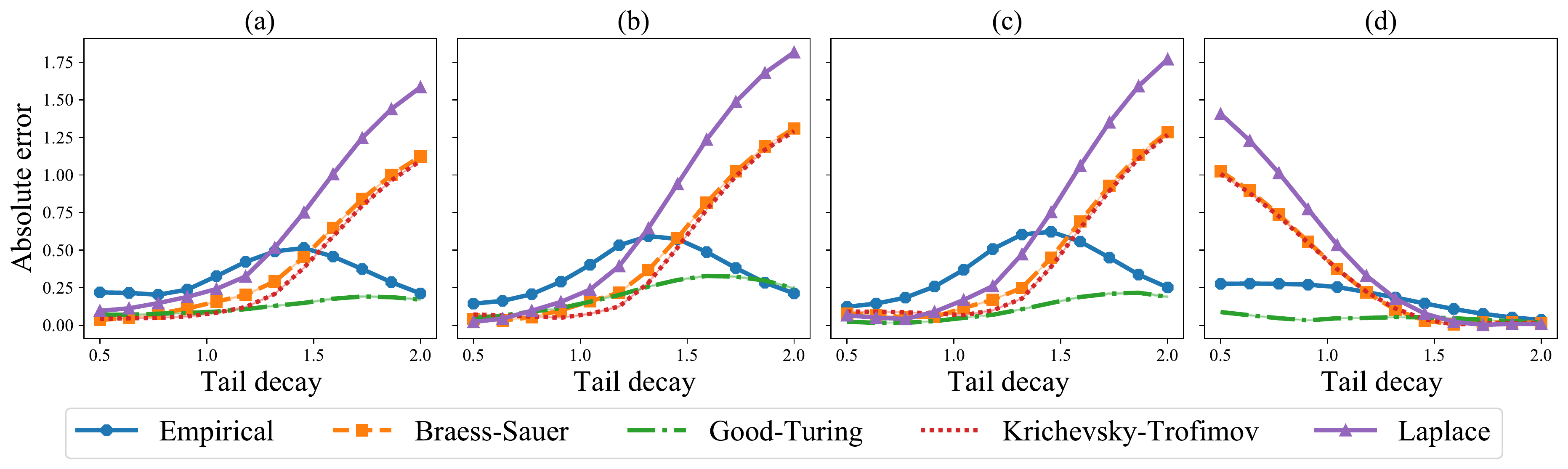}
    \caption{Absolute error versus sample size on synthetic data with $k = 10^3$ and $n = 10^4$ (log-log scale) for the frontier integral \textbf{(top)} and the divergence frontier \textbf{(bottom)}. \textbf{(a)}: $P \sim \dir(\mathbf{1})$; \textbf{(b)}: $P \sim \mbox{Step}$; \textbf{(c)}: $\zipf(0)$; \textbf{(d)}: $\zipf(2)$.}
    \label{fig:smoothing_qvary_appendix}
\end{figure}

\Cref{fig:smoothing_qvary_appendix} shows that the Good-Turing estimator is relatively more robust to the tail decaying index than other estimators.
When $P \sim \zipf(2)$, the absolute error of the add-constant estimators is much larger than the one of the empirical measures in the beginning and then becomes slightly smaller in the end.
In other cases, this behavior is reversed.

To summarize, when two distributions are the same, all estimators performs similarly with the Good-Turing estimator being the worst.
When there is one distribution whose tail decays fast, the Good-Turing estimator slightly outperforms the empirical measure; while the add-constant estimators have much larger absolute errors.
When the tails of both distributions decay slowly, the Krichevsky-Trofimov estimator has the best performance over all estimators.

\myparagraph{Quantization error}
We study the bound on the quantization error as in \eqref{eq:quant_error}.
Since the absolute error is always zero when $P = Q$, we have $21 - 6 = 15$ different pairs of $(P, Q)$.
We consider three different quantization strategies: 1) the \emph{uniform} quantization which quantizes the distributions into equally spaced bins based on their original ordering; 2) the \emph{greedy} quantization which sorts the bins according to the ratios $\{P(a)/Q(a)\}_{a \in \Xcal}$ and then add split one bin at a time so that the \mauveray is maximized; 3) the \emph{oracle} quantization we used to prove \eqref{eq:quant_error}; see also \Cref{fig:quantization_appendix}.

As shown in \Cref{fig:quantization_appendix}, the absolute error of the oracle quantization can have a faster rate than $O(k^{-1})$ in some cases.
To be more specific, when both $P$ and $Q$ have slow-decaying tails, its absolute error decays roughly as $O(k^{-1.7})$; when one of them has fast-decaying tail, its absolute error decays slower than $O(k^{-1})$ in the beginning and then faster than $O(k^{-1})$.
Comparing different quantization strategies, the oracle quantization always outperforms the greedy one.
When either $P$ or $Q$ is not ordered, the uniform quantization has the worst performance. When both $P$ and $Q$ are ordered, its absolute error is not monotonic---it is quite small in the beginning and then becomes larger.

\begin{figure}[t]
    \centering
    \includegraphics[width=\textwidth]{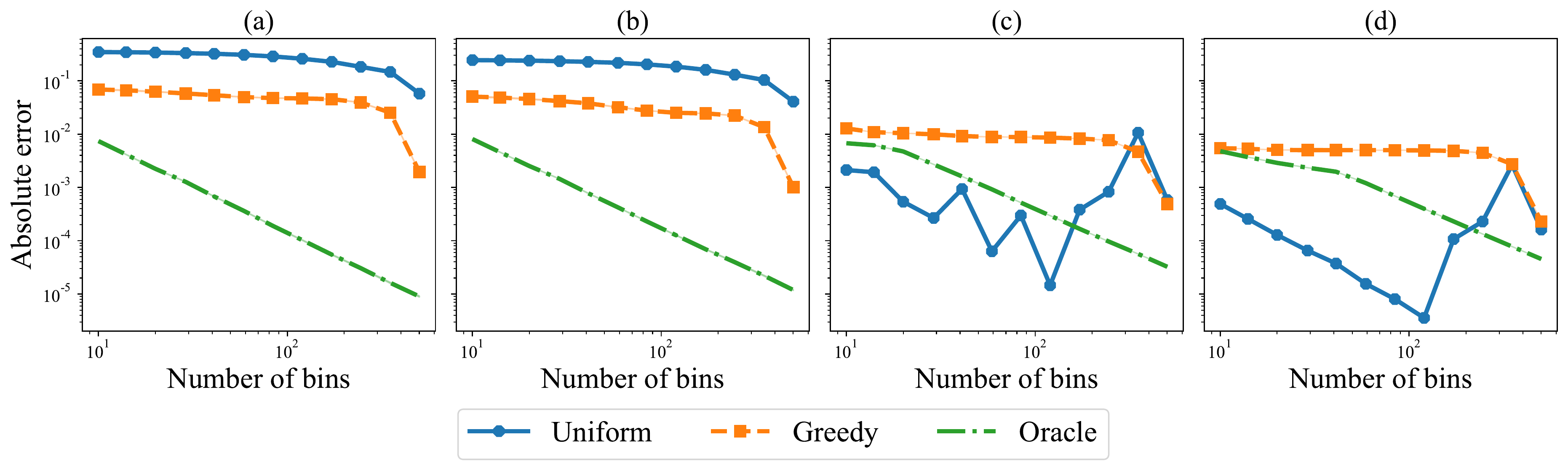}
    \caption{Absolute error versus number of bins for different quantization strategies with support size $600$ (log-log scale). \textbf{(a)}: $\dir(\mathbf{1})$ and $\dir(\mathbf{1}/2)$; \textbf{(b)}: $\zipf(0)$ and $\dir(\mathbf{1}/2)$; \textbf{(c)}: $\zipf(2)$ and Step; \textbf{(d)}: $\zipf(1)$ and $\zipf(2)$.}
    \label{fig:quantization_appendix}
\end{figure}

\subsection{Real data}
We analyze the performance of the bounds as well as the various smoothed estimators in the context evaluating generative models for images and text using divergence curves.
All experiments models are trained on a 
workstation with $8$ Nvidia Quadro RTX GPUs ($24$G memory each).
The image experiments were trained with 2 GPUs at once while the text ones used all $8$.

\myparagraph{Tasks and datasets}
We consider two domains: images and text. 
For the image domain, we train a
generative model for the CIFAR-10 dataset~\cite{krizhevsky2009learning}
based on StyleGAN2-Ada~\cite{karras2020ada}.
We use the publicly available code\footnote{
\url{https://github.com/NVlabs/stylegan2-ada-pytorch}
} with their default hyperparameters and train on 2 GPUs. 
In order to enable the code to run faster, we make two architectural simplifications: (a) we reduce the channel dimensions for each convolution layer in the generator from $512$ to $256$, and, (b) we reduce the number of styled convolution layers for each resolution from $2$ to $1$. In particular, the latter effectively cuts the number of convolution layers in half. 
This leads to a 6.6x reduction in running time at the cost of a slightly worse FID~\cite{heusel2017gans} of $4.7$ rather than the $2.4$ of the original network.
In order to compute the divergence frontier, 
we use the test set of $10000$ images as the target distribution $P$ 
and we sample $10000$ images from the generative model as the model distribution $Q$. 

For the text domain, we finetune a pretrained GPT-2~\cite{radford2019language} model with 124M parameters (i.e., GPT-2 small)
on the Wikitext-103 dataset~\cite{merity2017pointer}.
We use the open-source HuggingFace Transformers library~\cite{wolf2020transformers} for training.
To form a sufficiently large evaluation set, we finetune on 90\% of the wikitext-103 training dataset, and use the remaining 10\% plus the validation set as an evaluation set.
Finetuning is done on 4 GPUs  for 2k iterations, with sequences of 1024 tokens and a batch size of 8 sequences.
For generation, we split the evaluation set into 10k sequences of 500 tokens, and split each sequence into a prefix of length 100 and a continuation of length 400.
The prefix paired with the continuation (a ``completion'') is considered a sample from $P$.
Using the finetuned model we generate a continuation for each prefix using top-$p$ sampling with $p=0.9$.
Each prefix paired with its generated continuation is considered a sample from $Q$.

\myparagraph{Settings}
In order to compute the divergence frontier, we jointly quantize $P$ and $Q$,
not directly in a raw image/text space, but in a feature space~\cite{sajjadi2018assessing,kynknniemi2019improved,heusel2017gans}.
Specifically, we represent each image by its features from a pretrained ResNet-50 model~\cite{he2016deep}, 
and each text generation by its terminal hidden state under a pretrained the 774M GPT-2 model (i.e., GPT-2 large). 
In order to quantize these features, we 
learn a $4$ or $5$ dimensional
embedding of the image/text features using a deep network which maintains the neighborhood structure of the data while encouraging the features to be uniformly distributed on the unit sphere~\cite{sablayrolles2018spreading},
and simply quantize these embeddings on a uniform lattice with $k$ bins. 
For each support size $k$, this gives us quantized distributions $P_{\Scal_k}, Q_{\Scal_k}$. We then sample $n$ i.i.d points each from these distributions and consider the empirical distributions $\hat P_{\Scal_k, n}, \hat Q_{\Scal_k, n}$ as well as the add-constant and Good-Turing estimators computed form these samples.
We repeat this $100$ times to
a Monte Carlo estimate of the expected absolute error $\expect|\mray(\hat P_{\Scal_k, n}, \hat Q_{\Scal_k, n}) - \mray(P_{\Scal_k}, Q_{\Scal_k})|$ as well as its standard error. 

\myparagraph{Statistical error}
We compare the distribution-dependent bound (``oracle bound'') and the distribution-free bound (``bound'') to the Monte Carlo estimates described above. 
We consider two experiments.
First, we fix the support size $k$ and vary the sample size $n$ from $100$ to $25000$. 
Second, we fix the sample size $n$ and vary the support size $k$ from $8$ to $2048$ in powers of $2$.

We observe \Cref{fig:real_bound_nvary_appendix} 
that both the distribution-free and distribution-dependent bounds decrease with the sample size $n$ at a similar rate. For $k=1024$ or $k=2048$, we observe that the bound has approximately the same slope as the Monte Carlo estimate in log-log scale; this means that they exhibit a near-identical rate in $n$. On the other hand, the Monte Carlo estimates exhibit fast rates of convergence than the bound for $k=64$ or $k = 128$. Therefore, the bounds capture the worst-case behavior of real image and text data. 

Next, we see from \Cref{fig:real_bound_kvary_appendix} that the two bounds again exhibit near-identical rates with the support size $k$. We observe again that the slope of the Monte Carlo estimate and that of the bounds are close for $n=1000$, indicating a similar scaling with respect to $k$. However, the Monte Carlo estimate grows faster than the bound for $n=10000$.

\myparagraph{Distribution estimators}
As in the previous section, we compare the empirical estimator, 
the (modified) Good-Turing estimator, and three add-$b$ smoothing estimators, namely Laplace, Krichevsky-Trofimov and Braess-Sauer. 
We consider the same two experiments as for the statistical error. 

From \Cref{fig:real_smoothing_nvary_appendix}, we see that for $n > k$, we observe similar rates (i.e., similar slopes) for all estimators with respect to the sample size $n$. 
The absolute error of the Good-Turing estimator is the worst among all estimators considered for $k = 64$ or $k=128$ and $n$ large. However, for $k=1024$ or $k=2048$, the empirical estimator is the worst. The various add-$b$ estimators work the best in the regime of $n < k$, where each add-$b$ estimator attains the smallest error at a different $n$. In particular, the Laplace estimator is the best or close to the best in all each of the settings considered. 

\Cref{fig:real_smoothing_kvary_appendix} shows the corresponding results for varying $k$. The results are similar to the previous setting, expect the error increases with $k$ rather than decreases. 

\myparagraph{Performance across training}
Next, we visualize the divergence frontiers and the corresponding frontier integral across training in \Cref{fig:real_training_appendix}. 
On the left, we plot the divergence curve at initialization (or with the pretrained model in case of text), at the first checkpoint (``Partly'') and the fully trained model (``Final''). We observe that the divergence frontiers for the fully trained model are closer to the origin than the partially trained ones or the model at initialization. This denotes a smaller loss of precision and recall for the fully trained model. The frontier integral, as a summary statistic, shows the same trend (right).

\begin{table}[t]
\caption{The frontier integral with pretrained and finetuned feature embedding models.}
\label{tab:feature_extractor}
\vspace{0.05in}
\centering
\adjustbox{max width=\textwidth}{
\begin{tabular}{@{}lllllllllll@{}}
\toprule
Quantization level $k$ & 2       & 4       & 8       & 16      & 32      & 64      & 128     & 256     & 512     & 1024    \\ \midrule
Pretrained            & 3.38e-5 & 2.64e-5 & 2.84e-4 & 6.95e-4 & 1.47e-3 & 3.25e-3 & 6.28e-3 & 1.18e-2 & 2.52e-2 & 5.09e-2 \\
Finetuned             & 7.23e-6 & 1.37e-4 & 3.98e-4 & 1.77e-3 & 2.36e-3 & 5.31e-3 & 9.84e-3 & 1.95e-2 & 3.49e-2 & 6.34e-2 \\ \bottomrule
\end{tabular}
}
\end{table}

\myparagraph{Fine-tuning the feature embedding model}
In our real data experiments, we follow the common practice in this line of research~\cite{sajjadi2018assessing,djolonga2020precision,pillutla2021mauve} and use a pre-trained feature embedding model to extract feature representations.
We also design a procedure to fine-tune the feature embedding model for comparing two distributions here.
Concretely, we compare the frontier integral using the following two feature embedding models.
First, we use a pretrained 4-layer ConvNet to extract feature embeddings for the generations of the StyleGAN.
Second, we reinitialize the output layer of the 4-layer ConvNet, finetune it to distinguish true images from generated ones, and use the finetuned ConvNet to extract features.
Finally, we compute the frontier integral using k-means clustering for various values of $k$.
As shown in \Cref{tab:feature_extractor}, the frontier integrals computed via the finetuned ConvNet are slightly larger than the ones without finetuning. This is as expected since the finetuned model usually gives a better feature representation in the sense of distinguishing distributions.

\begin{figure}[t]
    \centering
    \adjincludegraphics[width=\textwidth, trim=0.0in 1in 0.0in 0.0in, clip=true]{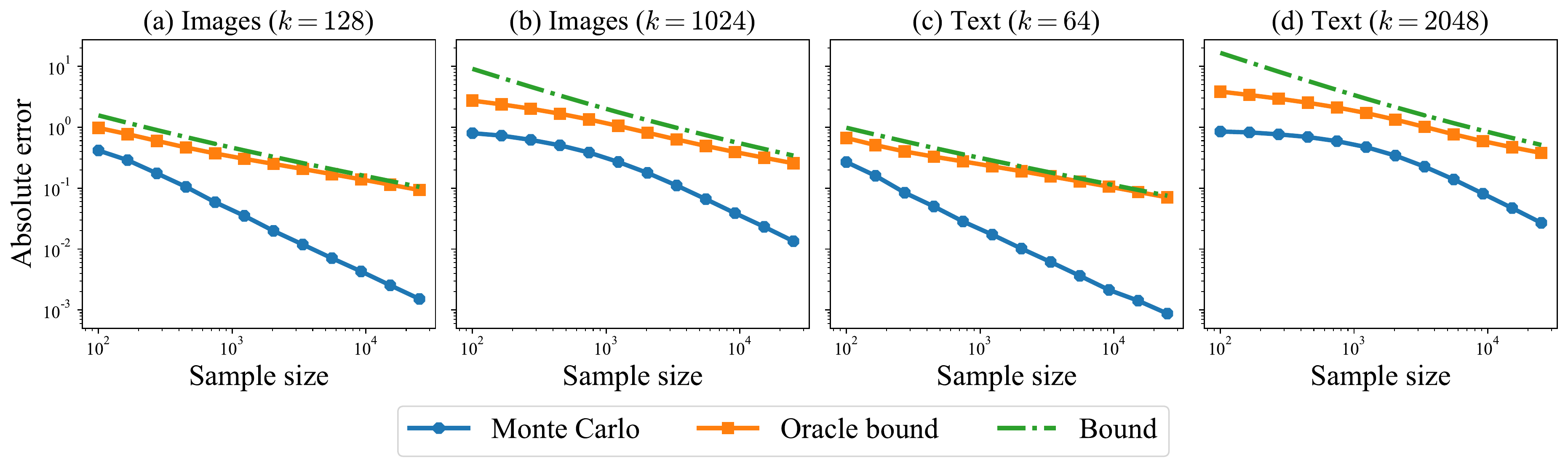}
    \adjincludegraphics[width=\textwidth, trim=0.0in 0in 0.0in 0.34in, clip=true]{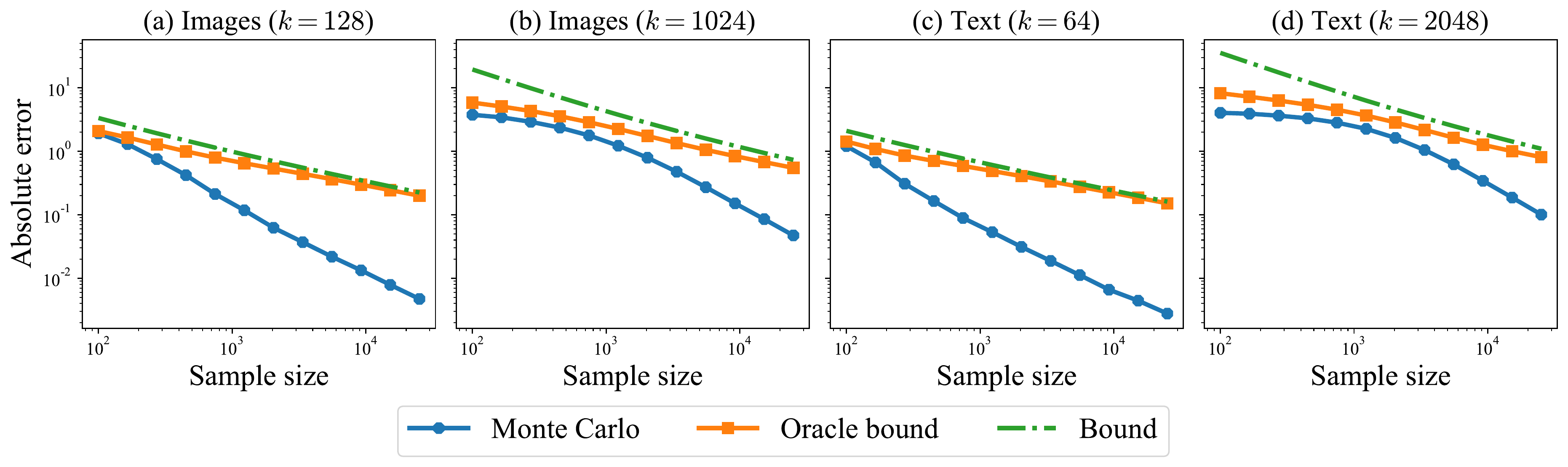}
    \caption{Absolute error versus sample size $n$ on real data (log-log scale) for the frontier integral \textbf{(top)} and the divergence frontier \textbf{(bottom)}. \textbf{Left Two}: Image data (CIFAR-10) with support sizes $k=128$ and $k=1024$. \textbf{Right Two}: Text data (WikiText-103) with support sizes $k=64$ and $k=2048$. The bounds are scaled by $15$.}
    \label{fig:real_bound_nvary_appendix}
\end{figure}

\begin{figure}[t]
    \centering
    \adjincludegraphics[width=\textwidth, trim=0.0in 1in 0.0in 0.0in, clip=true]{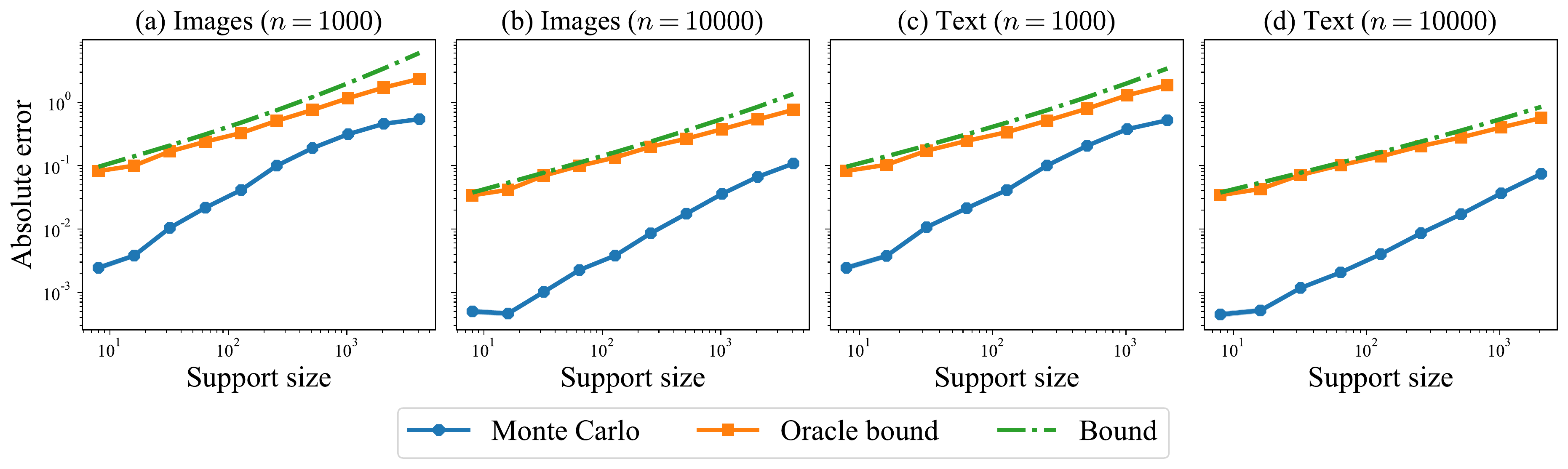}
    \adjincludegraphics[width=\textwidth, trim=0.0in 0in 0.0in 0.34in, clip=true]{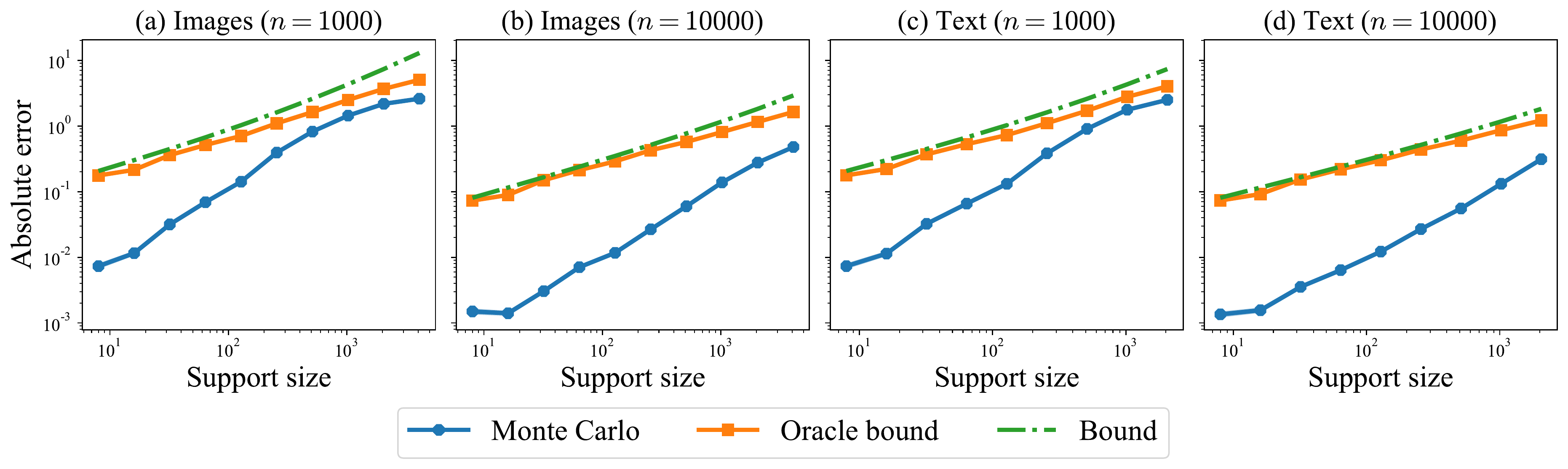}
    \caption{Absolute error versus support size $k$ on real data (log-log scale) for the frontier integral \textbf{(top)} and the divergence frontier \textbf{(bottom)}. \textbf{Left Two}: Image data (CIFAR-10) with sample sizes $n=1000$ and $n=10000$ \textbf{Right Two}: Text data (WikiText-103) with sample sizes $n=1000$ and $n=10000$. The bounds are scaled by $15$.}
    \label{fig:real_bound_kvary_appendix}
\end{figure}

\begin{figure}[t]
    \centering
    \adjincludegraphics[width=\textwidth, trim=0.0in 1in 0.0in 0.0in, clip=true]{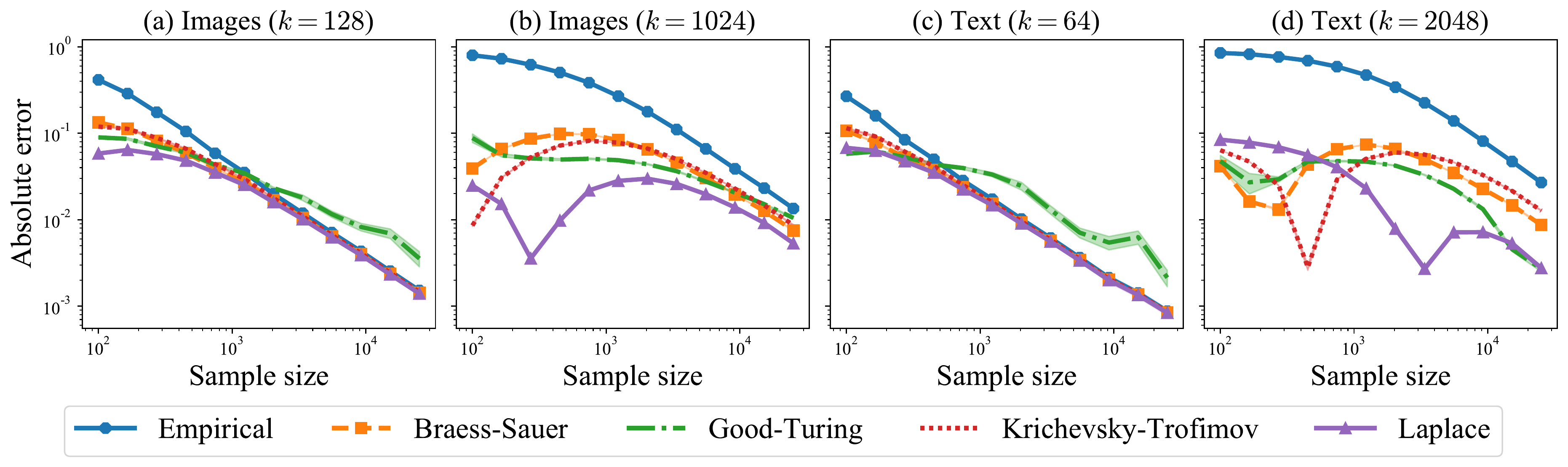}
    \adjincludegraphics[width=\textwidth, trim=0.0in 0in 0.0in 0.34in, clip=true]{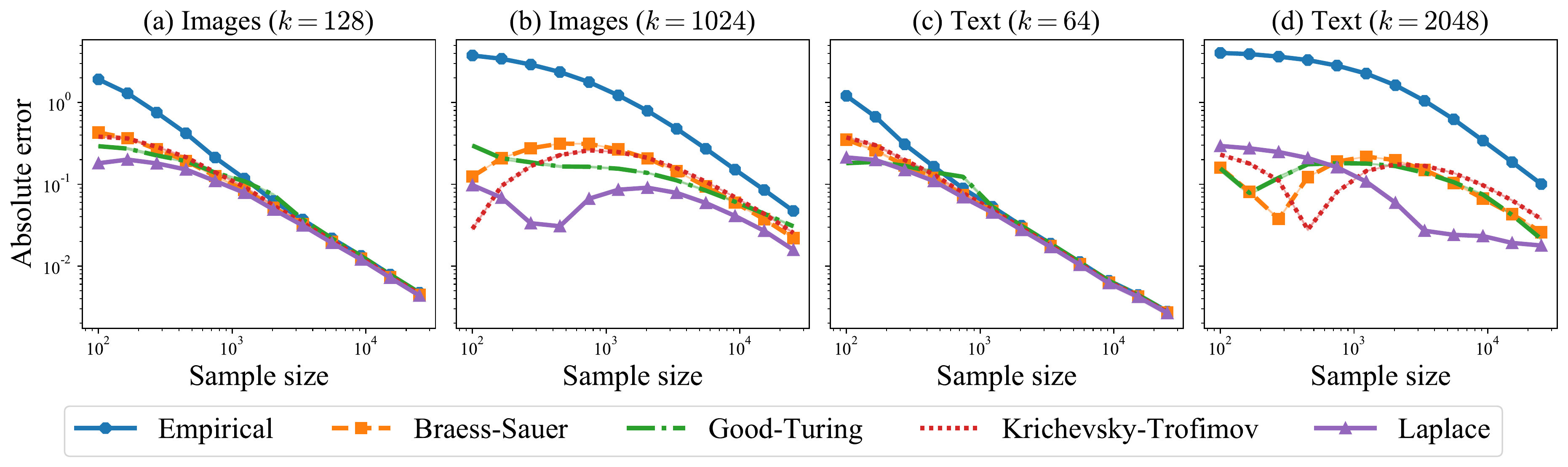}
    \caption{Absolute error versus sample size $n$ on real data (log-log scale) for the frontier integral \textbf{(top)} and the divergence frontier \textbf{(bottom)}. \textbf{Left Two}: Image data (CIFAR-10) with support size $k=128$ and $k=1024$ \textbf{Right Two}: Text data (WikiText-103) with support size $k=64$ and $k=2048$.}
    \label{fig:real_smoothing_nvary_appendix}
\end{figure}

\begin{figure}[t]
    \centering
    \adjincludegraphics[width=\textwidth, trim=0.0in 1in 0.0in 0.0in, clip=true]{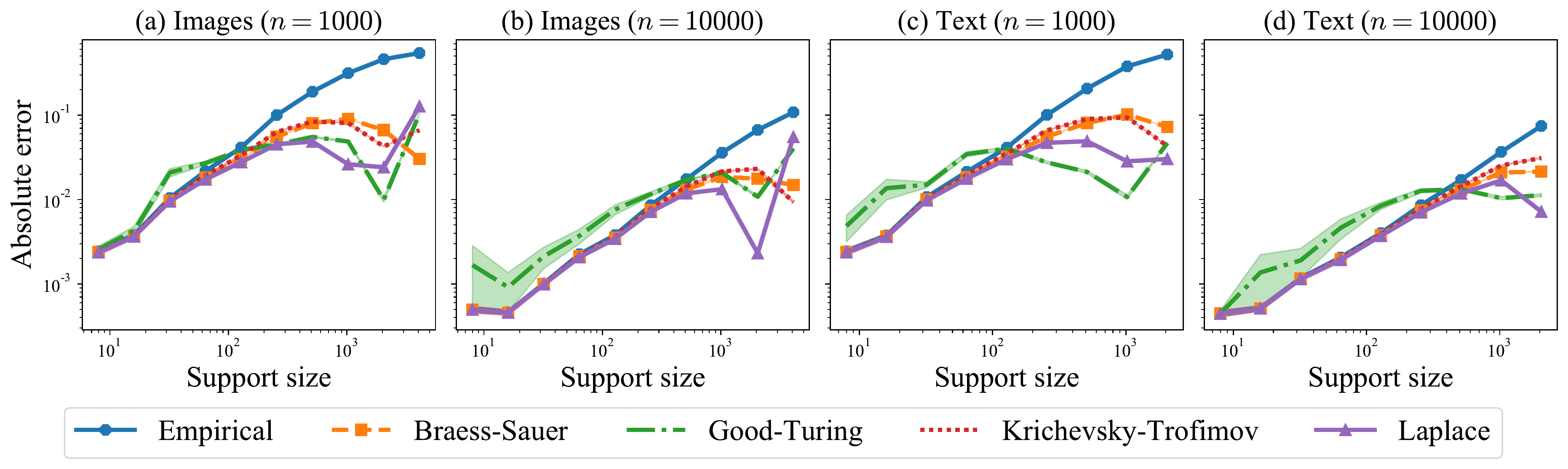}
    \adjincludegraphics[width=\textwidth, trim=0.0in 0in 0.0in 0.34in, clip=true]{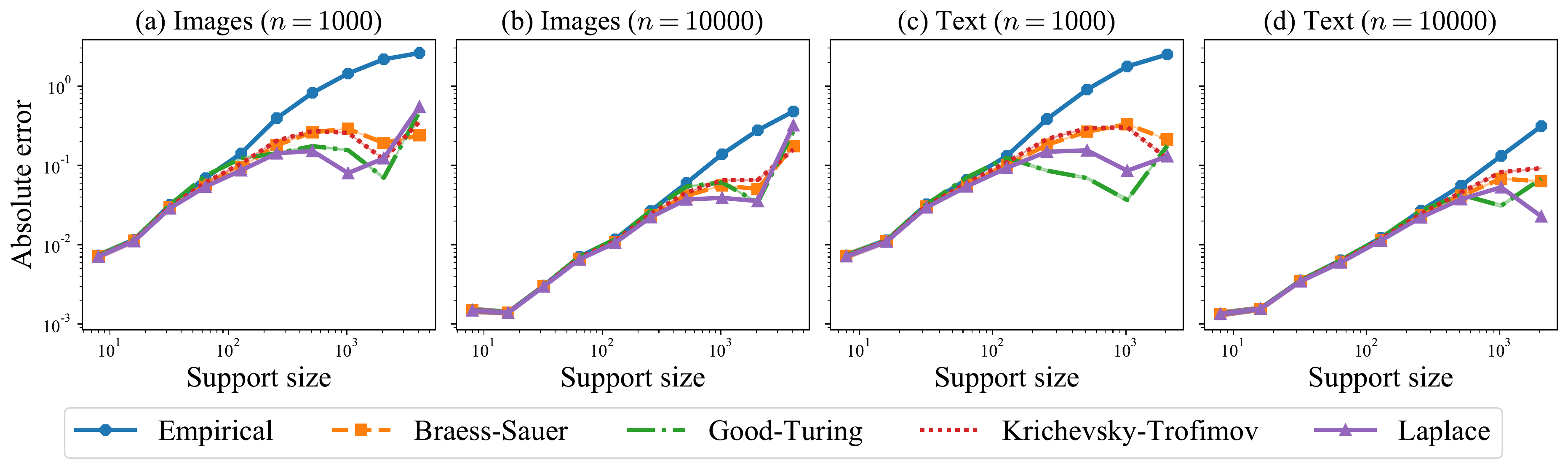}
    \caption{Absolute error versus support size $k$ on real data for the frontier integral \textbf{(top)} and the divergence frontier \textbf{(bottom)}. \textbf{Left Two}: Image data (CIFAR-10) with sample sizes $n=1000$ and $n=10000$. \textbf{Right Two}: Text data (WikiText-103) with sample sizes $n=1000$ and $n=10000$.}
    \label{fig:real_smoothing_kvary_appendix}
\end{figure}

\begin{figure}[t]
    \centering
    \includegraphics[width=0.49\textwidth]{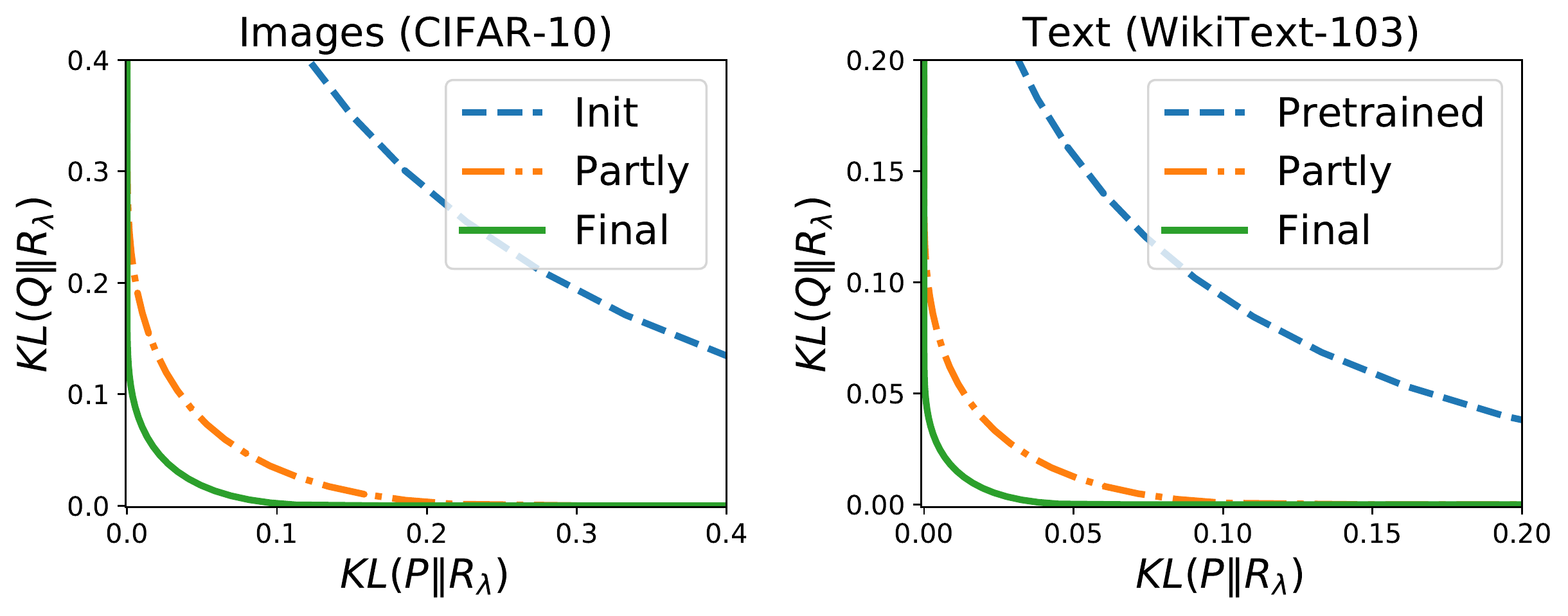}
    \includegraphics[width=0.49\textwidth]{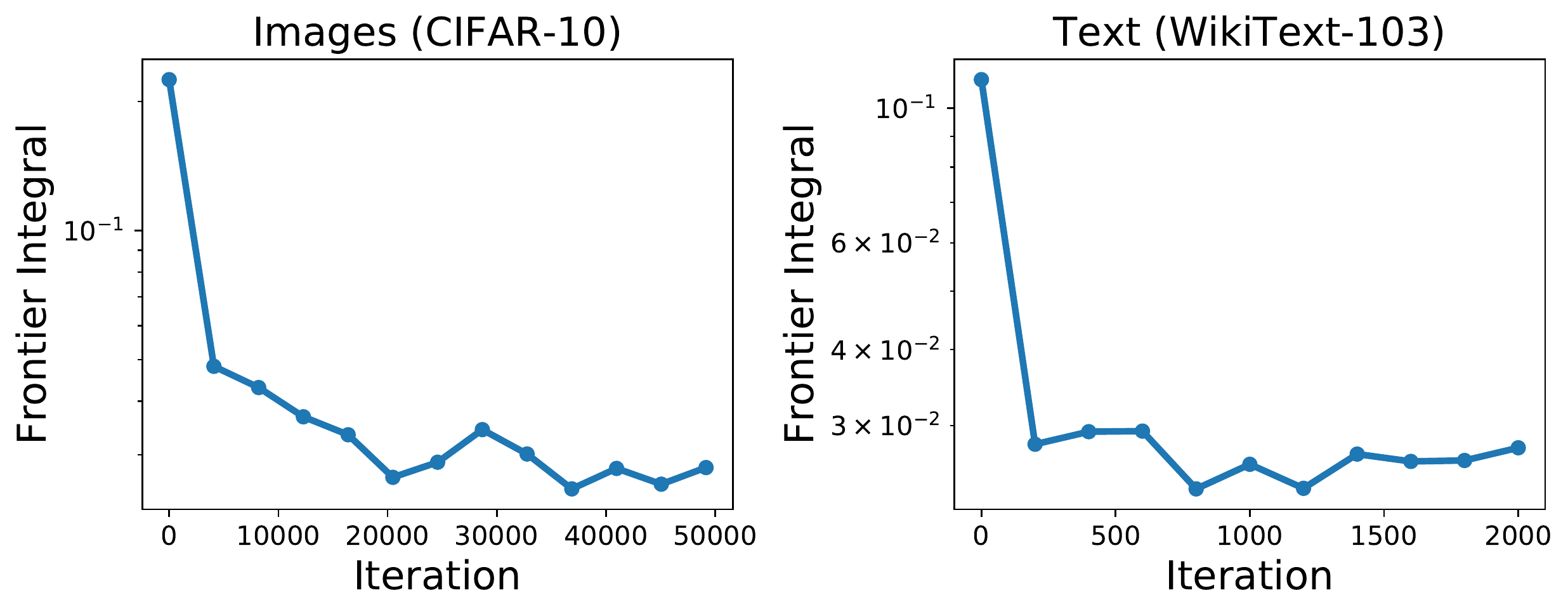}
    \hfill
    \caption{\textbf{Left Two}: The divergence frontier at different points in training. \textbf{Right Two}: The frontier integral plotted at different training checkpoints.}
    \label{fig:real_training_appendix}
\end{figure}

\section{Length of the divergence frontier}
\label{sec:a:length}
In this section, we discuss how the length of the divergence frontier is different from the frontier integral. 
In particular, we show that the length of the 
divergence frontier is lower bounded by the Jeffery divergence, which could be unbounded, whereas the frontier integral is always bounded between $0$ and $1$.

\myparagraph{Setup}
Let $P, Q$ be two distributions on a finite alphabet $\Xcal$.
Recall that the divergence frontier is defined as the parametric curve
$\Fcal(P, Q) := (x(\lambda), y(\lambda))$ for $\lambda \in (0, 1)$ where
\begin{align} \label{eq:div-curve}
\begin{aligned}
    x(\lambda) &= \klam{1-\lambda}(Q \Vert P) 
        = \sum_{a \in \Xcal} Q(a) \log\frac{Q(a)}{\lambda P(a) + (1-\lambda) Q(a)}  \\
    y(\lambda) &= \klam{\lambda}(P \Vert Q)
        = \sum_{a \in \Xcal} P(a) \log\frac{P(a)}{\lambda P(a) + (1-\lambda) Q(a)}  \,.
\end{aligned}
\end{align}

Recall that the Jeffery divergence between $P$ and $Q$ is defined as 
\[
    \mathrm{JD}(P, Q) = \kl(P\Vert Q) + \kl(Q \Vert P) = \sum_{a \in \Xcal} \big(P(a) - Q(a)\big) \big(\log P(a) - \log Q(a)\big) \,.
\]
Note that $\mathrm{JD}(P, Q)$ is unbounded when
there exists an atom such that 
$P(a) = 0, Q(a) \neq 0$ 
or $P(a) \neq 0, Q(a) = 0$. 

We show that the length of the divergence frontier between $P, Q$
is lower bounded by the corresponding Jeffrey's divergence, which can be unbounded. 

\begin{proposition}
    Consider two distributions $P, Q$ on a finite alphabet $\Xcal$.
    The length $\mathrm{length}(\Fcal(P, Q))$ 
    of the divergence frontier $\Fcal(P, Q)$ satisfies 
    \[
    \mathrm{length}(\Fcal(P, Q)) \ge \frac{1}{\sqrt{2}} \mathrm{JD}(P, Q) \,.
    \]
\end{proposition}
\begin{proof}
    We assume without loss of generality that 
    $P(a) + Q(a) > 0$ for each $a \in \Xcal$.
    Define shorthand $R_\lambda = \lambda P + (1-\lambda) Q$.
    We bound the length of the divergence frontier $\Fcal(P, Q)$, which is given by
    $\int_0^1 L(\lambda) \D \lambda$, as
    \begin{align*}
        L(\lambda)^2 &= x'(\lambda)^2 + y'(\lambda)^2 \\
        &=  \left(\sum_{a \in \Xcal} Q(a) \frac{Q(a) - P(a)}{R_\lambda(a)} \right)^2 + 
            \left(\sum_{a \in \Xcal} P(a) \frac{Q(a) - P(a)}{R_\lambda(a)} \right)^2 \\
        &\ge \frac{1}{2} \left( \sum_{a\in \Xcal} \frac{(P(a) - Q(a))^2}{R_\lambda(a)} \right)^2
            =: \frac{1}{2}\widetilde L(\lambda)^2 \,,
    \end{align*}
    where we used the inequality $(a-b)^2 \le 2(a^2 + b^2)$ for $a, b \in \reals$.
    We can now complete the proof 
    by computing this integral as 
    \begin{align*}
    \sqrt{2} \cdot  \mathrm{length}(\Fcal(P, q))
    &\ge \int_0^1 \widetilde L(\lambda) \D \lambda \\
        &=  \int_0^1   \sum_{a\in \Xcal} \frac{(P(a) - Q(a))^2}{R_\lambda(a)} \D \lambda \\
        &= \sum_{a\in \Xcal} (P(a) - Q(a))^2
        \int_0^1 \frac{1}{\lambda P(a) + (1-\lambda) Q(a)}\D\lambda \\
        &= \sum_{a\in \Xcal} (P(a) - Q(a)) (\log P(a) - \log Q(a)) 
        = \mathrm{JD}(P, Q) \,.
    \end{align*}
\end{proof}

\section{Technical lemmas}
\label{sec:a:lemma}
We state here some technical results used in the paper.

\begin{theorem}[McDiarmid's Inequality]
\label{thm:technical:mcdiarmid}
    Let $X_1, \cdots, X_m$ be independent random variables such that $X_i$ has range $\Xcal_i$. 
    Let $\Phi:\Xcal_1\times \cdots\times \Xcal_n \to \reals$ be any function which satisfies
    the bounded difference property. That is, there exist constants $B_1, \cdots, B_n > 0$ such that for every $i = 1, \cdots, n$ and $(x_1, \cdots, x_n), (x_1', \cdots, x_n') \in  \Xcal_1 \times \cdots \Xcal_n$ which differ only on the $i$\textsuperscript{th} coordinate (i.e.,  $x_j = x_j'$ for $j \neq i$), we have, 
    \[
        |\Phi(x_1, \cdots, x_n) - \Phi(x_1', \cdots, x_n')|
        \le B_i \,.
    \]
    Then, for any $t > 0$, we have, 
    \[
    \prob \left( \left| \Phi(X_1, \cdots, X_n) -   \expect[\Phi(X_1, \cdots, X_n)] \right| > t \right)  \le 2\exp\left( -\frac{2t^2}{\sum_{i=1}^n B_i^2} \right) \,.
    \] 
\end{theorem}

\begin{property} \label{property:fdiv:fprime:monotonicity}
    Suppose $f: (0, \infty) \to [0, \infty)$ is convex 
    and continuously differentiable with $f(1) = 0 = f'(1)$. 
    Then, $f'(x) \le 0$ for all $x \in (0, 1)$
    and $f'(x) \ge 0$ for all $x \in (1, \infty)$.
\end{property}
\begin{proof}
    Monotonicity of $f'$ means that we have for any $x \in (0, 1)$ and $y \in (1, \infty)$
    that $f'(x) \le f'(1) = 0 \le f'(y)$.
\end{proof}

\begin{lemma} \label{lem:techn:missing-mass-2}
    For all $x \in (0, 1)$ and $n \ge 3$, we have
    \[
        0 \le (1-x)^n x \log \frac{1}{x} \le \frac{\log n}{n} \,.
    \]
\end{lemma}
\begin{proof}
    Let $h(x) = (1-x)^n x \log(1/x)$ be defined on $(0, 1)$.
    Since $\lim_{x \to 0} h(x) = 0 < h(1/n)$, the global supremum does not occur as $x \to 0$.
    We first argue that $h$ obtains its global maximum in 
    $(0, 1/n]$. 
    We calculate
    \[
        h'(x) = (1-x)^{n-1}\left( -nx \log\frac{1}{x} + (1-x)\left(\log\frac{1}{x} - 1\right) \right) 
        \le (1-x)^{n-1} (1-nx) \log\frac{1}{x} \,.
    \]
     Note that $h'(x) < 0$ for $x > 1/n$, 
     so $h$ is strictly decreasing on $(1/n, 1)$. Therefore, 
     it must obtain its global maximum on $(0, 1/n]$.
     On this interval, we have, 
    \[
        (1-x)^n x \log\frac{1}{x} \le x \log\frac{1}{x}
        \le \frac{\log n}{n} \,,
    \]
    since $x \log(1/x)$ is increasing on $(0, \exp(-1))$.
\end{proof}

The next lemma comes from \cite[Theorem 1]{berend2012missing}.
\begin{lemma}
\label{lem:techn:missing-mass-1}
    For all $x \in (0, 1)$ and $n \ge 1$, we have
    \[
        0 \le (1-x)^n x \le \exp(-1)/(n+1) < 1/n \,.
    \]
\end{lemma}

\end{document}